\documentclass[12pt,a4paper]{article}

\usepackage{amsthm}
\usepackage{amsmath}
\usepackage{mathtools}
\usepackage{amssymb}
\usepackage{cite}
\usepackage{ulem}
\usepackage{stackrel}
\usepackage{enumerate}
\usepackage[table,xcdraw]{xcolor}
\usepackage[printonlyused]{acronym}
\usepackage{dsfont}
\usepackage{xfrac}
\usepackage[acronym,toc,shortcuts, nohypertypes={acronym}]{glossaries}
\usepackage[hidelinks]{hyperref}       

\newtheorem{lemma}{Lemma}
\newtheorem{prop}{Proposition}
\newtheorem{theorem}{Theorem}
\newtheorem{cor}{Corollary}
\theoremstyle{definition}
\newtheorem{remark}{Remark}
	\definecolor{blue-violet}{rgb}{0.54, 0.17, 0.89}

\usepackage{subcaption}
\usepackage{float}

\def\E{\mathbb{E}}
\def\RR{\mathbb{R}}
\def\eps{\epsilon}
\def\snr{\mathsf{SNR}}
\def\lambdamax{\lambda_{\text{max}}}
\def\Ytag{Y'}
\def\YY{Y}
\def\X{\bs{Y}}
\def\bX{\bs{X}}
\def\bY{\bs{Y}}
\def\Y{\bs{Y}'}
\def\x{\bs{y}}
\def\bx{\bs{x}}
\def\by{\bs{y}}
\def\U{\bs{U}}
\def\Q{\bs{Q}}
\def\H{\bs{X}}
\def\h{\bs{x}}
\def\D{\bs{D}}
\def\P{\bs{P}}
\def\Z{\bs{Z}}
\def\z{\bs{z}}
\def\V{\bs{V}}
\def\W{\bs{W}}
\def\I{\bs{I}}
\def\A{\bs{A}}
\def\B{\bs{B}}
\def\F{\bs{F}}
\def\a{\bs{a}}
\def\T{\bs{T}}
\def\t{\bs{t}}
\def\S{\bs{S}}
\def\G{\bs{G}}
\def\L{\bs{L}}
\def\s{\bs{s}}
\def\g{\bs{g}}
\def\SS{S}
\def\ss{s}
\def\GG{G}
\def\dW{\mathcal{W}}
\def\vmuQ{\bs{\mu}}
\def\muQ{\mu}
\DeclareMathOperator{\Noise}{\mathsf{Noise}}
\DeclareMathOperator{\pBias}{\mathsf{p-Bias}}

\DeclareMathOperator{\Var}{\mathsf{Var}}

\DeclareMathOperator{\trace}{Tr}

\newcommand{\bs}[1]{\boldsymbol{#1}}

\newcommand{\indfunc}[1]{\mathds{1}\left(#1\right)}
\DeclareMathOperator*{\argmin}{\arg\!\min}

\newif\ifpaperonly
\paperonlytrue
 \paperonlyfalse

\def\bt{\bs{\beta}}
\def\hbt{\hat{\bt}}
\def\hbtMN{\hbt_{\mathrm{MN}}}
\def\hbtB{\hbt_{\mathrm{BMN}}}

\def\hbtSB{\hbt_{\mathrm{SBMN}}}

\def\inAS{\overset{\text{a.s}}{\longrightarrow}}
\def\inP{\overset{P}{\longrightarrow}}
\def\inL1{\overset{L^1}{\longrightarrow}}

\newacronym{SGD}{SGD}{Stochastic Gradient Descent}
\newacronym{GD}{GD}{Gradient Descent}
\newacronym{min-norm}{min-norm}{minimum norm}
\newacronym{mse}{MSE}{minimum squared error}
\newacronym{snr}{SNR}{signal-to-noise ratio}
\newacronym{batch-min-norm}{BMN}{batch minimum norm}
\newacronym{s-batch-min-norm}{SBMN}{shrinkage batch minimum norm}
\newacronym{batch-ridge}{BR}{batch ridge regression}

\usepackage{scalerel,stackengine}
\stackMath
\newcommand\reallywidehat[1]{%
\savestack{\tmpbox}{\stretchto{%
  \scaleto{%
    \scalerel*[\widthof{\ensuremath{#1}}]{\kern-.6pt\bigwedge\kern-.6pt}%
    {\rule[-\textheight/2]{1ex}{\textheight}}
  }{\textheight}%
}{0.5ex}}%
\stackon[1pt]{#1}{\tmpbox}%
}
\begin{document}

\title{Batches Stabilize the Minimum Norm Risk in High-Dimensional Overparametrized  Linear Regression}

\author{%
Shahar~Stein~Ioushua,    Inbar~Hasidim,  Ofer~Shayevitz and  Meir~Feder\thanks{The authors are with the Department of EE--Systems, Tel Aviv University, Tel Aviv, Israel \{steinioushua@mail.tau.ac.il,inbarhasidim@mail.tau.ac.il, ofersha@eng.tau.ac.il,meir@tauex.tau.ac.il\}. This work was supported by ISF grant no. 1766/22 and 819/20.}
}
\date{}

\maketitle

\begin{abstract}
    Learning algorithms that divide the data into batches are prevalent in many machine-learning applications, typically offering useful trade-offs between computational efficiency and performance. In this paper, we examine the benefits of batch-partitioning through the lens of a minimum-norm overparametrized linear regression model with isotropic Gaussian features. We suggest a natural small-batch version of the minimum-norm estimator and derive bounds on its quadratic risk. We then characterize the optimal batch size and show it is inversely proportional to the noise level, as well as to the overparametrization ratio. In contrast to minimum-norm, our estimator admits a stable risk behavior that is monotonically increasing in the overparametrization ratio, eliminating both the blowup at the interpolation point and the double-descent phenomenon. We further show that shrinking the batch minimum-norm estimator by a factor equal to the Weiner coefficient further stabilizes it and results in lower quadratic risk in all settings. Interestingly, we observe that the implicit regularization offered by the batch partition is partially explained by feature overlap between the batches. Our bound is derived via a novel combination of techniques, in particular normal approximation in the Wasserstein metric of noisy projections over random subspaces. 
\end{abstract}

\section{Introduction}

Batch-based algorithms are used in various machine-learning problems. Particularly, partition into batches is natural in distributed settings, where data is either collected in batches by remote sensors that can send a small number of bits to a central server, or collected locally but offloaded to multiple remote workers for computational savings, see e.g. \cite{zhang2012communication,zhang2015divide,richtarik2016distributed,dean2012large,newman2009distributed,verbraeken2020survey,dobriban2020wonder,dobriban2021distributed,mucke2022data,lee2015communication}.  Learning in batches is also employed in centralized settings; this is often done to reduce computational load but is also known (usually empirically) to sometimes achieve better convergence, generalization, and stability, see e.g.,~\cite{dean2012large,goyal2017accurate,hoffer2017train}.  One of the most basic and prevalent learning tasks is linear regression, which has been extensively studied in both centralized and distributed settings.
Linear regression is of particular contemporary interest in the overparametrized regime, where the number of parameters exceeds the number of samples. In this regime, there are infinitely many interpolators, and a common regularization method is to pick the \ac{min-norm} solution, i.e., the interpolator whose $\ell^2$ norm is minimal. However, this method requires inverting a matrix whose dimensions are the number of samples, a task that can be computationally costly and also result in a non-stable risk, growing unbounded close to the interpolation point~\cite{marchenko1967distribution,hastie2022surprises}. Performing linear regression separately in batches and combining the solutions (usually by averaging) can help with the computational aspects, and has been studied before mainly for large (linear in the number of samples) batches \cite{dobriban2020wonder,dobriban2021distributed,mucke2022data,lee2015communication}. However, such solutions break down and cannot control the risk for sublinear batch size; they also shed no light on the performance benefits heuristically known to be offered by small batches. Can the \ac{min-norm} solution benefit more from small batch partitioning? We answer this question in the affirmative, by suggesting a simple and natural \ac{min-norm}-based small-batch regression algorithm, and showing it stabilizes the \ac{min-norm} risk. We discuss the ramifications of our result in several settings. 

\subsection{Our Contribution}

We consider a linear model with isotropic features, in the overparametrized regime with $n$ data samples and $p > n$ parameters, where the $n\times p$ feature matrix is i.i.d.~Gaussian. The risk attained by \ac{min-norm} in this setting and related ones was previously analyzed in~\cite{hastie2022surprises}.  Here, we suggest the following small batch variation of \ac{min-norm}. First, the data is partitioned into small disjoint batches of equal size $b$, and a simple \ac{min-norm} estimator is computed separately for each batch. Then, the resulting $\frac{n}{b}$ weak estimators are pooled together to form a new $\frac{n}{b}\times p$ feature matrix for a modified ``linear'' model, with suitably weighted modified samples. The modified model is not truly linear, since both the new features and noise depend on the parameter. Finally, a \ac{min-norm} estimator is computed in this new setting, yielding our suggested \ac{batch-min-norm} estimator. Note that while the modified model is far more overparametrized than the original model (by a factor of $b$), its features are now favorably correlated / better aligned with the underlying parameters. We shall see that this trade-off can be beneficial. 
\begin{figure*}[ht]
\begin{minipage}{.5\textwidth}
  \centering
  \includegraphics[width=1\linewidth]{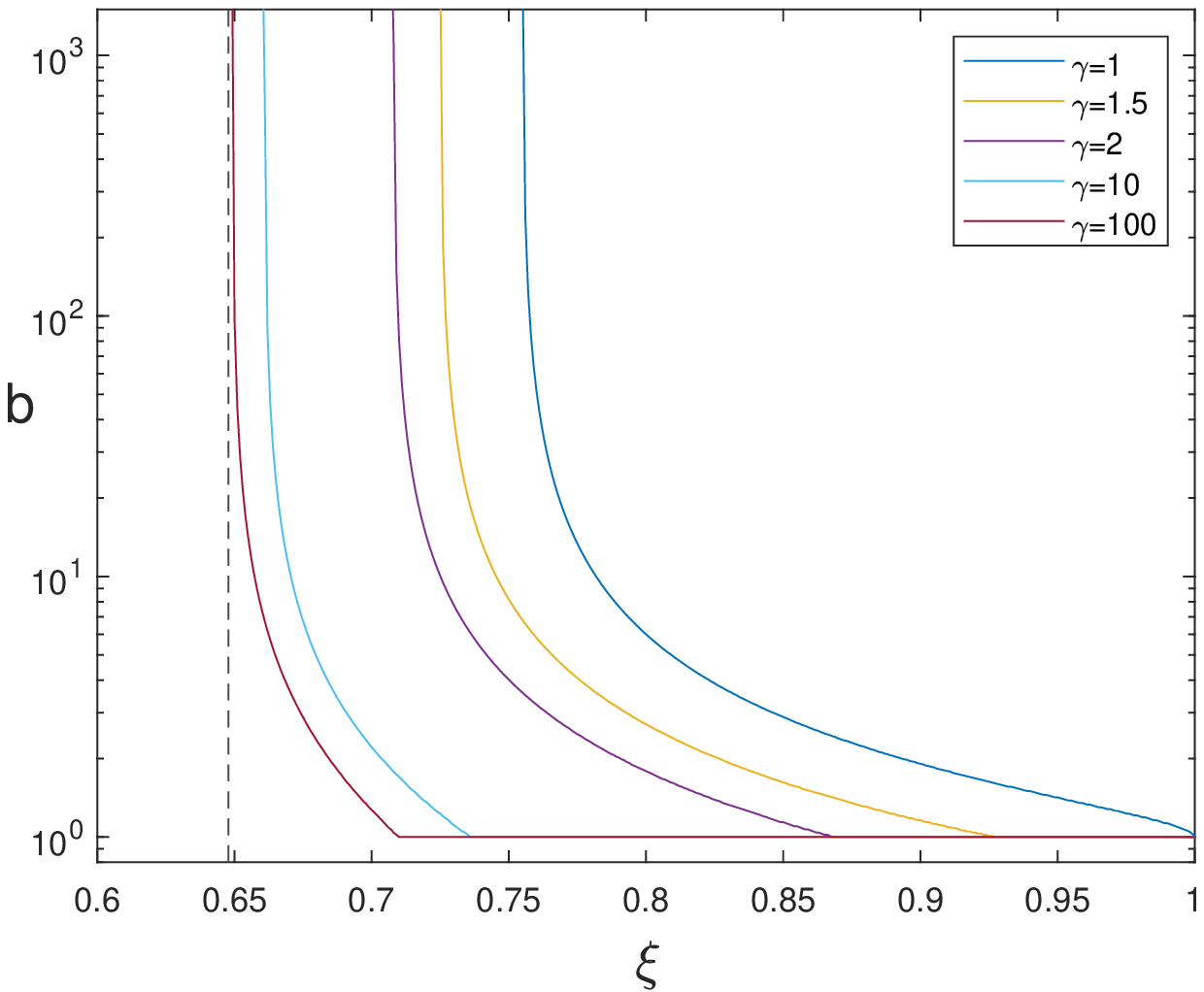}
  \caption*{(a)}
\end{minipage}%
\begin{minipage}{.5\textwidth}
  \centering
  \includegraphics[width=1\linewidth]{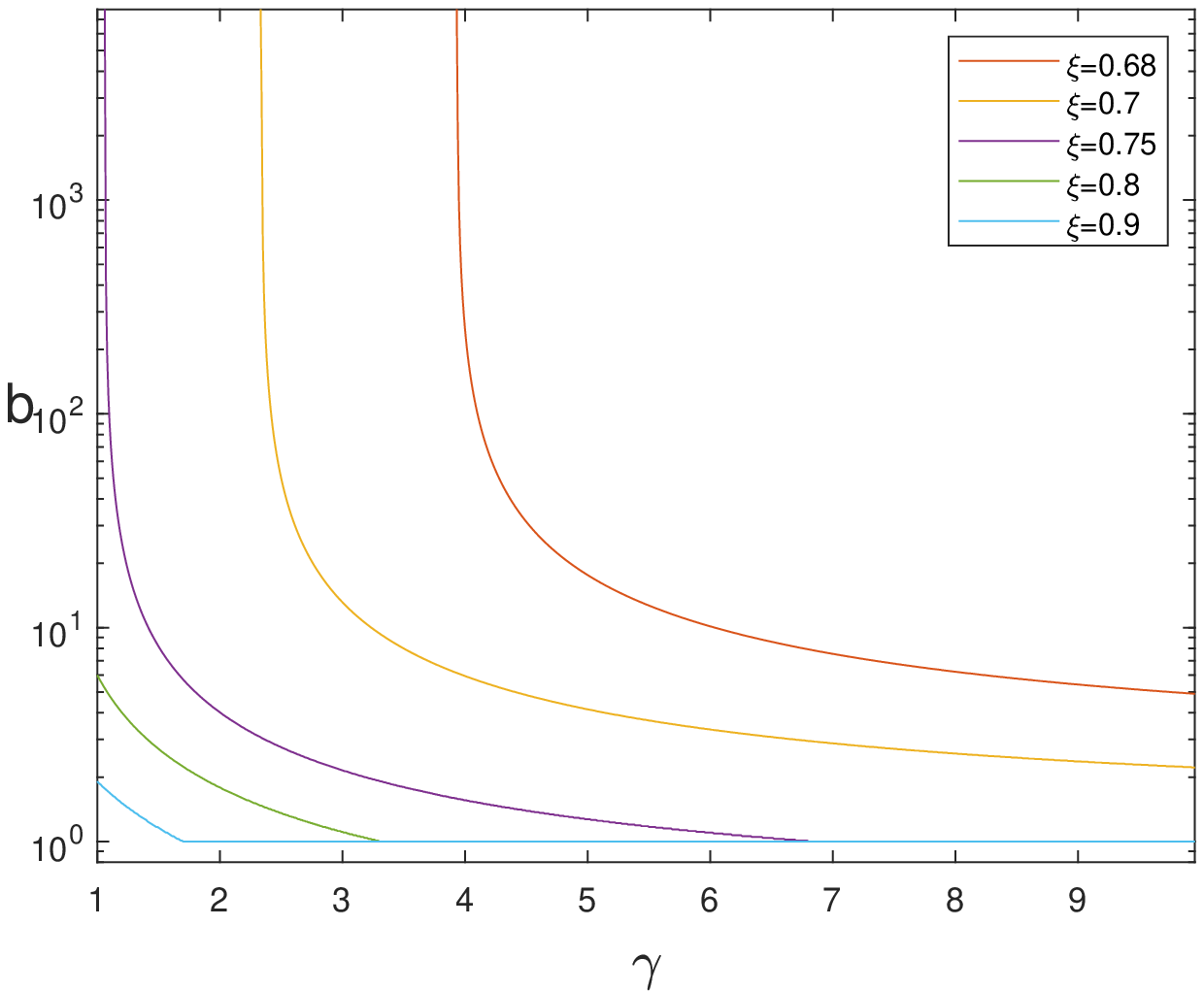}
  \caption*{(b)}
\end{minipage}
  \caption{\textit{Optimal batch size vs. (a) normalized signal-to-noise ratio  $\xi = \frac{\snr}{1+\snr}$, and (b) overparametrization ratio $\gamma$. When $\xi<0.6478$, the optimal batch size $b\rightarrow \infty$, for any $\gamma>1$.}}\label{fig:opt_b} 
\end{figure*}

 To that end, we derive upper and lower bounds on the risk obtained by our estimator, in the limit of $n,p\to \infty$ with a fixed overparametrization ratio $\gamma = p/n$, as a function of the $\snr$. When compared to simulations, our upper bound is demonstrated to be quite tight. We then analytically find the batch size minimizing the upper bound, and show that it is inversely proportional to both $\gamma$ and $\snr$; in particular, there is a low-$\snr$ threshold point below which increasing the batch size (after taking $n,p\to \infty$) is always beneficial (albeit at very low $\snr$ we can do worse than the null solution), see  Figure~\ref{fig:opt_b}. Unlike \ac{min-norm}, and similarly to optimally-tuned ridge regression \cite{nakkiranoptimal}, the risk attained by \ac{batch-min-norm} is generally stable; it is monotonically increasing in $\gamma$, does not explode near the interpolation point $\gamma=1$, and does not exhibit a \textit{double descent} phenomenon \cite{hastie2022surprises}  (all this assuming $\snr$ $ \geq 1$, see Figure~\ref{fig:opt_algs_intro}). It is (trivially) always at least as good (and often much better) than min-norm. Another interesting observation is that the batch algorithm exactly coincides with the regular \ac{min-norm} algorithm for any batch size, whenever the feature matrix has orthogonal rows. Thus, somewhat intriguingly, the reason that batches are useful can be partially attributed to the fact that feature vectors are slightly linearly dependent between batches, i.e., there is a small overlap between the subspaces spanned by the batches. From a technical perspective, as we shall later see, this overlap implicitly regularizes the noise amplification suffered by the standard min-norm. Finally, we describe a shrinkage variation of our \ac{batch-min-norm} estimator, which has a stable risk for all $\snr$ levels and always outperforms both \ac{min-norm} and the null solution.

 The derivation of the upper bound is the main technical contribution of the paper. The main difficulty lies in the second step of the algorithm, namely analyzing the \ac{min-norm} with the feature matrix comprised of per-batch \ac{min-norm} estimators. Note that this step is no longer under a standard linear model; in particular, the new feature matrix and the corresponding noise vector depend on the parameters, and the noise vector also depends on the feature matrix, in a generally non-linear way. This poses a significant technical barrier requiring the use of several nontrivial mathematical tools. To overcome this challenge, we write the risk of the algorithm as the sum of a parameter bias term and an excess noise term. To compute the bias of the algorithm, we first write it as a recursive perturbed projection onto a random per-batch subspace, drawn from the Haar measure on the Stiefel manifold. We show that the statistics of these projections are asymptotically close in the Wasserstein metric to i.i.d.~Gaussian vectors, a fact that allows us to obtain a recursive expression for the bias as batches are being added, with a suitable control over the error term. We then translate this recursion into a certain differential equation, whose solution yields the asymptotic expression for the bias. To bound the excess noise of the algorithm, we show that it converges to a variance of a Gaussian mixture with $\chi^2$-distributed weights, projected onto the row-space of a large Wishart matrix. 

\begin{figure*}[ht]
   \begin{minipage}[b]{1\textwidth}
  \begin{minipage}[b]{0.56\textwidth}
\includegraphics[width=\columnwidth]{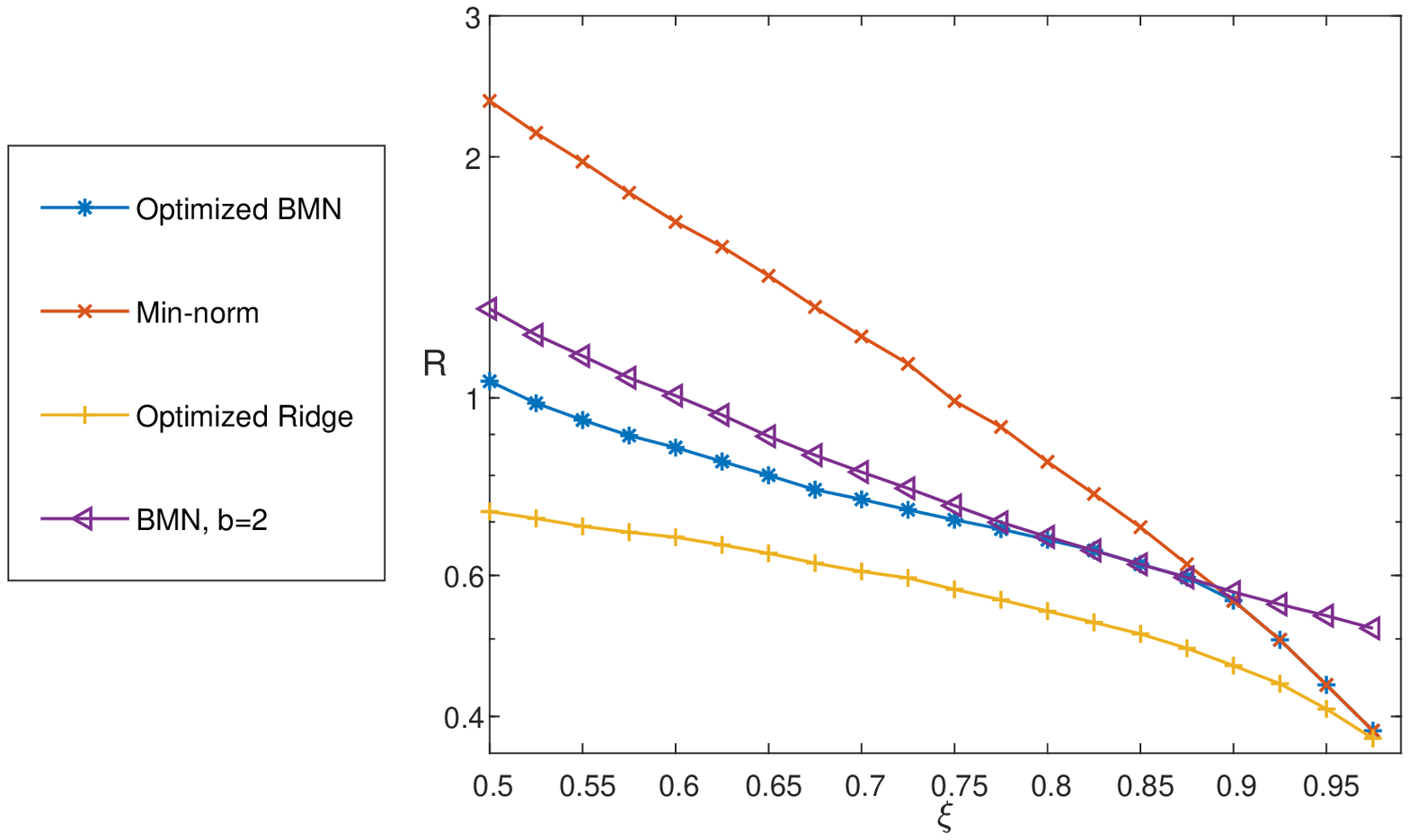}
\caption*{\qquad\qquad(a)}
\end{minipage}
\begin{minipage}[b]{0.5\textwidth}
\includegraphics[width=0.84\columnwidth]{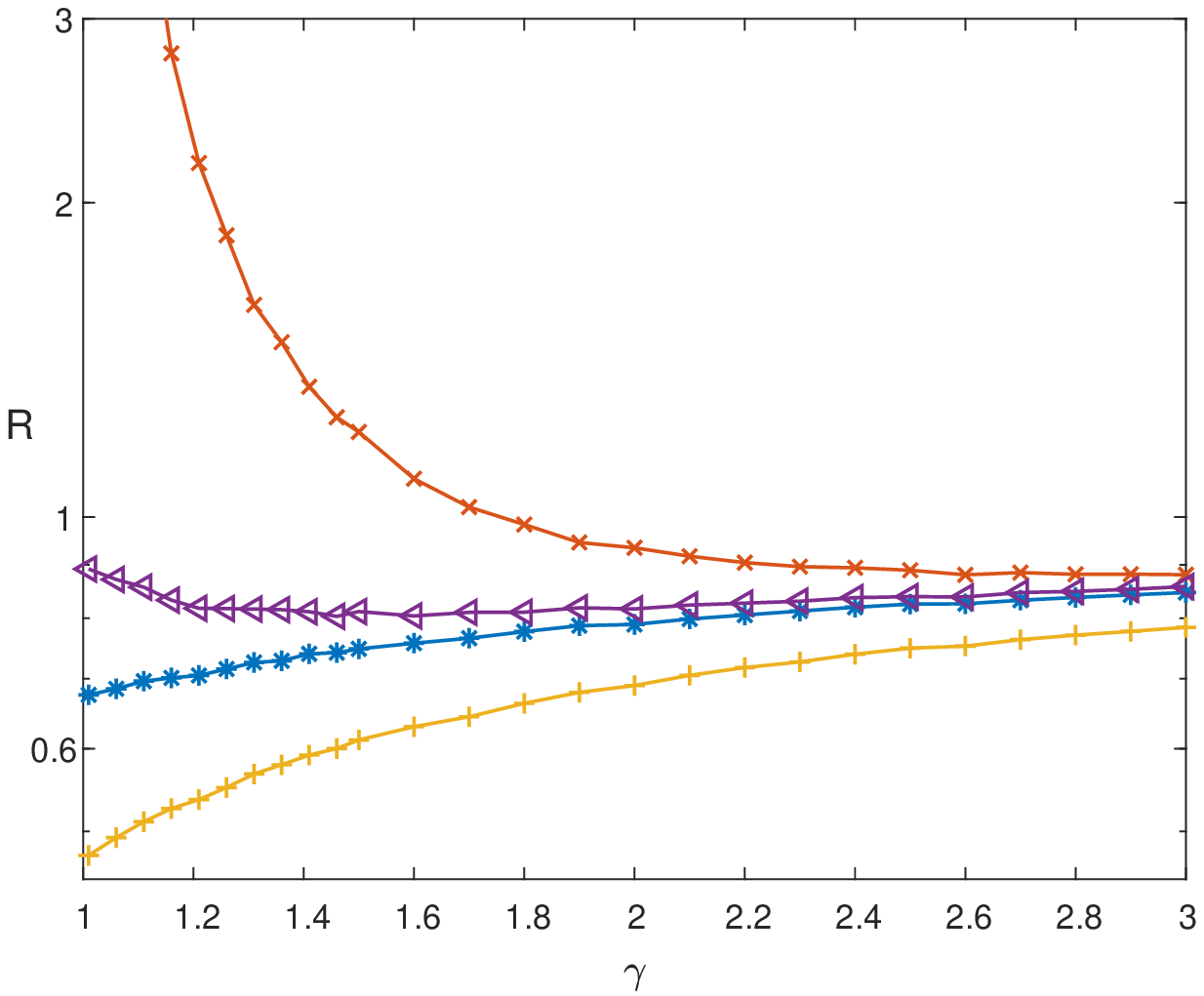}
\caption*{(b)}
\end{minipage}%
\end{minipage}%
\caption{\textit{Risk of \ac{batch-min-norm} with optimal batch size vs. (a) normalized signal-to-noise ratio  $\xi = \frac{\snr}{1+\snr}$, with $\gamma=1.5$, and (b) overparametrization ratio $\gamma$ with $\xi =0.7$. Optimized ridge is ridge regression with the optimal regularization parameter. 
\label{fig:opt_algs_intro}
}}
\end{figure*}

\subsection{Ramifications}\label{sec:remifications}
\textbf{Distributed linear regression.} 
In this setting, the goal is usually to offload the regression task from the main server by distributing it between multiple workers; the main server then merges the estimates given by the workers. This merging is typically done by simple averaging, e.g., \cite{zhang2012communication,zhang2015divide,dobriban2020wonder,mucke2022data}. The number of workers is typically large but fixed, hence the batch size is linear in the sample size. The regime of sublinear batch sizes is less explored in the literature, perhaps due to practical reasons. When the batch size is sublinear, the server-averaging approach breaks down since its risk is trivially dominated by the per-batch bias, and hence it attains the null risk asymptotically. In contrast, our algorithm projects the modified observations onto the subspace spanned by the entire collection of weak estimators. Hence, the resulting estimator is much less biased than each weak estimator separately. Therefore, our algorithm is far superior to server-averaging for fixed batch size. Numerical results indicate that this is also true in the general sublinear regime. 

\textbf{Linear regression under real-time communication constraints.} Consider a low-complexity sensor sequentially viewing samples (feature vectors and observations). The sensor cannot collect all the samples and perform the regression task, hence it offloads the task to a remote server, with which it can communicate over a rate-limited channel. Due to the communication constraint, the sensor cannot send all of its samples to the server. Our algorithm gives a natural solution: choose a batch size $b$ large enough such that sending $(p+1)/b$ samples per unit of time becomes possible, and send the \ac{min-norm} solution of each batch to the server, who in turn will run the second step of the algorithm. Our bounds show how the risk behaves as a function of $b$. For reasons already mentioned, this algorithm is far better than simple solutions such as averaging or trivially sending one in $b$ samples. We note that the general topic of statistical inference under communication constraints has been extensively explored (see e.g. \cite{amari1998statistical,zhang2013information,braverman2016communication,han2018geometric,zhang1988estimation,hadar2019communication,hadar2019distributed} but typically for a fixed parameter dimension, which does not cover the overparametrized setting we consider here.

\textbf{Mini-batch learning,} High-dimensional overparametrized linear regression is known to sometimes serve as a reasonable proxy (via linearization) to more complex settings such as deep neural networks \cite{jacot2018neural,du2018gradient,du2019gradient,allen2019convergence,chizat2019lazy}. Furthermore, \ac{min-norm} is equivalent to full \ac{GD} in linear regression, and even exhibits similar behavior observed when using \ac{GD} in complex models, e.g. the double-descent \cite{hastie2022surprises}. Learning using small batches, e.g. mini-batch SGD, is a common approach that originated from computational considerations\cite{robbins1951stochastic,goyal2017accurate}, but was also observed to improve generalization \cite{keskar2017large, masters2018revisiting, lin2019don,he2019control, kandel2020effect, lin2022analysis}. There is hence a clear impetus to study the impact of small batches on \ac{min-norm}-flavor algorithms in the linear regression setting. Indeed, mini-batch SGD for linear regression has been recently studied in~\cite{paquette2021sgd}, who gave closed-form solutions for the risk in terms of Volterra integral equations (see also \cite{gerbelot2022rigorous} for a more general setting using mean field theory). In particular, and in contrast to a practically observed phenomenon in deep networks,~\cite{paquette2021sgd} showed that the linear regression risk of mini-batch SGD does not depend on the batch size $b$ as long as  $b \ll n^{1/5}$. Our \ac{batch-min-norm} algorithm, while clearly not equivalent to mini-batch SGD in the linear regression setting, does exhibit the small-batch gain phenomenon, and hence could perhaps nevertheless shed some light on similar effects empirically observed in larger models. In particular, the batch regularization effect we observe can be traced back to a data ``overlap'' between the batches, and it is interesting to explore whether this effect manifests itself in other settings. Moreover, from a high-level perspective, our algorithm ``summarizes'' each batch to create a ``representative sample'', and then trains again via \ac{GD} only on these representative samples. It is interesting to explore whether this approach can be rigorously generalized to more complex models.

\section{Organization}

The paper is organized as follows. In Section~\ref{sec:perlim}, we introduce notation and review some required preliminary results. In Section~\ref{sec:problem_formulation}, we describe the isotropic Gaussian linear regression problem and the well-known \ac{min-norm} estimator. Then, in Section~\ref{sec:algorithm}, we introduce the \ac{batch-min-norm} estimator, our batch version of the \ac{min-norm} estimator. In Section~\ref{sec:main_result}, we present our main results -- upper and lower bounds on the limiting risk of the \ac{batch-min-norm} estimator as a function of the batch size, $\snr$ and overparametrization ratio. The optimal batch size is discussed in Subsection~\ref{sec:optimal_batch_size}. The main technical part of the paper is presented in Section~\ref{sec:proof_of_main_result}, where we prove the main result by deriving the bounds on the parameter bias  (Subsection~\ref{sec:semi_noisy}) and excess noise (Subsection~\ref{sec:noisy}) of our \ac{batch-min-norm} algorithm. The gap between the bounds is discussed in Subsection~\ref{sec:bounds_gap}.  We then proceed to suggest a shrinkage variation of the \ac{batch-min-norm} estimator in Section~\ref{sec:de_meaned}. In Section~\ref{sec:sims} we evaluate the performance of our estimators via numerical experiments, and in Section~\ref{sec:discussion} we discuss some of the limitations of our results and review possible future work.

\section{Preliminaries}\label{sec:perlim}
\subsection{Notation}
We write $x,\x$, and $A$ for scalars, vectors, and matrices, respectively. Vectors can be either row or column, as clear from the context. We use  $X$, $\bX$, and $\A$ to denote random variables, random vectors, and random matrices, respectively. We write $\|A\|_2$ and sometimes $\|A\|$ for the $\ell^2$ norm of the matrix $A$, and $\|A\|_{op}$, $\|A\|_F$ for the operator and Frobenius norms of $A$, respectively. 
For a real-valued p.s.d.~matrix $A$, we write $\lambda_i(A)$ to denote the $i$-th largest eigenvalue of $A$, and $\lambda_\text{min}(A)$, $\lambdamax(A)$ for the minimal and maximal eigenvalues of $A$, respectively. 
We say that $b$ orthonormal random column vectors $\U_1,\cdots,\U_b\in\mathbb{R}^n$ are \textit{uniform} or \textit{uniformly drawn},  if the matrix $\U=[\U_1,\cdots,\U_b]$ is drawn from the Haar measure on the Stiefel manifold $\mathbb{V}_b(\mathbb{R}^n) \triangleq \{\A\in\mathbb{R}^{b\times n}: \A\A^T=\I_{b}\}$, where
$\I_b$ is the $b\times b$ identity matrix.  We sometimes drop the subscript and write $\I$, when the dimension is clear from context.
We write $X_n\inP X$,  $X_n\inAS X$, and $X_n\inL1 X$ to indicate that the sequence $X_1,X_2,\cdots$ of random variables converges to the random variable $X$ in probability,  almost surely (a.s.) and in $L^1$, respectively. Similar notation for (finite-dim) random vectors / matrices means convergences of all entries.

\subsection{Wasserstein Distance}
The $p$-Wasserstein distance between two probability measures $\mu$ and $\nu$ on $\mathbb{R}^n$ is
\begin{align}\label{eq:Wp_distance_def}
    \dW_p(\mu,\nu) \triangleq \left(\,\inf\E\|\bX-\bY\|_2^p\,\right)^{\frac{1}{p}},
\end{align}
where the infimum is taken over all random vector pairs $(\bX,\bY)$ with marginals $\bX \sim \mu$ and $\bY \sim \nu$. 
With a slight abuse of notations, we write $\dW_p(\bX,\bY)$ to indicate the $p$-Wasserstein distance between the corresponding probability measures of $\bX$ and $\bY$. 
Throughout this paper, we say Wasserstein distance to mean the $1$-Wasserstein distance $\dW_1(\mu,\nu)$, unless explicitly mentioned otherwise. 
The following theorem is the well-known dual representation of the $1$-Wasserstein distance. 
\begin{theorem}[$\dW_1$ duality (\cite{Kantorovich58duality})]\label{thrm:Wasser_dual_rep}
For any two probability measures $\mu$ and $\nu$ over $\RR^n$, it holds that 
\begin{align}\label{eq:Wasser_dual_rep}
    \dW_1(\mu,\nu) =  \sup\, \E f(\bX)-\E f(\bY),
\end{align}    
where $\bX\sim \mu$, $\bY\sim \nu$, and the supremum is taken over all $1$-Lipschitz functions $f:\RR^n\to \RR$. 
\end{theorem}

It was previously shown in \cite{chatterjee2007multivariate} that low-rank projection of uniformly random orthonormal matrix is close to Gaussian in Wasserstein distance. This result is summarized in the following theorem.

\begin{theorem}[Theorem 11 in \cite{chatterjee2007multivariate}]\label{thm:Chatterjeee_Wasserstein_thm}
Let $B_1, \cdots, B_k$ be linearly independent $p\times p$ matrices (i.e. the only linear combination of them which is equal to the zero matrix has all coefficients equal to zero) over $\mathbb{R}$ such that $\trace{(B_iB^T_i)}=p$ for  $i=1,\cdots, k$. 
Define $b_{ij} \triangleq\trace{(B_i B^T_j)}$, let $\G = [\GG_1, . . . , \GG_k]$ be a random vector whose components have Gaussian distribution, with zero mean and covariance matrix $C \triangleq \frac{1}{p}(b_{ij})$ and let $\U$ be a uniformly random orthogonal matrix. Then, for
\begin{align}
\S \triangleq \left[\trace{(B_1\U)}, \trace{(B_2\U)},\cdots , \trace{(B_k\U)}\right]\in \mathbb{R}^k,    
\end{align}
and any $p \geq 2$ we have
\begin{align}
    \dW_1(\S,\G)\leq \frac{k\sqrt{2\|C\|_{\textrm{op}}}}{p-1}.
\end{align}
\end{theorem}

The Wasserstein distance plays a key role in our proofs, mainly due to the following facts. First, our \ac{batch-min-norm} algorithm performs projections onto small random subspaces, and Wasserstein distance can be used to quantify how far these are from projections onto i.i.d.~Gaussian vectors using Theorem~\ref{thm:Chatterjeee_Wasserstein_thm}. Second, closeness in Wasserstein distance implies change-of-measure inequalities for expectations of Lipschitz functions via the Kantorovich-Rubinstein duality theorem (Theorem~\ref{thrm:Wasser_dual_rep}), which allows us to compute expectations in the more tractable Gaussian domain with proper error control. For further details on the properties of Wasserstein distance, see Appendix~\ref{appen:wassertein}.

\section{Problem setup}\label{sec:problem_formulation}

Let $\X=(\YY_1,\cdots, \YY_n)^T$ be a data samples vector obtained from the linear model 
\begin{align}\label{eq:model}
\X = \H\bt  + \W   
\end{align}
where $\bt  \in \RR^p$  is a vector of unknown parameters, $\H\in\RR^{n\times p}$ is a given feature matrix with i.i.d. standard Gaussian entries, and $\W\in\RR^n$ is a  noise vector independent of $\H$ with i.i.d.~$\mathcal{N}(0,\sigma^2)$ entries. We write 
\begin{align}
    \|\bt \|_2 = r,
\end{align} 
to denote the norm of $\bt$, which is assumed unknown unless otherwise stated. 
Define the overparametrization ratio $\gamma\triangleq \frac{p}{n}$. When $\gamma>1$ we call the problem \textit{overparametrized} and when $\gamma<1$ we say it is {\em underparameterized}. 
An \textit{estimator} $\hbt = \hbt(\X, \H)$ for $\bt$ from the samples and features is a mapping $\hbt:\RR^n\times \RR^{n\times p}\to \RR^p$. We measure the  performance of an estimator via the quadratic (normalized) \textit{risk} $R(\hbt)$ it attains: 
\begin{align}\label{eq:risk_definition}
R(\hbt) \triangleq \frac{1}{r^2}\,\E \|\hbt(\X, \H) - \bt \|^2.    
\end{align}
Note that while $R(\hbt)$ is the parameter estimation risk, it is also equal in this case (up to a constant) to the associated prediction risk, namely the mean-squared prediction error $\E \|\mathbf{x}^T\hbt - \mathbf{x}^T\bt\|^2$ when using $\hbt$ to estimate the response to a new i.i.d.~feature vector $\mathbf{x}$. 

\subsection{Minimum-Norm Estimation}
In the overparametrized case $\gamma>1$, there is an infinite number of solutions to the linear model $\H\bt=\x$, and to choose one we need to impose some regularization. A common choice is the $\ell^2$-norm regularization, which yields the {\em \ac{min-norm} estimator}, defined as 
\begin{align}\label{eq:LSE}
    \hbtMN \triangleq  \argmin\|\bt\|_2^2, \text{ s.t. } \H\bt=\x,
\end{align}
and explicitly given by
\begin{align}
      \hbtMN= \H^T(\H \H^T)^{-1}\x.
\end{align}
The risk of the \ac{min-norm} estimator is then 
\begin{align}\label{eq:LS_risk}
    R(\hbtMN) 
    = \left(1-\frac{1}{r^2}\mathbb{E}\left\|\L\bt\right\|^2\right)+\frac{1}{r^2}\mathbb{E}\left[\W^T(\H \H^T)^{-1}\W\right],  
\end{align}
where $\L=\H^T(\H \H^T)^{-1}\H$ is the orthogonal projection onto the row space of $\H$, and we used the fact that the matrix $\H$ is orthogonal to its null space $\I-\L$. 
\begin{remark}
If we replace the expectations in~\eqref{eq:LS_risk} with conditional expectations given $\H$, we get that the first and second terms are the (normalized) {\em bias} the {\em variance} of the estimator (given the features), respectively. Note that here we define the risk \eqref{eq:risk_definition} as the average over both the noise and features. Hence, the above bias-variance decomposition does not extend to this setting, since now the ``bias'' part also exhibits variance due to the random features $\H$. Nevertheless, in the sequel we will loosely use the term bias when referring to the first term in the risk of min-norm and later \ac{batch-min-norm} and its shrinkage variation, to mean the part of the risk caused by the overparameterization, and similarly the term variance when referring to the associated second term.
\end{remark}
As previously shown in \cite{hastie2022surprises}, the asymptotic risk of the \ac{min-norm} estimator under the above model is 
\begin{align}\label{eq:asymptotic_MN_risk}
  \lim_{p\rightarrow \infty }  R(\hbtMN) =  1-\gamma^{-1} + \frac{1-\xi}{\xi}\cdot\frac{1}{\gamma-1}.
\end{align}
where 
\begin{align}
    \xi \triangleq \frac{r^2}{r^2+\sigma^2} = \frac{\snr}{1+\snr},
\end{align}
is the \textit{normalized $\snr$}, with $\snr\triangleq r^2/\sigma^2$.
The parameter $\xi$ is also known as the Wiener coefficient.

\section{Batch Minimum-Norm Estimation}\label{sec:algorithm}
We proceed to suggest and study a natural batch version of the \ac{min-norm} estimator. 
Let us divide the samples $\X$ to $n/b$ batches of some fixed size $b$, and denote by $\X_j\in\mathbb{R}^b$ the $j$th batch.
From each batch $\X_j$ we can obtain a \ac{min-norm} estimate of $\bt$, given by 
    \begin{align}
        \hbt_j \triangleq \H_j^T(\H_j\H_j^T)^{-1}\X_j,\quad j=1,\cdots,n/b,
    \end{align}
    where $\H_j$ are the feature vectors that correspond to $\X_j$. We now have $n/b$ weak estimators for $\bt$, each predicting only a tiny portion of $\bt$'s energy. However, these estimators are clearly better correlated with $\bt$ compared to random features. Hence, it makes sense to think of each $\hbt_j^T$ as a \textit{modified feature vector} $\h'_j$ that summarizes what was learned from batch $j$. Since each modified feature vector is a linear combination of its batch's feature vectors, with coefficients $\X_j^T \cdot (\H_j\H_j^T)^{-1}$, we can construct the corresponding \textit{modified sample} for batch $j$, given by 
    \begin{align}\label{eq:modefied_observ_Yi}
        \Ytag_j \triangleq \X_j^T \cdot (\H_j\H_j^T)^{-1}\X_j = \hbt_j^T\bt + W'_j = \h'_j\bt+W'_j, 
    \end{align}
    where 
    \begin{align}
        W'_j = \X_j^T (\H_j\H_j^T)^{-1}\W_j, \;\; j= 1, \cdots, n/b
    \end{align} 
    is the corresponding \textit{modified noise}. 
    
    We can now pool all these modified quantities to form a new model:
    \begin{align}\label{eq:modified_linear_model}
        \Y = \H'\bt+\W',
    \end{align}
    where the \textit{modified feature matrix} $\H'$ and \textit{modified noise vector} $\W'$ are given by 
    \begin{align}\label{eq:modified_H_and_W}
        \H'&=\left[{\hbt}_1,\cdots,{\hbt}_{n/b}\right]^T = \left[{\h'}_1^T,\cdots,{\h'}_{n/b}^T\right]^T, \\  
        \W'&=\left[ W'_1,\cdots, W'_{n/b}\right]^T.
    \end{align}
    Of course, the above is not truly a linear model, since both the matrix $\H'$ and the noise $\W'$ depend on the parameter $\bt$. But we can nevertheless naturally combine all the batches estimators, by simply applying \ac{min-norm} estimation to~\eqref{eq:modified_linear_model}. This  yields our suggested \ac{batch-min-norm} estimator:      

   \begin{align}\label{eq:batch_estimator}
        \hbtB \triangleq \H'^T(\H'\H'^T)^{-1}\Y.
    \end{align}
    
The risk of the the \ac{batch-min-norm} estimator $\hbtB$ is then given by 
\begin{align}\label{eq:BMN_risk_bias_var_decomp}
        R(\hbtB) = &\left(1-\frac{1}{r^2}\mathbb{E}\left\|\L'\bt\right\|^2\right)  
         \quad+\frac{1}{r^2}\mathbb{E}\left[\W'^T(\H'\H'^T)^{-1}\W'\right],
\end{align}
where $\L'$ is now the projection operator onto the subspace spanned by the rows of $\H'$, again using the fact that $\H'$ is orthogonal to its null space $\I-\L'$.

The first term in~\eqref{eq:BMN_risk_bias_var_decomp}, denoted hereinafter as $\pBias$, is the part of the risk that is caused by the overparametrization, namely due to the insufficient number of samples. The second term, denoted as $\Noise$, is, loosely speaking, the energy of the modified noise filtered by the (inverse) features. 
Note that, unlike the \ac{min-norm} case, the modified noise $\W'$ has a non-zero mean that does not depend on the parameter vector $\bt$ but contributes to the bias of the estimator. This contribution is part of $\Noise(\hbtB)$.  Nonetheless, the decomposition of the risk into $\pBias$ and $\Noise$ is beneficial in terms of analysis, as we will later see in Sections~\ref{sec:semi_noisy} and~\ref{sec:noisy} when we prove the bounds on $R(\hbtB)$.  Note that in Section~\ref{sec:de_meaned}, we introduce a shrinkage variation of \ac{batch-min-norm} for which the modified noise has zero mean, thus eliminating the part of the estimator's bias that originates from the noise.

Unlike in the \ac{min-norm} estimator case,  the rows of $\H'$ depend on the parameter $\bt$, and the noise $\W'$ depends on $\H'$, which makes the analysis of the parameter bias and excess risk significantly more challenging. 

It is interesting to point out that if $\H$ happens to have orthogonal rows, then $\L=\L'$, which means the parameter bias of the batch estimator coincides with that of \ac{min-norm}. Moreover, in this case, the variance of both batch- and regular \ac{min-norm} is simply the variances of the noises $\bs{W}'$ and $\W$ respectively, which are identical. Therefore, the risk of both estimators coincides in the orthogonal case, for any batch size. However, as we show in the next section, in the general case the risk can benefit from batch partition. This suggests that the gain of \ac{batch-min-norm} can be partially attributed to the linear dependence between the feature vectors in different batches.

\section{Main Result}\label{sec:main_result}

Our main results are upper and lower bounds on the risk of \ac{batch-min-norm}. Throughout the paper, limits are taken as $n,p \rightarrow \infty$ and $\gamma=p/n$ held fixed.

\begin{theorem}\label{thm:min_norm_risk}
For any  $\gamma>1/b$, the asymptotic risk of \ac{batch-min-norm} is bounded by
\begin{align}\label{eq:UB}
 \left(1-\frac{1}{\gamma b}\right)^{1+(b-1)\xi} +  \frac{1-\xi}{\xi}&\cdot\frac{b - (b-1)\xi}{\gamma b-1}\cdot C_{\gamma,\xi,b}\leq \lim_{p\rightarrow \infty}  R(\hbtB) \\
 & \leq \frac{\gamma b-1}{\gamma b + (b-1)\xi} +  \frac{1-\xi}{\xi}\cdot\frac{b - (b-1)\xi}{\gamma b-1},
    \end{align}
    where the first (resp. second) addend upper bounds the asymptotic parameter bias (resp. excess noise) in \eqref{eq:BMN_risk_bias_var_decomp}, and
    \begin{align}
        C_{\gamma,\xi,b}= 1-\frac{ \frac{1+(b-1)\xi}{\gamma b-1}}{1+ \frac{1+(b-1)\xi}{\gamma b-1}}.
    \end{align}
 \end{theorem}

Note that while we are interested in the $\gamma > 1$ regime, our bound applies verbatim to $\gamma > 1/b$. Loosely speaking, the reason is that working in batches ``shifts'' the interpolation point from $1$ to $1/b$, similarly to what would happen if we naively discarded all but $n/b$ samples and applied \ac{min-norm} to the remaining ones (\ac{batch-min-norm} is superior to this naive approach, see Section~\ref{sec:sims}). 

The bounds on the parameter bias in Theorem~\ref{thm:min_norm_risk} are tight for either low $\snr$ ($\xi\rightarrow 0$), high overparametrization ratio ($\gamma\rightarrow \infty$) or minimal batch size, $b=1$, in which case they coincide with the \ac{min-norm} performance. The variance bounds also coincide in the high overparametrization ratio regime. Therefore, when $\gamma \to \infty$  the bounds \eqref{eq:UB} are tight. Moreover, when $b=1$ the upper bound on the variance coincides with \ac{min-norm}, therefore in this case the upper bound \eqref{eq:UB} is tight (although it does not coincide with the lower bound). The lower bound on the variance is equal to the upper bound up to the multiplicative factor $C_{\gamma,\xi,b}$, which indeed tends to $1$ as $\gamma$ grows. However, when $\gamma$ approaches the (effective) interpolation point $\gamma=1/b$, $C_{\gamma,\xi,b}$ tends to zero and the bound becomes trivial.   As the batch size $b$ grows large, $C_{\gamma,\xi,b}\to \frac{\gamma}{\gamma+\xi}$, which tend to $1$ as either $\gamma\to 1$  or $\xi\to 0$. 
The gap between the lower and upper bounds is further discussed in Section~\ref{sec:bounds_gap}.

Our upper bound \eqref{eq:UB} turns out to be quite tight for a wide range of parameters (see Section~\ref{sec:sims}). It can therefore be used to obtain a very good estimate for the optimal batch size (which we therefore loosely refer to as ``optimal'' in the sequel), by minimizing~\eqref{eq:UB} as a function of $b$. 
To that end, we need to assume that the $\snr$ (namely $\xi$) is known; this is often a reasonable assumption, but otherwise, the $\snr$ can be estimated well from the data for almost all $\bt$ (see, implicitly, in~\cite{hastie2022surprises} Section~7).
In the next subsection, we show that the minimization of the upper bound yields an explicit formula for the optimal batch size as a function of the overparametrization ratio $\gamma$ and the $\snr$. 
In particular, it can be analytically verified that the optimal batch size is inversely proportional to both $\gamma$ and $\snr$; more specifically, there is a low-$\snr$ threshold point below which increasing the batch size (after taking $n,p\to \infty$) is always beneficial. This can be seen in Figure~\ref{fig:opt_b}, which plots the optimal batch size for different values of $\snr$ and overparametrization ratio. For further discussion see Subsection~\ref{sec:optimal_batch_size}.

Let us briefly outline the proof of Theorem~\ref{thm:min_norm_risk}. We start with the parameter bias and write it as a recursive relation, where a single new batch is added each time, and its expected contribution to the bias reduction is quantified. In a nutshell, we keep track of the projection of $\bt$ onto the complementary row space of the modified feature vectors from all preceding batches. We then write the batch's contribution as a function of the inner products between the (random) batch's basis vectors and the basis of that space (Lemma~\ref{lem:proj_clos_to_fx}). We show that this collection of inner products is close in Wasserstein distance to a Gaussian vector with independent entries (Corollary~\ref{lem:Wasser_distance_XY}), and derive an explicit recursive rule for the bias as a function of the number of batches processed, under this approximated Gaussian distribution (Lemma~\ref{lem:Efy}). This function is then shown to be Lipschitz in a region where most of the distribution is concentrated, which facilitates the use of Wasserstein duality to show that the recursion rule is asymptotically correct under the true distribution (Lemma~\ref{lem:fx_close_to_fy}). Finally, we convert the recursive rule into a certain differential equation, whose solution yields the bias bound (Lemma~\ref{lem:semi_noisy_projection}). This is done in Subsection~\ref{sec:semi_noisy}.

To calculate the variance, we note that the $j$th modified sample $\Ytag_j$, features vector $\h'_j$, and noise $W'_j$, are all linear combinations of the corresponding batch elements, $\X_j$, $\H_j$ and $\W_j$, with the same (random) coefficients. Moreover, these coefficients converge a.s.~to the original samples $\X_j$. We use this to show that the variance converges to that of a Gaussian mixture noise with $\chi^2$-distributed  weights that is projected onto the rows of a Wishart matrix. This is done in Subsection~\ref{sec:noisy}.

 \subsection{Optimal Batch Size: Discussion}\label{sec:optimal_batch_size}

 As discussed before, given a pair $(\gamma,\xi)$ one can use standard function analysis tools to find the batch size $b$ that minimizes the upper bound in Theorem~\ref{thm:min_norm_risk}. 
Then, the optimal batch size will be given by
 \begin{align}
     b_\text{opt}(\gamma,\xi) = \argmin_{b\in\{1,\infty,t\}} \mathrm{UB}(b,\gamma,\xi).
 \end{align}
 where  $\mathrm{UB}(b,\gamma,\xi)$ denotes the upper bound \eqref{eq:UB}, and  
  \begin{align}
     t =\begin{cases}
        \max\{ b_1,1\},\;& \text{if $\mathcal{A}_1$ or $\mathcal{A}_2$}\\
         \max\{ b_2,1\},\;&\text{otherwise},
     \end{cases}
 \end{align}
 with 
 \begin{align}
     \mathcal{A}_1 &= c(\gamma,\xi)<0 \text{ and } b_1<b_2,\\
     \mathcal{A}_2 &= c(\gamma,\xi)>0 \text{ and } b_1>b_2,
 \end{align}
 and
\begin{align}
     b_1 &= \frac{\biggl(\splitdfrac{2\xi^2\gamma^2 - \xi^3\gamma^2 + \xi\gamma - \xi\gamma^2 - 3\xi^2\gamma}{ + 2\xi^3\gamma - \xi^4\gamma + \xi^2 - 2\xi^3 + \xi^4 }\biggr) }{\biggl(\splitdfrac{- \xi^4\gamma + \xi^4 - 2\xi^3\gamma^2 + 3\xi^3\gamma - 2\xi^3 + 2\xi^2\gamma^2}{ - 4\xi^2\gamma + \xi^2 - 2\xi\gamma^2 + 2\xi\gamma + \gamma^2}\biggr)} \\ 
     &\quad + \frac{\sqrt{-\xi(\xi - 1)(\xi\gamma - \xi + 1)(\xi + \gamma - \xi\gamma)^3} }{\biggl(\splitdfrac{- \xi^4\gamma + \xi^4 - 2\xi^3\gamma^2 + 3\xi^3\gamma - 2\xi^3 + 2\xi^2\gamma^2}{ - 4\xi^2\gamma + \xi^2 - 2\xi\gamma^2 + 2\xi\gamma + \gamma^2}\biggr)},\\
    b_2 &=  \frac{\biggl(\splitdfrac{2\xi^2\gamma^2 - \xi^3\gamma^2 + \xi\gamma - \xi\gamma^2 - 3\xi^2\gamma}{ + 2\xi^3\gamma - \xi^4\gamma + \xi^2 - 2\xi^3 + \xi^4 }\biggr) }{\biggl(\splitdfrac{- \xi^4\gamma + \xi^4 - 2\xi^3\gamma^2 + 3\xi^3\gamma - 2\xi^3 + 2\xi^2\gamma^2}{ - 4\xi^2\gamma + \xi^2 - 2\xi\gamma^2 + 2\xi\gamma + \gamma^2}\biggr)} \\ 
     &\quad - \frac{\sqrt{-\xi(\xi - 1)(\xi\gamma - \xi + 1)(\xi + \gamma - \xi\gamma)^3} }{\biggl(\splitdfrac{- \xi^4\gamma + \xi^4 - 2\xi^3\gamma^2 + 3\xi^3\gamma - 2\xi^3 + 2\xi^2\gamma^2}{ - 4\xi^2\gamma + \xi^2 - 2\xi\gamma^2 + 2\xi\gamma + \gamma^2}\biggr)},
 \end{align} 
and
\begin{align}
         c(\gamma,\xi) & = -\xi^4(1-\gamma)-\xi^3(-2\gamma^2+3\gamma-2)\\
         &\;\;\;-\xi^2(2\gamma^2-4\gamma+1)-\xi(2\gamma-2\gamma^2)-\gamma^2.
 \end{align}
Figure~\ref{fig:opt_b} demonstrates the optimal batch size as a function of $\gamma$ and $\xi$. 
 To understand the asymptote at $\xi=0.6478$ to which the curves in Figure~\ref{fig:opt_b}-(a) converge as $\gamma$ grows, note that, after some mathematical manipulations, we get  
 \begin{align}
     \mathsf{sign}&\left(\frac{d }{db}\mathrm{UB}(b,\gamma,\xi)\right) =\mathsf{sign}\Big( \gamma((2\xi^3-2\xi^2+2\xi-1)b-2\xi^2+4\xi-2) +o(\gamma) \Big)\\
     &\underset{\gamma\rightarrow\infty}{\longrightarrow} \mathsf{sign}\Big((2\xi^3-2\xi^2+2\xi-1)b-2\xi^2+4\xi-2\Big).
 \end{align}
 The term $2\xi^2+4\xi-2$ is negative for any $\xi\in(0,1)$ and the polynomial $2\xi^3-2\xi^2+2\xi-1$ has a unique real root at $\xi=0.6478$. Therefore, for any $\xi\leq0.6478$ the derivative is negative for any $b\geq 1$ hence the upper bound $\mathrm{UB}(b,\gamma,\xi)$ is monotonically decreasing with $b$. 
 For $\xi>0.6478$, the minimal risk will be attained at some $1\leq b\leq \infty$, as a function of $\xi$.
 
 Figure~\ref{fig:opt_risk} shows the upper bound on the risk of the \ac{batch-min-norm} algorithm with optimized batch size.  
 For moderate and high  $\snr$ the upper bound is monotonically increasing with the overparametrization ratio $\gamma$. This behavior is what we would expect from a ``good'' algorithm in this setting -- the more overparametrized the problem is, the larger the risk (as is the case, e.g., also with optimal ridge). This should be contrasted with the behavior of the \ac{min-norm} estimator, whose risk explodes around $\gamma=1$, reflecting the stabilizing effect of the batch partition.

\begin{figure*}[ht]
\centering
 \begin{minipage}{.5\textwidth}
  \centering
  \includegraphics[width=1\linewidth]{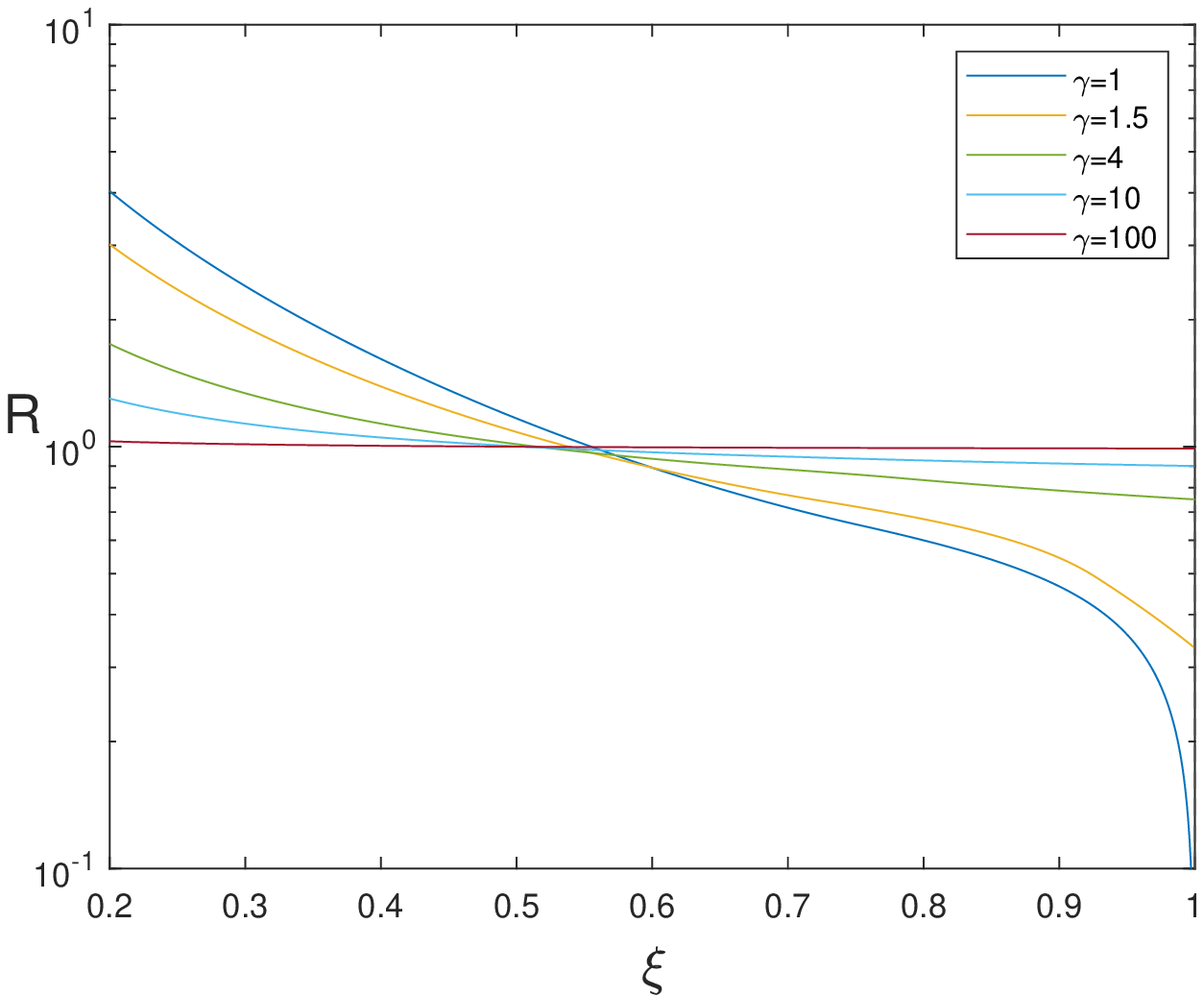}
  \caption*{(a)}
\end{minipage}%
\begin{minipage}{.5\textwidth}
  \centering
  \includegraphics[width=1\linewidth]{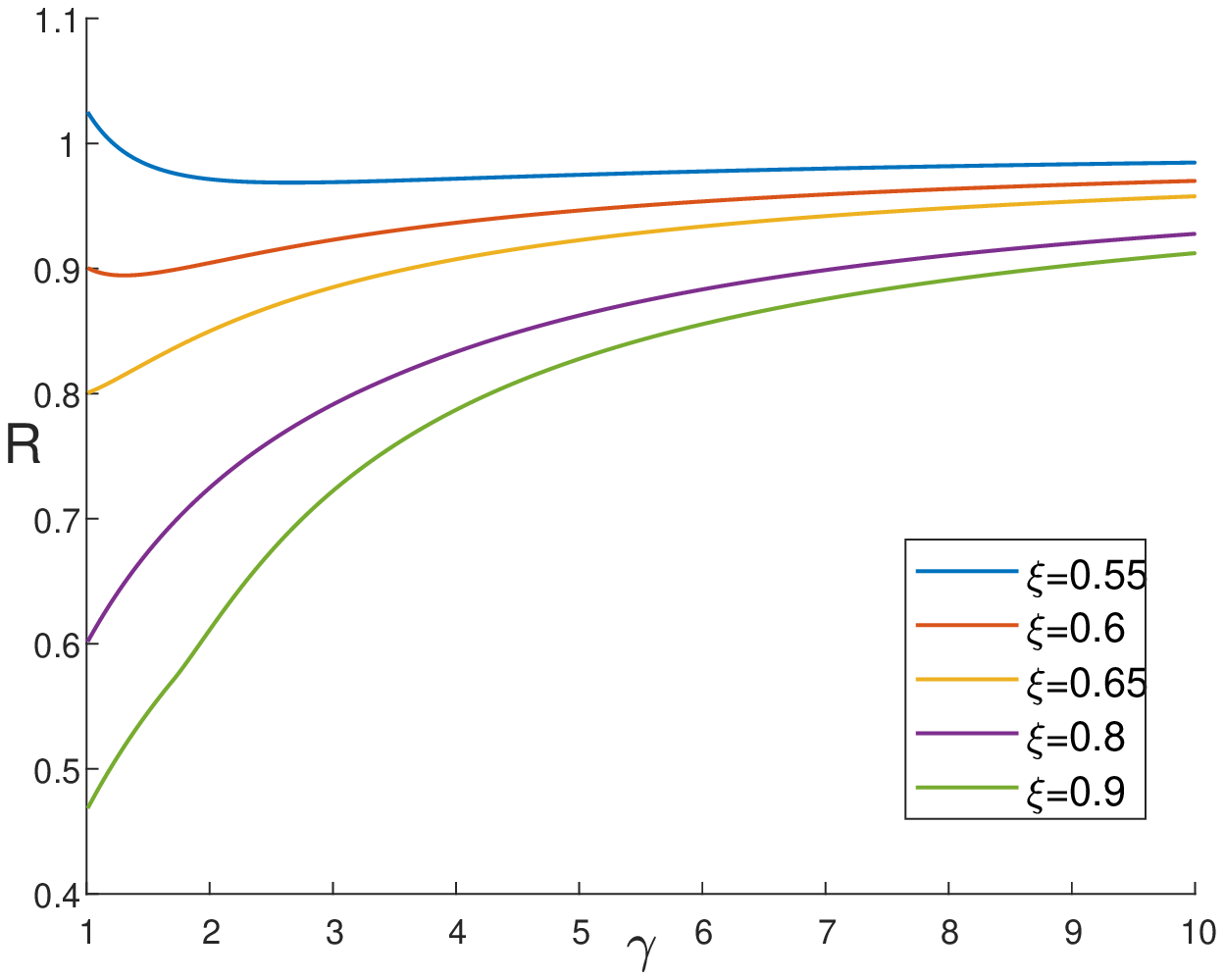}
  \caption*{(b)}
\end{minipage}
  \caption{\textit{Upper bound on the risk of the \ac{batch-min-norm} estimator with  optimized batch size vs. (a) normalized $\snr$ ($\xi$), and (b) overparametrization ratio $\gamma$. At moderate-to-high $\snr$, the bound monotonically increases with $\gamma$.}}\label{fig:opt_risk}
\end{figure*}

 Interestingly, the behavior changes when the $\snr$ is below $\approx 1.83$ ($\xi < 0.6478$). In this region, the optimal batch size $\rightarrow \infty$ and \ac{batch-min-norm} behaves similarly to \ac{min-norm}, producing phenomena such as double-decent and local risk minima (see $\xi=0.55$ and $\xi=0.6$ in Figure~\ref{fig:opt_risk}-(b)). In fact, when $\snr<1$ , ac{batch-min-norm} performs worse than the null estimator, resulting in (normalized) risk $\geq 1$, just like the standard \ac{min-norm}. As we will later see, this happens since the modified noise $\W'$ has a non-zero mean whose magnitude is inversely proportional to $\xi$.  In Section~\ref{sec:de_meaned} we introduce a shrinkage variation of the \ac{batch-min-norm} estimator that effectively eliminates this mean, and results in a risk that is monotonically increasing in $\gamma$ for all $\snr$ levels, and is also uniformly better than both \ac{min-norm} and the null estimator. 
 We further discuss the behavior of both estimators in various settings in Section~\ref{sec:sims}.

Note that our optimal batch size analysis is taken w.r.t the upper bound on the asymptotic risk, i.e., $n,p\rightarrow \infty$. In practice, one uses a finite number of data samples, in which case we need $b\ll n$ in order for $\mathrm{UB}(b,\gamma,\xi)$ to be reliable. Hence, there will be a finite (possibly very large) batch size $b$ that will minimize the risk for any pair $(\gamma,\xi)$.

\section{Proof of Main Result}\label{sec:proof_of_main_result}

\subsection{Asymptotic Parameter Bias} \label{sec:semi_noisy}

In order to estimate the asymptotic parameter bias, we rewrite $\pBias(\hbtB)$ from~\eqref{eq:BMN_risk_bias_var_decomp} as a recursive equation where at the $j$th step we add the $j$th batch $\X_j$, that corresponds to the matrix rows $\H_j$, and update the contribution of this batch to the overall projection.
Recall that 
\begin{align}
    \X_j = \H_j \bt+\W_j= \H_j (\bt+\Z_j),
\end{align} 
with $\Z_j\triangleq \H_j^T(\H_j\H_j^T)^{-1}\W_j$.
Then, the $j$th row in the modified feature matrix $\H'$ is 
\begin{align}
    \h'_j  =\hat{\bt}_j^T=(\bt^T+\Z_j^T) \bs{D}_j,
\end{align} 
with $\D_j$ the projection matrix onto the row space of $\H_j$.
Denote by 
\begin{align}
    \H'_j= [\h'^T_1,\cdots,\h'^T_j]^T,
\end{align} 
the modified feature matrix after the first $j$ steps and let $\L'_j$ be the projection onto the row space of $\H'_j$, that is 
\begin{align}
    \L'_j=  \H'^T_j(\H'_j\H'^T_j)^{-1}\H'_j.
\end{align}
Then, using the Schur complement \cite{zhang2006schur} on $(\H'_j\H'^T_j)^{-1}$ we get
\begin{align}\label{eq:P_j_recursive}
    \|\L'_j\bt&\|^2 = \|\L'_{j-1}\bt\|^2+
\mathbb{E}\left[\frac{((\bt^T+\Z_j^T)\D_j^T\left(\I-\L'_{j-1}\right)\bt)^2}{(\bt^T+\Z_j^T)\D_j(\I-\L'_{j-1})\D_j(\bt+\Z_j)}\right].
\end{align}

It can be seen that at each step the numerator of the update term in \eqref{eq:P_j_recursive} is the projection of the part of $\bt$ that lies in the null space of $\L'_{j-1}$, namely the part of $\bt$ that was not captured by the first $j-1$ rows of $\H'$, onto the row space of the new batch. However, the projection is affected by the noise $\Z_j$ of the current batch. We can view this as a noisy version of the projection $\D_j$, a perspective that will be made clear in the next lemma, which is the key tool for analyzing the recursive rule \eqref{eq:P_j_recursive}.

\begin{lemma}\label{lem:semi_noisy_projection}
Let $\P$ be a projection onto a subspace of dimension $\delta p$ for $\delta \in[0,1]$. Write $\|\bt\| = r$ and $\|\P\bt\|^2 = \alpha r^2$ for $\alpha \in[0,1]$. Let $\{\U_i\}_{i=1}^b$ be uniformly drawn orthonormal vectors, and $\tilde{\D}$ a noisy projection onto the span of $\{\U_i\}$ given by
\begin{align}\label{eq:noisy_proj}
    \tilde{\D}\bs{v} = \D\bs{v} +\sum_{i=1}^b \U_i Z_i,
\end{align} 
where $\D=\sum_{i=1}^b \U_i\U_i^T$,  $Z_i\sim \mathcal{N}(0,T_i\cdot \sigma^2/p)$, and $T_i$ are r.vs. mutually independent of $\{\U_i\}$ and concentrated in the interval $[1-o(1),1+o(1)]$ with probability at least $1-o(1/p)$. Then, the expected squared noisy projection of $\bt$ in the direction of $\P\tilde{\D}\bt$ is given by 
\begin{align}\label{eq:semi_proj_exp}
    \E\left[ \frac{\langle \P\tilde{\D}\bt, \bt\rangle^2}{\|\P\tilde{\D}\bt\|^2}\right] = \frac{1}{p} \left(\frac{\alpha r^2}{\delta}\left(1  + \frac{(b-1)\alpha r^2}{\sigma^2 + r^2} \right) + o(1)\right).
\end{align}
\end{lemma}

The remainder of this section is dedicated to the proof of Lemma~\ref{lem:semi_noisy_projection}, via a Gaussian approximation technique. But first, we use this lemma to prove the bounds on the parameter bias in Theorem~\ref{thm:min_norm_risk}.

\begin{proof}[Proof of bias part in Theorem~\ref{thm:min_norm_risk}.] 
Write 
\begin{align}
    B_j= 1 -\frac{1}{r^2}\|\L'_j\bt\|^2
\end{align} 
for the parameter bias after $j$ steps. Since all of our results are normalized by $r$, we assume from now on that $r=1$.  Then $B_0 =1$, and the desired bias is given by 
\begin{align}
    \lim_{n,p\rightarrow\infty }\pBias\left(\hbtB\right)=\lim_{n,p\rightarrow\infty } \E B_{\frac{n}{b}}.
\end{align} 
Let $\U_i,\cdots ,\U_b$, be the orthonormal vectors that span  the row space of $\H_j$, hence 
\begin{align}
    \D_j=\sum_{i=1}^b \U_i\U_i^T.
\end{align} 
Since the entries of $\H_j$ are i.i.d.~$\mathcal{N}(0,1)$, then $\{\U_i\}$ are uniformly distributed. Moreover, $\D_j(\bt+\Z_j)$ can be written as a noisy projection as defined in the Lemma~\ref{lem:semi_noisy_projection} (see Proposition~\ref{prop:noisy_D} in Appendix~\ref{appen:semi_noisy}).
Denote $\bar{\L}_{j-1}\triangleq\I-\L'_{j-1}$, then using Lemma~\ref{lem:semi_noisy_projection} with $\P= \bar{\L}_{j-1}$ we get 
\begin{align}
    \mathbb{E}\left[\frac{((\bt^T+\Z_j^T)\D_j^T\bar{\L}_{j-1}\bt)^2}{\|\bar{\L}_{j-1}\D_j(\bt+\Z_j)\|^2}\right]&= \mathbb{E}\left[\frac{(\bt^T\tilde{\D}_j^T\bar{\L}_{j-1}\bt)^2}{\bt^T\tilde{\D}_j\bar{\L}_{j-1}\tilde{\D}_j\bt}\right]
    \\
    &\hspace{-2cm}=\E\left[ \frac{B_j(1+(b-1)\xi B_j)}{\frac{p-j}{p}\cdot p}\right] + o(1/p),\label{line:using_the_proj_lemma}
\end{align}
 where we set $\alpha= B_j$ and used the fact that with probability $1$ the dimension of $\bar{\L}_{j-1}$ is $\delta=1-j/p$.
 Then, from \eqref{eq:P_j_recursive} we obtain
 \begin{align}
    \E(B_{j+1}\mid B_j) =B_j - \frac{B_j(1+(b-1)\xi B_j)}{\frac{p-j}{p}\cdot p}+o(1/p),  
\end{align}
and hence 
\begin{align}
    \E B_{j+1} = \E B_j -  \frac{ \E B_j + (b-1)\xi\E B_j^2}{p-j}+o(1/p). 
\end{align}

Define $t_j \triangleq \E B_j$. Then, using the (trivial) bound  $\left(\E B_j\right)^2 \leq \E B_j^2 \leq \E B_j$ that holds for any random variable with support in the unit interval, we obtain
\begin{align}\label{eq:r_j_upper_lower_bounds}
     t_j\cdot \left(1- \frac{1+(b-1)\xi}{p-j}\right) &+o(1/p)\leq t_{j+1} \\
     &\hspace{-1cm}\leq t_j\cdot \left(1- \frac{1+(b-1)\xi t_j}{p-j}\right)+o(1/p). 
\end{align}
We are interested in $t_{n/b}$ under the initial condition $t_0 = 1$, as $n,p\to\infty$ and $p/n = \gamma$ is held fixed.
Noticing that over $n/b$ iterations of the recursive bound \eqref{eq:r_j_upper_lower_bounds} the error term $o(1/p)$ can grow to at most $o(1)$, we drop it hereafter and add it back later.  
Note that for any $j \leq n/b$
\begin{align}
    \frac{d}{dt} t\left(1- \frac{1+(b-1)\xi t}{p-j}\right)&= 1- \frac{1+2(b-1)\xi t}{p-j}
    \geq 1- \frac{1}{p}\cdot \frac{1+2(b-1)\xi}{(1-1/\gamma b)}.
\end{align}
Since the r.h.s~of the above is positive for any large enough $p$, we get that the upper bound in \eqref{eq:r_j_upper_lower_bounds} is monotonically increasing in $t_j$, and therefore we can use it iteratively. Specifically, any sequence of numbers $s_0,\cdots, s_{n/b}$ with $s_0=1$ that obeys the r.h.s inequality in \eqref{eq:r_j_upper_lower_bounds} (replacing $t_j\to s_j$), will dominate the sequence $t_0,\cdots,t_{n/b}$ in the sense that $t_j\leq s_j$, $j=0,\cdots,n/b$.
To find such a sequence, let us rewrite the upper bound \eqref{eq:r_j_upper_lower_bounds} as
\begin{align}\label{eq:difference_ineq}
 \frac{t_{j+1}- t_j}{1/p} \leq \frac{t_j+(b-1)\xi t_j^2}{1-j/p}.
\end{align}
Now, suppose there exists a nonnegative convex function $g$ over $[0,1]$ with $g(0)=1$, satisfying
\begin{align}\label{eq:ODT}
    g'(x) \leq \frac{g(x+1/p)-g(x)}{1/p}\leq -\frac{g(x)\cdot (1+(b-1)\xi g(x))}{1 -x} ,
\end{align}
where the first inequality is by convexity of $g(x)$ and the second translates~\eqref{eq:difference_ineq} to a differential inequality, replacing $t_j$ with $g(x)$ and $j/p$ with $x$. If we can find such a function, then taking $s_j=g(j/p)$ would yield the desired sequence. Let us show this is possible. 
First, rewrite~\eqref{eq:ODT} as 
\begin{align}
    \frac{ g'(x)}{g(x)(1+(b-1)\xi g(x))}& = \frac{d}{dx}\ln\left(\frac{g(x)}{1+(b-1)\xi g(x)}\right)     \leq  -\frac{1}{1-x}. 
\end{align}
The initial condition is $\ln(g(0)/(1+(b-1)\xi g(0))) = -\ln (1+(b-1)\xi )$. Integrating we get
\begin{align}
    \ln \frac{g(x)}{1+(b-1)\xi g(x)}& \leq
    -\ln (1+(b-1)\xi) - \int_0^x\frac{1}{1-x}dx\\
    &=  \ln \frac{1 -x}{1+(b-1)\xi } ,
\end{align}
yielding $g(x) \leq \frac{1-x}{1+(b-1)\xi \cdot x}$, which is indeed nonnegative and convex over $[0,1]$. 
Hence, we have 
\begin{align}
    t_j \leq g(j/p) \leq \frac{1-j/p}{1+(b-1)\xi \cdot j/p} . 
\end{align}
Then, taking the limit with $j=n/b$ and adding back the error term, we get 
\begin{align}
    \lim_{p\to\infty }\frac{1}{r^2}\pBias\left(\hbtB\right)&\leq \lim_{p\to\infty } g(1/(\gamma b)) +o(1)\leq \frac{\gamma b-1}{\gamma b + (b-1)\xi}.
\end{align}

To derive the lower bound  note that \eqref{eq:r_j_upper_lower_bounds} 
 yields 
\begin{align}
    \ln t_j  &\geq \sum_{\ell=0}^{j-1} \ln\left(1-\frac{1+(b-1)\xi}{p-\ell}\right) 
    \\&= -\sum_{\ell=0}^{j-1}\left(\frac{1+(b-1)\xi}{p-\ell}\right)
    \\&= -\sum_{\ell=0}^{j-1}\left(\frac{1}{p}\cdot \frac{1+(b-1)\xi}{1-\frac{\ell}{p}}\right)
    \\&\geq  -(1+(b-1)\xi)\int_0^{j/p}\frac{dx}{1-x}
     \\&=(1+(b-1)\xi)\cdot \ln(1-j/p).
\end{align}
After adding back the error term we get
\begin{align}
    t_j \geq (1-j/p)^{1+(b-1)\xi}\cdot e^{o(1)},
\end{align}
and therefore
\begin{align}
    \lim_{p\to\infty } \pBias\left(\hbtB\right)=\lim_{p\to\infty } t_{n/b} \geq \left(1-\frac{1}{\gamma b}\right)^{1+(b-1)\xi}.
\end{align}

\end{proof}

Next, we turn to prove Lemma~\ref{lem:semi_noisy_projection}. 
Let  $\S= \left[ \SS_1,\cdots,\SS_{2b}\right]$ be a random vector given by
\begin{align}\label{eq:X_i_defenition}
\SS_i\triangleq { \begin{cases}  \langle \P \U_i, \bt\rangle,& i=1,\cdots b,\\  \langle (\I-\P)\U_i, \bt\rangle,& i= b+1,\cdots, 2b \end{cases}},
\end{align}
and define the function
\begin{align}\label{eq:f}
     f(\s,\z) \triangleq \frac{(\sum_{i=1}^b  \ss_i\left(\ss_i+ \ss_{b+i}+z_i\right) )^2}{\sum_{i=1}^b \left(\ss_i+ \ss_{b+i}+z_i \right)^2}.      
\end{align}

The outline of the proof for Lemma~\ref{lem:semi_noisy_projection} is as follows. First, in Lemma~\ref{lem:proj_clos_to_fx} we show that the expected squared projection  \eqref{eq:semi_proj_exp} is approximately equal to $\E f(\S,\Z)$, for any noise $\Z$ that is sufficiently close in distribution to i.i.d.~$\mathcal{N}(0,1/p)$ noise. Then, we show that $\E f(\S,\Z)$ can be calculated with good accuracy by replacing the vector $\S$ with a Gaussian vector $\G$ with independent entries that have the same variance as the elements of $\S$. To do so, in Lemma~\ref{lem:fx_close_to_fy} we bound the Wasserstein distance between $\S$ and $\G$ using Corollary~\ref{lem:Wasser_distance_XY}, and show that $f$ is Lipschitz where $\S$ and $\G$ are concentrated, hence $|\E f(\S,\Z)-\E f(\G,\Z)| \lessapprox \dW_1(\S,\G)$.
Then in Lemma~\ref{lem:Efy} we explicitly calculate $\E f(\G,\Z)$ as a (random) weighted sum of MMSE estimators, yielding ~\eqref{eq:semi_proj_exp} and concluding the proof. 

First, we show that indeed $\E f(\S,\Z)$ approximates the mean-squared projection~\eqref{eq:semi_proj_exp}. 

\begin{lemma}\label{lem:proj_clos_to_fx}
    Let $\S$ be given by \eqref{eq:X_i_defenition},  $\Z=[Z_1,\cdots,Z_b]$ as in Lemma~\ref{lem:semi_noisy_projection}, and  $f(\S,\Z)$ as in~\eqref{eq:f}. 
    Then 
    \begin{align}  \label{eq:proj_clos_to_fx}
    \left|  \mathbb{E}\biggl[\frac{1}{\delta}f(\S,\Z)\right]&-\mathbb{E}\left[\frac{ \langle \P\tilde{\D}\bt, \bt\rangle^2}{\|\P\tilde{\D}\bt\|^2}\right]\biggr|
    \leq \frac{1}{\sqrt[4]{p}}\cdot \mathbb{E}\left[\frac{1}{\delta}f(\S,\Z)\right]+o(1/p).
   \end{align}
\end{lemma}
\begin{proof}
    See Appendix~\ref{appen:semi_noisy}.
\end{proof}

The above lemma shows that the expected squared noisy projection in \eqref{eq:semi_proj_exp} is approximately given by $\E f(\S,\Z)$. Unfortunately, calculating $\E f(\S,\Z)$ w.r.t.~the exact statistics of $\S$ is very challenging. Nevertheless, the next corollary, which is a direct result of Theorem~\ref{thm:Chatterjeee_Wasserstein_thm}, shows that for large $p$ the vector $\S$ is close to Gaussian in the Wasserstein distance.  This fact can then be utilized to approximate $\E f(\S,\Z)$. 

\begin{cor}\label{lem:Wasser_distance_XY}
Let $\S\in\mathbb{R}^{2b}$ be the random vector given in \eqref{eq:X_i_defenition},
and let $\G\in\mathbb{R}^{2b}$ have independent entries distributed as
\begin{align}\label{eq:Y_i_defenition}
{\GG_i\sim \begin{cases}  \mathcal{N}(0,\frac{\alpha}{p}),& i=1,\cdots b,\\  \mathcal{N}(0,\frac{1-\alpha}{p}),& i= b+1,\cdots, 2b \end{cases}.}
\end{align}
Then $\dW_1(\S,\G) \leq \sqrt{\frac{b}{p}}\cdot\frac{2\sqrt{2}b}{p-1}$. 
\end{cor}
\begin{proof}
    See Appendix~\ref{appen:semi_noisy}.
\end{proof}

We established that $\S$ in~\eqref{eq:X_i_defenition} is $O(p^{-3/2})$-close in Wasserstein distance to $\G$ in~\eqref{eq:Y_i_defenition}. If $f$ was $k$-Lipschitz in $\s$, this would yield a $O(k p^{-3/2})$ approximation for $\E f(\S,\Z)$, by calculating the latter using the Gaussian statistics. This is however not the case, since $f$'s gradient diverges along certain curves. Moreover, $f$ is also a function of the noise $\Z$. Nonetheless, we will show that $f$ is Lipschitz in the region where $\S$ and $\G$ are concentrated, which along with the fact that $\Z$ is independent of $\S$ and $\G$, will yield an upper bound on $|\E f(\S,\Z)-\E f(\G,\Z)|$.

\begin{lemma}\label{lem:fx_close_to_fy}
Let $f:\mathbb{R}^{2b}\times \mathbb{R}^b\rightarrow \mathbb{R}$ be defined in \eqref{eq:f},  $\S$ and $\Z$ defined in \eqref{eq:X_i_defenition} and Lemma~\ref{lem:semi_noisy_projection}, respectively, and $\G$ be distributed as in \eqref{eq:Y_i_defenition} and independent of $\Z$. 
Then, 
\begin{align}
\left|\mathbb{E}f(\S,\Z)-\mathbb{E}f(\G,\Z)\right| =o\left(1/p\right).
\end{align}
\end{lemma}
\begin{proof}
    See Appendix~\ref{appen:semi_noisy}
\end{proof}

\begin{lemma}\label{lem:Efy}
    Let $\G$ be as in Lemma~\ref{lem:Wasser_distance_XY} and $\Z$ as in Lemma~\ref{lem:semi_noisy_projection}, then
    \begin{align}
        \mathbb{E}f(\G,\Z)= \frac{\alpha r^2}{ p}\left(1  + (b-1)\xi \alpha\right) + o(1/p).
    \end{align}
\end{lemma}
\begin{proof}
    See Appendix~\ref{appen:semi_noisy}
\end{proof}

We are now ready to prove the main lemma of this section. 
\begin{proof}[Proof of Lemma~\ref{lem:semi_noisy_projection}.]
    From Lemma~\ref{lem:proj_clos_to_fx} and Lemma~\ref{lem:fx_close_to_fy} we get that 
          \begin{align}  
 \left|  \mathbb{E}\left[\frac{ \langle \P\D\bt, \bt\rangle^2}{\|\P\D\bt\|^2}\right]-\mathbb{E}\left[\frac{1}{\delta}f(\G,\Z)\right]\right|\\ \leq \frac{2}{\delta\sqrt[4]{p}}\cdot \mathbb{E}\left[f(\G,\Z)\right]+o\left(1/p\right).
   \end{align}
   Plugging in the result of Lemma~\ref{lem:Efy} then yields
 \begin{align}  
 \Big|  \mathbb{E}\left[\frac{ \langle \P\D\bt, \bt\rangle^2}{\|\P\D\bt\|^2}\right]&-\frac{\alpha(1+(b-1)\xi\alpha)}{\delta p}\Big| \leq \\
 &\hspace{-2cm}\frac{2}{\delta \sqrt[4]{p}}\cdot \left(\frac{\alpha(1+(b-1)\xi\alpha)}{\delta p}\right)+o\left(1/p\right)= o\left(1/p\right),
   \end{align}
   which concludes the proof.
\end{proof}

\subsection{Asymptotic Excess Noise}\label{sec:noisy}

We now turn to evaluate the (asymptotic) excess noise term of \ac{batch-min-norm} in~\eqref{eq:BMN_risk_bias_var_decomp}.
To do so, we note that by construction, the modified elements $\Ytag_j$, $\h'_j$ and $W'_j$ are linear combinations of $\X_j$, $\H_j$ and $\W_j$, respectively, with the exact same coefficients. Moreover, these coefficients converge almost surely to the batch samples $\X_j$ as $p\to \infty$. We then use this observation to show that the excess noise $\Noise(\hbtB)$ converges almost surely to the variance of a Gaussian mixture vector multiplied by a large Wishart matrix, for which we derive an explicit asymptotic expression.

Recall that the $j$th modified sample is 
\begin{align}
    \Ytag_j &= \X_j^T(\H_j\H_j^T)^{-1}\H_j\bt + \X_j^T(\H_j\H_j^T)^{-1}\W_j
    =\h'^T_j\bt+W'_j
 \end{align}
 where $\h'_j$ is a linear combination of the $j$th batch rows, with the coefficients 
 \begin{align}\label{eq:A_j}
     \A_j = \X_j^T(\H_j\H_j^T)^{-1}.
 \end{align}
We proceed to analyze each batch separately.  
Denote the matrix rows in the $j$th batch by $\{\h_{ij}\}_{i=1}^b$, the noise elements by $\{W_{ij}\}_{i=1}^b$, and the samples by $\{\YY_{ij}\}_{i=1}^b$. 
The modified feature vector resulting from the batch is 
\begin{align}
    \h'_j =  \sum_{i=1}^b A_{ij} \h_{ij}.
\end{align}
Note that $\h'_j$ are the weak min-norm per-batch estimators, constituting the rows of the matrix $\H'$. 
The modified sample  is then 
\begin{align}
    \Ytag_j = \sum_{i=1}^b A_{ij}\YY_{ij} 
    =  \sum_{i=1}^b A_{ij} \cdot (\langle\h_{ij}, \bt\rangle  + W_{ij}),
\end{align} 
and  the associated modified noise is 
\begin{align}
    W'_j =  \sum_{i=1}^b A_{ij} W_{ij}.
\end{align} 

In what follows, we assume w.l.o.g.~that $\bt=[r,0\cdots,0]$, as our algorithm is invariant to rotations of $\bt$. The key observation now is that the inverse Wishart matrix $(\frac{1}{p} \H_j\H_j^T)^{-1}$ converges almost surely to identity, due to the strong law of large numbers. Denoting  $\alpha_{ij}$ as the first entry of the vector $\h_{ij}$, this asymptotically yields 
\begin{align}
    p A_{ij}= \YY_{ij}= r\cdot \alpha_{ij}+W_{ij},
\end{align} 
i.e., independent coefficients. In this case, given $\X$, the modified feature matrix $\H'$ has i.i.d.~$\mathcal{N}(0,1)$ entries (except its first column), and the noise is independent and Gaussian (although not with zero mean). Using these "clean" statistics, we can evaluate the excess noise term in~\eqref{eq:BMN_risk_bias_var_decomp}, see Lemma~\ref{lem:Q_variance}. However, this is only asymptotically true, thus we need to show that the true excess noise term in~\eqref{eq:BMN_risk_bias_var_decomp} converges to the expected value taken w.r.t.~the "clean" statistics. This is done next, in Lemma~\ref{lem:norm_AS_conv}.  

  \begin{lemma}\label{lem:norm_AS_conv}
    For the $j$th batch, define 
    \begin{align}
        Q_j&= \sum_{i=1}^b \YY_{ij} W_{ij}\\
        \bs{f}_j&=   \sum_{i=1}^b \YY_{ij} \h_{ij}. 
    \end{align}  
        Let  $\F$ be the $n/b\times p$ matrix with $\{\bs{f}_j\}$ as its rows, and $\Q=[Q_1,\cdots,Q_{n/b}]^T$. 
    Then 
   \begin{align}
 \W'^T(\H'\H'^T)^{-1}\W'-\Q^T(\bs{F}\bs{F}^T)^{-1}\Q\inL1 0.
\end{align}
\end{lemma}
In the above, note that $\bs{f}_j$ are random vectors constituting the rows of the random matrix $\F$. 
\begin{proof}
    See Appendix~\ref{appen:variance}.
\end{proof}
An immediate result of Lemma~\ref{lem:norm_AS_conv}
is
   \begin{align}
{    \lim_{p\rightarrow \infty }\E \left[\W'^T(\H'\H'^T)^{-1}\W'-\Q^T(\bs{F}\bs{F}^T)^{-1}\Q\right]= 0.}
\end{align}

\begin{lemma}\label{lem:QFFQ}
Let $\Q$ and $\F$ be as in Lemma~\ref{lem:norm_AS_conv} and  $C_{\gamma,\xi,b}$  as in Theorem~\ref{thm:min_norm_risk}. Then
\begin{align}
 \sigma^2\cdot\frac{b - (b-1)\xi}{\gamma b-1}\cdot C_{\gamma,\xi,b}&\leq \lim_{p\rightarrow \infty }\mathbb{E}\left[\Q^T(\bs{F}\bs{F}^T)^{-1}\Q\right]
 \leq \sigma^2\cdot\frac{b - (b-1)\xi}{\gamma b-1}.
  \end{align}
\end{lemma}
\begin{proof}
    See Appendix~\ref{appen:variance}.
\end{proof}

The proof of the excess noise bound in Theorem~\ref{thm:min_norm_risk} follows directly from Lemmas~\ref{lem:norm_AS_conv} and~\ref{lem:QFFQ},  and thus conclude the proof for our main result in Theorem~\ref{thm:min_norm_risk}.

\subsection{On the Gap between the Bounds} \label{sec:bounds_gap}
In this subsection, we highlight and discuss the loose steps in our bound derivations.
The gap between the upper and lower bound in \eqref{eq:UB} is rooted in two main reasons. In the parameter bias part, the gap between the bounds originates from   \eqref{eq:r_j_upper_lower_bounds}. There, we used the relation $(\E B_j)^2\leq \E B_j^2\leq \E B_j$ (which holds for any random variable on the unit interval), to (recursively) bound $B_j$, the parameter bias after adding $j$ batches. To obtain the upper bound, we used $(\E B_j)^2\leq \E B_j^2$, which is tight only for deterministic random variables. However, the upper bound on the bias appears to be rather tight in numerical experiments. This suggests that perhaps the per-batch bias indeed has vanishing variance (as $n,p\to\infty$) which results in a tight upper bound in \eqref{eq:r_j_upper_lower_bounds}. The lower bound in \eqref{eq:r_j_upper_lower_bounds} was derived using $\E B_j^2\leq \E B_j$, which is tight (for a unit interval support) only for Bernoulli random variables. Hence, even if indeed $B_j$ has vanishing variance, this inequality will not be tight (except if $\lim_{p\to\infty}\E B_j$ happens to be either $0$ or $1$, which is not the case whenever $1<\gamma <\infty$). A natural step toward closing the gap between the upper and lower bounds on the parameter bias is, therefore, to show that $\lim_{p\to\infty}\Var(B_j) =o(1/p)$, which will mean that the upper bound on the per-batch bias in \eqref{eq:r_j_upper_lower_bounds} is (asymptotically) tight.

The gap in the excess noise bounds is rooted in \eqref{eq:QFQ_QHQ}. In this step, we express the product of the (limiting) modified noise $\Q$ with the (limiting) modified features $\F$ as the sum $\Q(\F'\F'^T)^{-1}\Q-\Q\bs{H}\Q$ using the Woodbury identity. The first term, $\Q(\F'\F'^T)^{-1}\Q$, is using the "clean" features $\F'$ obtained from $\F$ by (loosely speaking) removing the ``direction of the parameter'', denoted as $\t$. The second term, $\Q\bs{H}\Q$, depends on the direction $\t$ through
\begin{align}\label{eq:QHQ}    \Q^T\bs{H}\Q&=\frac{\left(\bs{t}^T\left(\F'\F'^T\right)^{-1}\Q\right)^2}{1+\bs{t}^T\left(\F'\F'^T\right)^{-1}\bs{t}}.
\end{align}
To obtain the upper bound we discarded the second term. However, since the modified noise is a function of the observations and therefore of the parameter, it is correlated with $\t$. Hence, the upper bound is looser in settings where this correlation is dominant. For example, when the $\snr$ is low, this term is enhanced by the noise variance. To obtain the lower bound on the excess noise we upper bounded $\eqref{eq:QHQ}$ using Cauchy--Schwarz inequality (see \eqref{eq:CS}). This step is, however, excessive, since although the modified noise is correlated with $\t$, the two are not perfectly aligned. This is particularly meaningful around the interpolation limit, where the lower bound fails to capture the explosion of the excess noise. To improve the lower bound, a direct evaluation of the expected value of \eqref{eq:QHQ} is required. This is, however, challenging due to the non-polynomial dependence on $\t$.

\section{Improved Batch Minimum Norm via Shrinkage}\label{sec:de_meaned}

As we saw in Lemma~\ref{lem:Q_variance}, the (asymptotic) modified noise $\Q$ has a non-zero mean. Specifically, we have that $\E Q_j|\X =(1-\xi)\|\X_j\|^2$. This expected value contributes to the overall parameter bias of the estimator, and it therefore makes sense to remove it from the samples before performing the  \ac{min-norm} step on the modified linear model \eqref{eq:modified_linear_model}. This results in the estimator 
\begin{align}\label{eq:shrunk_batch_estimator}
        \hbtSB \triangleq \H'^T(\H'\H'^T)^{-1}(\Y-\vmuQ).
    \end{align}
    where
    \begin{align}
        \muQ_j = (1-\xi)\|\X_j\|^2, \;\; j= 1,\cdots, n/b.
    \end{align}
  Recalling that the modified samples $Y'_j\to \|\X_j\|^2$, $j=1,\cdots,n/b$, as $n,p\to \infty$ we get that
\begin{align}
    Y'_j-(1-\xi)\|\X_j\|^2\to \xi  Y'_j
\end{align}
and therefore 
   \begin{align}
        \hbtSB \to \xi \hbtB,
    \end{align}
   that is, a shrinkage variation of the \ac{batch-min-norm} estimator \eqref{eq:batch_estimator}, by a factor equal to the Weiner coefficient $\xi$. We, therefore, loosely refer to this estimator hereafter as \ac{s-batch-min-norm}. Note that unlike \ac{batch-min-norm}, \ac{s-batch-min-norm} requires knowing the $\snr$. As already explained in Subsection~\ref{sec:optimal_batch_size}, this is often a reasonable assumption. 
   The risk of \ac{s-batch-min-norm} is given by 
   \begin{align}\label{eq:SBMN_risk_decomp}
        R(\hbtSB) =& \left(1-\frac{1}{r^2}\mathbb{E}\left\|\L'\bt\right\|^2\right) 
        + \frac{1}{r^2}\mathbb{E}\left[(\W'-\vmuQ)^T(\H'\H'^T)^{-1}(\W'-\vmuQ)\right].
\end{align}

    Note that now $\W'-\vmuQ$ has zero mean (asymptotically), and therefore the second term in the above affects only the variance of the estimator.

The next theorem gives upper and lower bounds on $R(\hbtSB)$. 
\begin{theorem}\label{thm:shrink_batch_min_norm_risk}
For any  $\gamma>1/b$, the asymptotic risk of \ac{s-batch-min-norm} is upper bounded by
\begin{align}\label{eq:shrunk_UB}
 \left(1-\frac{1}{\gamma b}\right)^{1+(b-1)\xi} \hspace{-0.1cm}&+  \frac{1-\xi}{\gamma b-1}\cdot C_{\gamma,\xi,b}
 \leq \lim_{p\rightarrow \infty}  R(\hbtSB) 
  \leq \frac{\gamma b-1}{\gamma b + (b-1)\xi} +  \frac{1-\xi}{\gamma b-1},
    \end{align}
    where the first (resp. second) addend upper bounds the asymptotic bias (resp. variance), and $C_{\gamma,\xi,b}$ as in Theorem~\ref{thm:min_norm_risk}.
 \end{theorem}

 The proof of Theorem~\ref{thm:shrink_batch_min_norm_risk} follows the same steps as in Theorem~\ref{thm:min_norm_risk}, with a few small differences. Since the first terms in \eqref{eq:SBMN_risk_decomp} and \eqref{eq:BMN_risk_bias_var_decomp} are identical, the derivation of the bias bound is the same as in Section~\ref{sec:semi_noisy}.  The variance part in \eqref{eq:SBMN_risk_decomp} is bounded using the same technique as the excess noise in \eqref{eq:BMN_risk_bias_var_decomp}, with the following adaptations. Lemma~\ref{lem:norm_AS_conv} is adapted (in a straightforward way) to show that   
   \begin{align}
 (\W'-\vmuQ)^T(\H'\H'^T)^{-1}(\W'-\vmuQ)-\Q'^T(\bs{F}\bs{F}^T)^{-1}\Q'\inL1 0,
\end{align}
with
    \begin{align}
        Q_j'&= \sum_{i=1}^b \YY_{ij} W_{ij}-(1-\xi)\|\X_j\|^2.
    \end{align}  
The limiting noise $\Q'$ has conditional distribution $Q'_j|\X_j\sim\mathcal{N}(0,\frac{\sigma^2 r^2}{\sigma^2+r^2})$ hence  we get 
\begin{align}
    \E (\Q'\Q'^T) =  \frac{\sigma^2 r^2 }{\sigma^2+r^2}\cdot \I. 
\end{align}
Then, following the same steps as the proof of Lemma~\ref{lem:QFQ_converge} we get 
\begin{align}
    \Q'^T\left(\F'\F'^T\right)^{-1}\Q'&\inL1 r^2 \frac{1-\xi}{\gamma b -1},
\end{align}
and the rest of the proof is the same as in Section~\ref{sec:noisy}.

 The optimal batch size of \ac{s-batch-min-norm} can be calculated by minimizing the upper bound in Theorem~\ref{thm:shrink_batch_min_norm_risk}, as was previously done for \ac{batch-min-norm} in Subsection~\ref{sec:optimal_batch_size}. In contrast to \ac{batch-min-norm}, the optimized \ac{s-batch-min-norm} exhibits a stable risk that is monotonically increasing in $\gamma$, and is always at least as good as the null risk, for all $\snr$ levels. This is the result of removing the constant bias term $\E \Q$ from the measurements. Another difference is that for \ac{s-batch-min-norm} the $\snr$ threshold below which the optimal batch size $\to\infty$ is $\xi=0.5$ (in contrast to $\xi = 0.6478$ for \ac{batch-min-norm}). As we will later see in Section~\ref{sec:sims}, at very low $\snr$, the performance of the optimized \ac{s-batch-min-norm} approaches that of optimized Ridge regression. The optimal batch size and its corresponding upper bound on the risk of \ac{s-batch-min-norm} are presented in Figures~\ref{fig:opt_b_shrunk} and~\ref{fig:opt_risk_shrunk}, respectively.  
Extended analysis of the \ac{s-batch-min-norm} estimator can be found in \cite{mythesis}.

\begin{figure*}[htbp]
\begin{minipage}{.5\textwidth}
  \centering
  \includegraphics[width=1\linewidth]{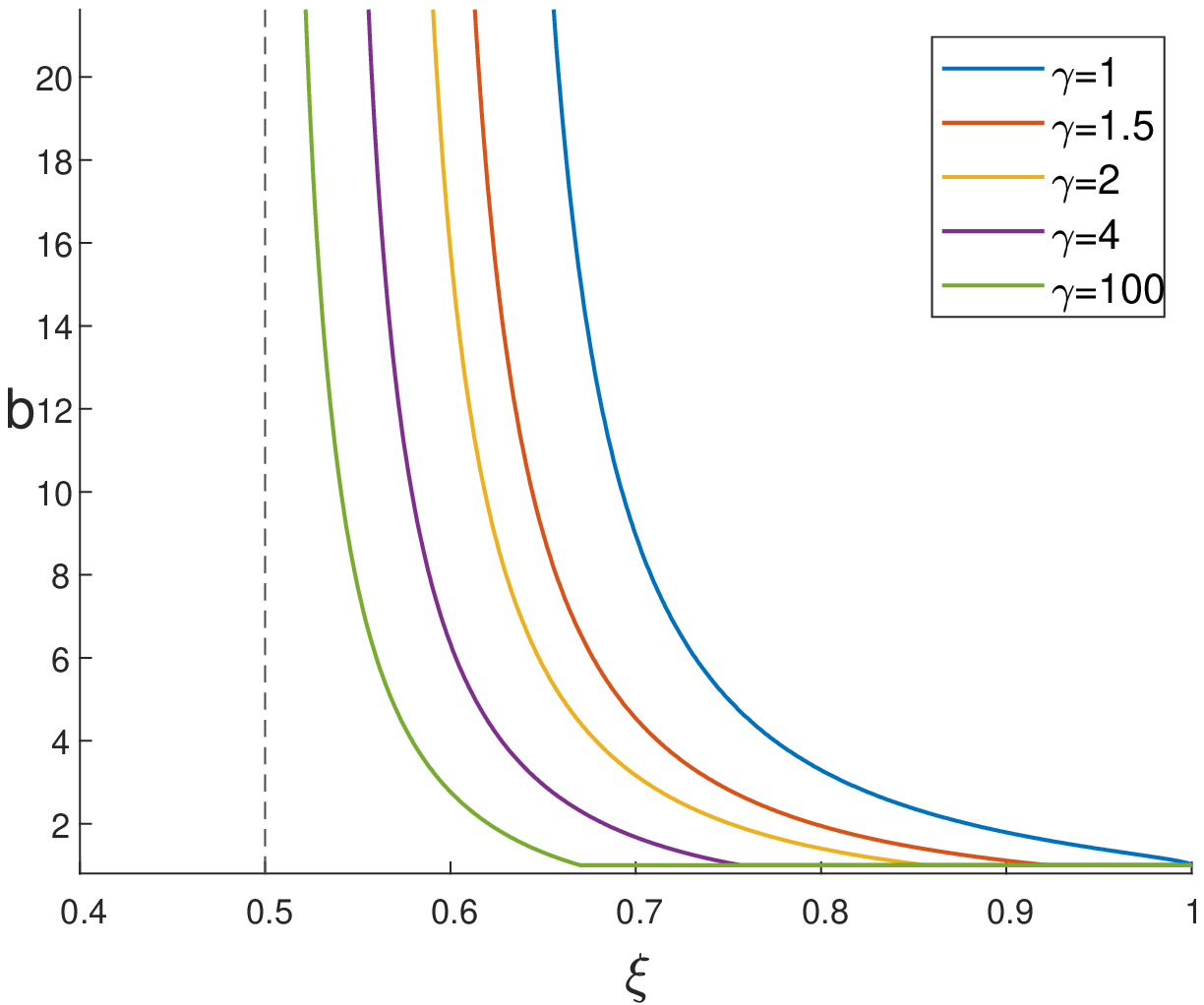}
  \caption*{(a)}
\end{minipage}%
\begin{minipage}{.5\textwidth}
  \centering
  \includegraphics[width=1\linewidth]{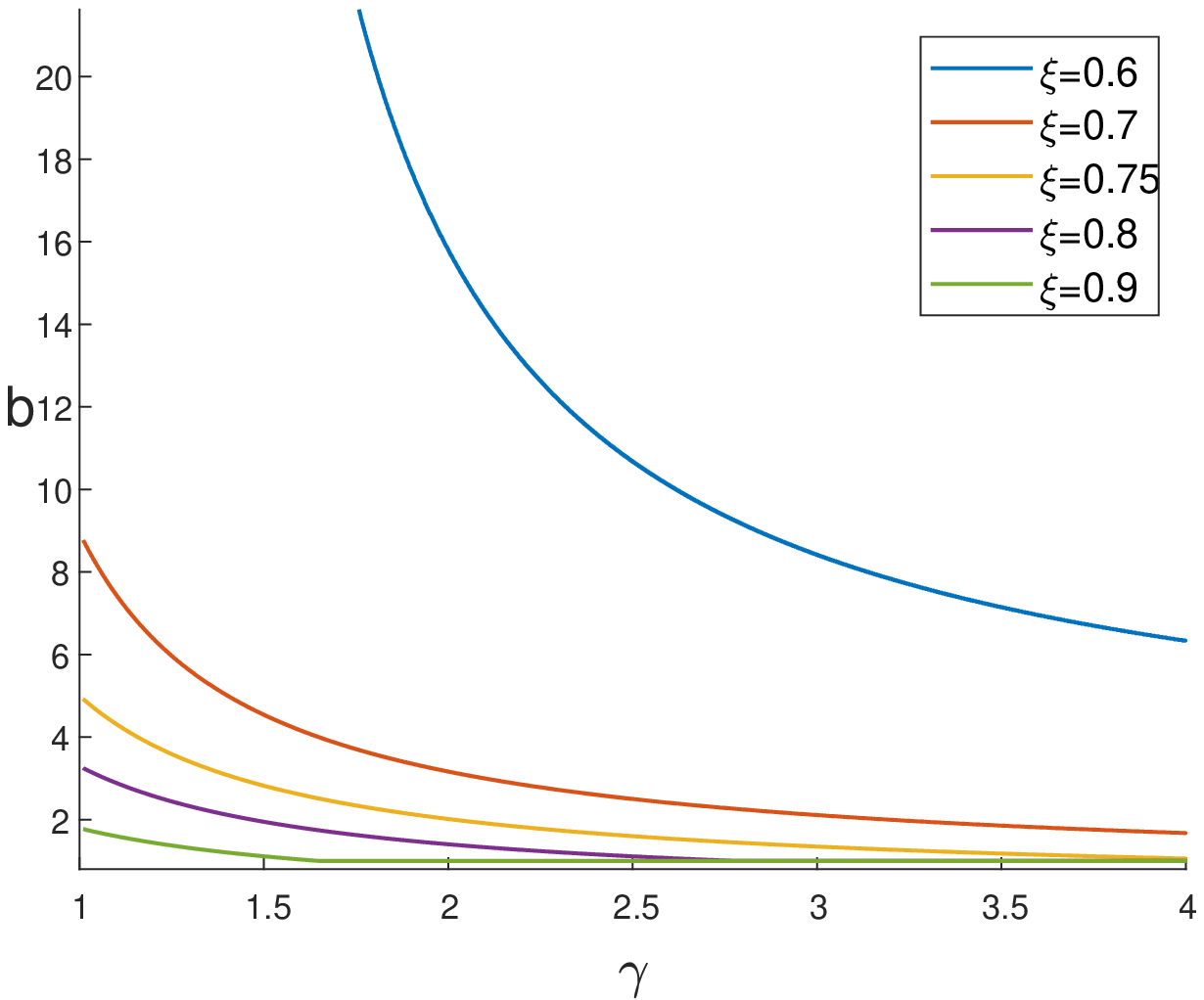}
  \caption*{(b)}
\end{minipage}
  \caption{\textit{Optimal batch size for \ac{s-batch-min-norm} vs. (a) the $\snr$ level $\xi$, and (b) overparametrization ratio $\gamma$. When $\xi<0.5$, the optimal batch size $b\rightarrow \infty$, for any $\gamma>1$.}}\label{fig:opt_b_shrunk} 
\end{figure*}

\begin{figure*}[ht]
\centering
\begin{minipage}{.5\textwidth}
  \centering
  \includegraphics[width=1\linewidth]{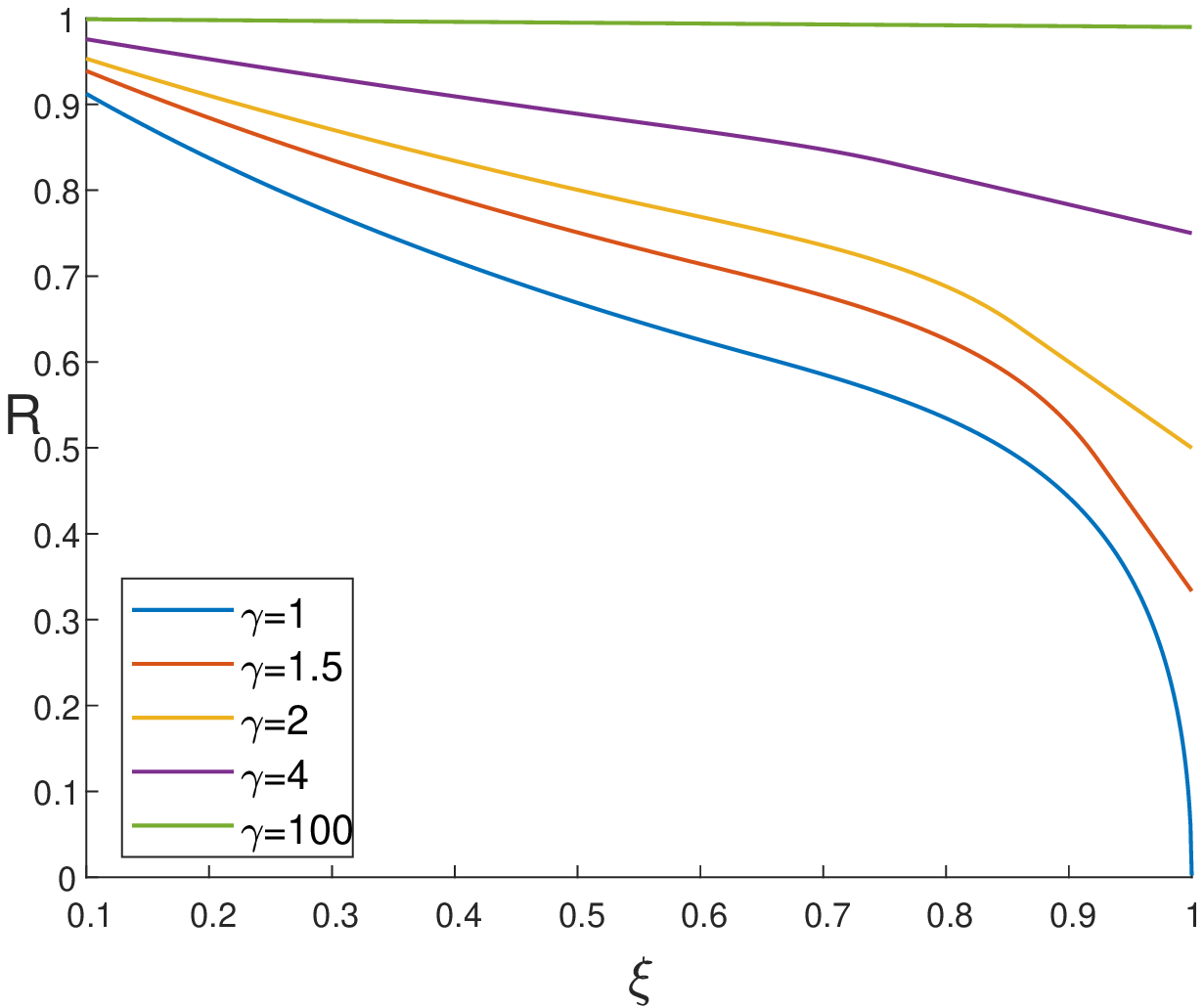}
  \caption*{(a)}
\end{minipage}%
\begin{minipage}{.5\textwidth}
  \centering
  \includegraphics[width=1\linewidth]{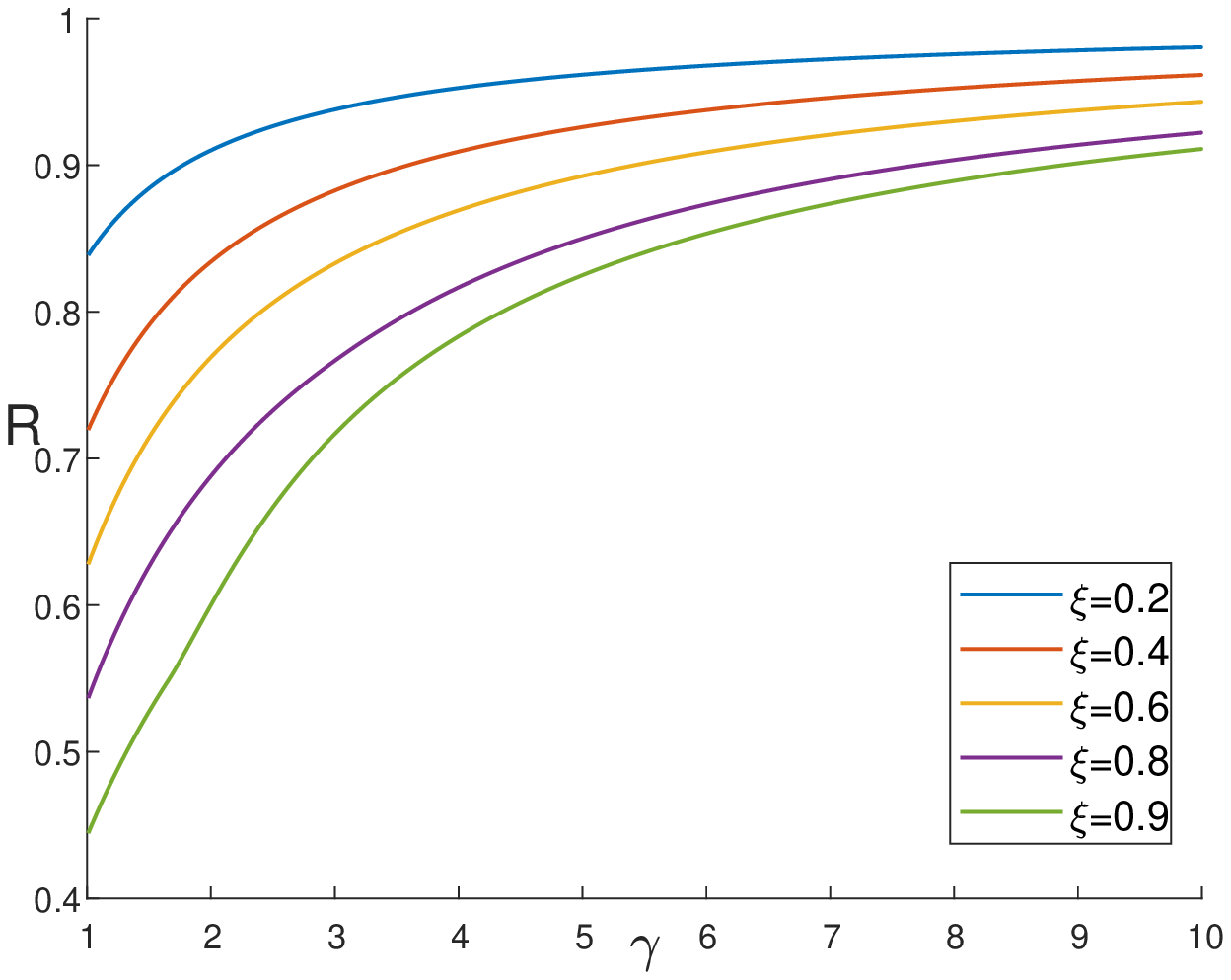}
  \caption*{(b)}
\end{minipage}
  \caption{\textit{Upper bound on the risk of \ac{s-batch-min-norm} with  optimized batch size vs. (a) normalized $\snr$ ($\xi$), and (b) overparametrization ratio $\gamma$. The bound monotonically increases with $\gamma$ for all $\snr$ levels $\xi$.}}\label{fig:opt_risk_shrunk}
\end{figure*}

\subsection{Asymptotic Comparison to Min-norm for $b=1$}
As explained above, the \ac{s-batch-min-norm} estimator is asymptotically equal to a shrinkage variation of \ac{batch-min-norm}. 
When $b=1$, this reduces to a shrinkage \ac{min-norm} estimator
   \begin{align}
        \hbtSB^1 \to \xi\hbtMN,
    \end{align}
    and for $\xi=1$ the two estimators coincide.
    The limiting risk of $\hbtSB^1$ is then 
\begin{align}
    R\left(\hbtSB^1\right) 
    &= \left(1-\frac{2\xi-\xi^2}{r^2}\mathbb{E}\left\|\L\bt\right\|^2\right)+
    \frac{\xi^2}{r^2}\mathbb{E}\left[\W^T(\H \H^T)^{-1}\W\right]\\
    &\longrightarrow 1-\frac{2\xi-\xi^2}{\gamma}+\cdot\frac{1-\xi}{\xi}\cdot\frac{\xi^2}{\gamma-1}.
\end{align}
It is easy to verify by comparing the above to \eqref{eq:asymptotic_MN_risk} that $\hbtSB^1$ has uniformly lower risk than $\hbtMN$ for any $\gamma>1$ and $\xi< 1$. This implies
that \ac{s-batch-min-norm} with the optimal batch size is uniformly better (in terms of quadratic risk) than \ac{min-norm} in the overparametrized regime.

\section{Numerical Results}\label{sec:sims}

Next, we demonstrate different aspects of the \ac{batch-min-norm} estimator (and its shrinkage variation) via numerical experiments.
The setup for the simulations matches the linear model of Section~\ref{sec:problem_formulation}. That is, isotropic feature matrix $\H$ with i.i.d.~Standard Gaussian entries, and i.i.d.~Gaussian noise vector $\W$, independent of $\H$, with zero mean and variance $\sigma^2$ that varies between the different experiments. The parameter vector $\bt$ is normalized to have $\|\bt\|=1$ unless explicitly mentioned otherwise.  In all figures, BMN and SBMN refer to \ac{batch-min-norm} and \ac{s-batch-min-norm}, respectively.

The first experiment compares the performance of the estimators to the asymptotic bounds \eqref{eq:UB} and~\eqref{eq:shrunk_UB}. Figure~\ref{fig:risk_and_UB} depicts 
the risk vs. the overparametrization ratio $\gamma$, for two batch sizes $b=2$ (Subfigures~\ref{fig:risk_and_UB}-(a) and~\ref{fig:risk_and_UB}-(c)) and $b=4$ (Subfigures~\ref{fig:risk_and_UB}-(b) and~\ref{fig:risk_and_UB}-(d)), and three different $\snr$ levels $\xi=0.2,0,6$ and $0.95$. 
The interpolation point for each batch size $b$ corresponds to $\gamma =1/b$. For \ac{batch-min-norm} it can be seen from Subfigures~\ref{fig:risk_and_UB}-(a) and~\ref{fig:risk_and_UB}-(b) that the upper bound \eqref{eq:UB} is tighter for smaller batch sizes.
This is likely because the mean of the modified noise $\Q$ grows with $b$, thus making the term \eqref{eq:QHQ} more dominant as the batch size grows, and therefore the upper bound less tight. On the contrary, for \ac{s-batch-min-norm}, the upper bound \eqref{eq:shrunk_UB} becomes tighter as the batch size grows, as can be seen in Subfigures~\ref{fig:risk_and_UB}-(c) and~\ref{fig:risk_and_UB}-(d). In this case, we reduce the mean of the modified noise, which eliminates the negative effect of the batch size on the variance.  
We further see that for both estimators the upper and lower bounds become tighter as either $\gamma$ or $\xi$ grow. As discussed in Section~\ref{sec:bounds_gap}, when $\gamma\to1/b$, the lower bound becomes trivial. 

It can be further seen in Subfigures~\ref{fig:risk_and_UB}-(a) and~\ref{fig:risk_and_UB}-(b) that when the $\snr$ is low the double-descent phenomenon occurs: the risk of the algorithm decreases as $\gamma$ grows.  This was previously explained by \cite{hastie2022surprises} as the result of the additional degrees of freedom allowing for a solution $\bt$ with a smaller $\ell^2$-norm to the linear model  \eqref{eq:modified_linear_model}. As the $\ell^2$-norm of the estimator decreases, it converges to the null solution. Interestingly, this phenomenon is almost completely eliminated in the \ac{s-batch-min-norm} estimator, whose risk eventually increases with $\gamma$ even for very low $\snr$, as demonstrated in Subfigure~\ref{fig:risk_and_UB}-(d). This is likely because the shrinkage increases the estimator's bias and decreases the noise. This makes the bias, which grows with $\gamma$,  the dominant term in the risk.
Then, a local minima can be observed beyond the interpolation limit (also known as the second bias-variance tradeoff). 
This phenomenon also occurs with the \ac{batch-min-norm} estimator in the high $\snr$ cases. 
\begin{figure*}
\centering
\begin{minipage}{.5\textwidth}
  \centering
  \includegraphics[width=1\linewidth]{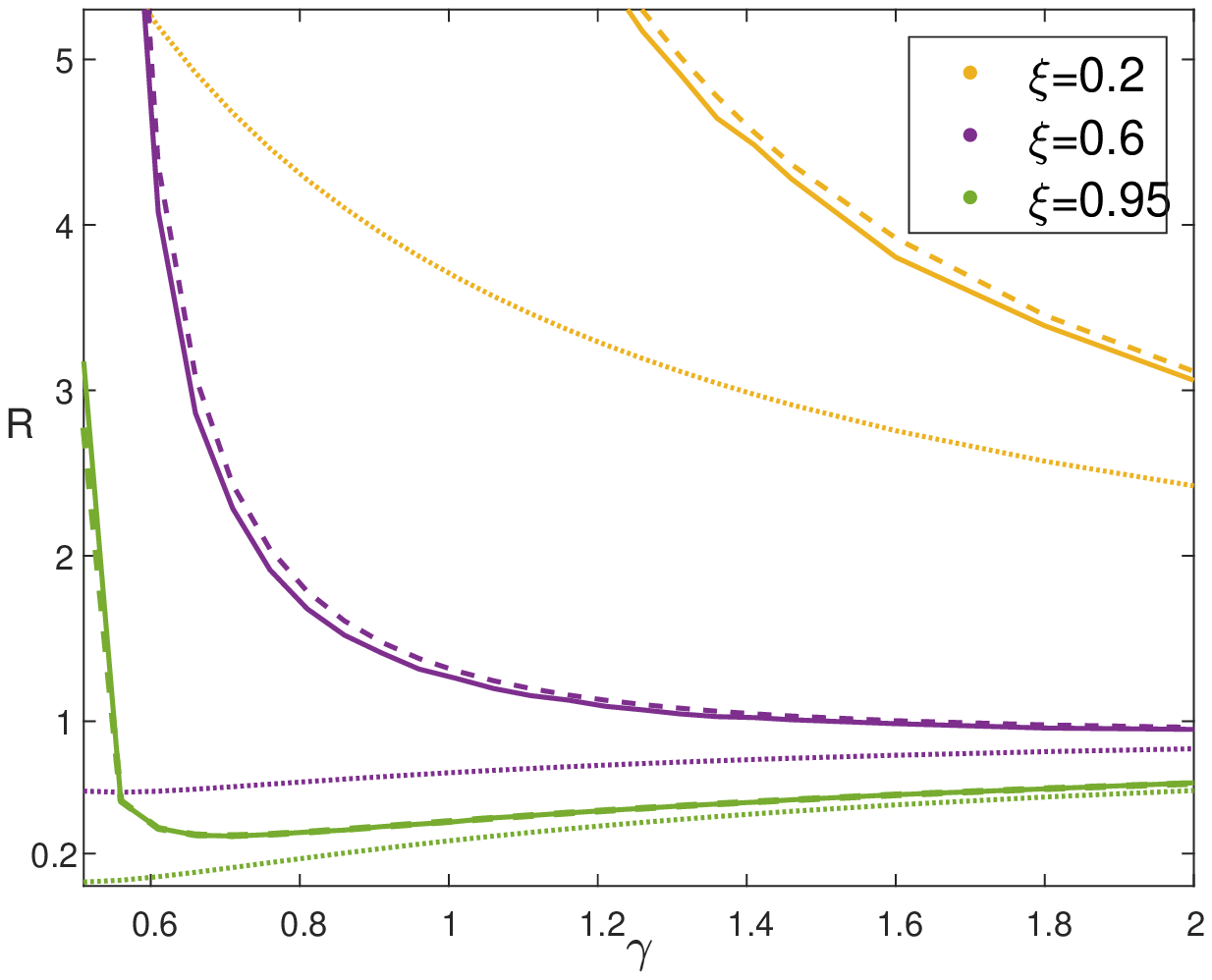}
  \caption*{(a)}
  \includegraphics[width=1\linewidth]{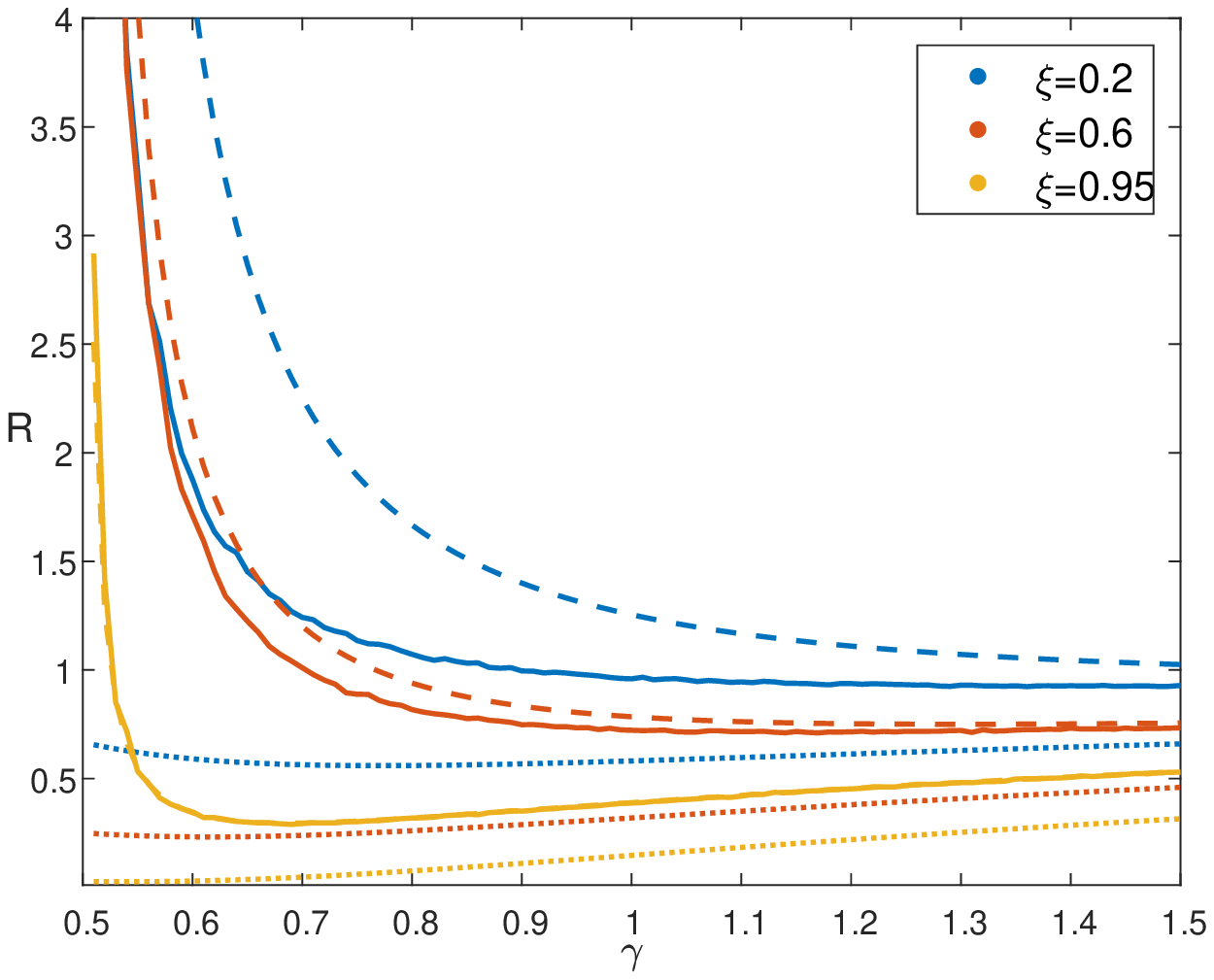}
  \caption*{(c)}
\end{minipage}%
\begin{minipage}{.5\textwidth}
  \centering
  \includegraphics[width=1\linewidth]{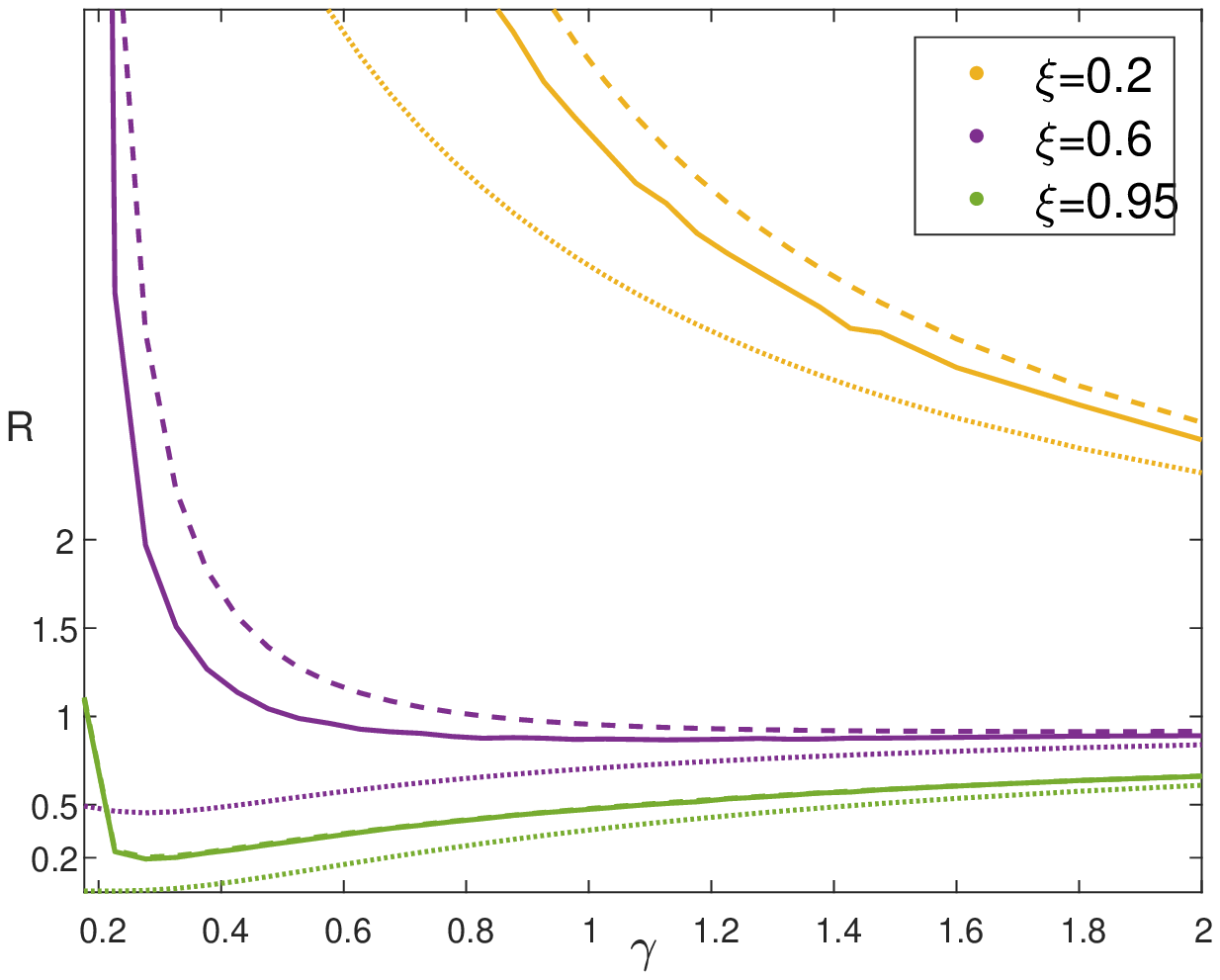}
  \caption*{(b)}
  \includegraphics[width=1\linewidth]{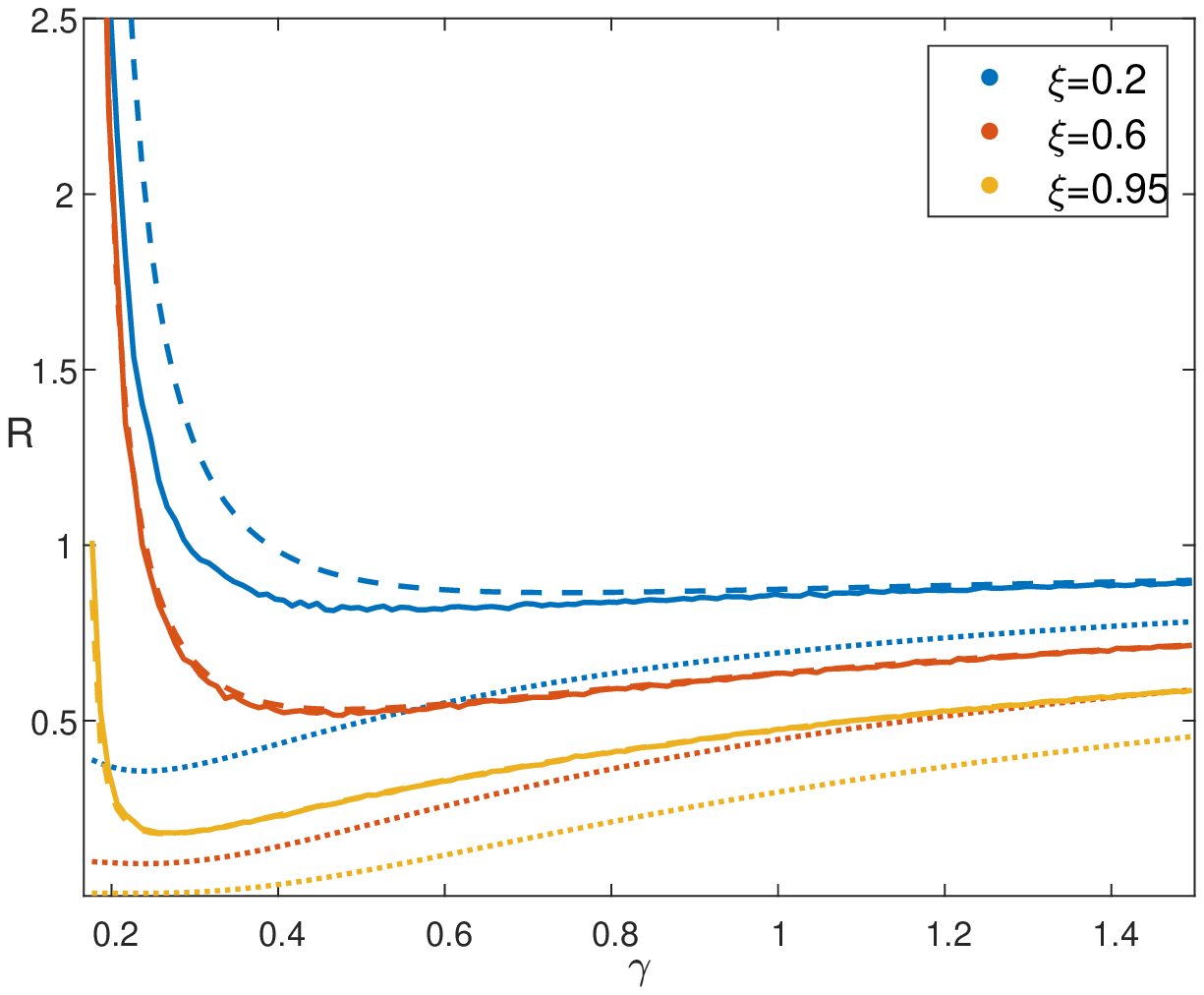}
  \caption*{(d)}
\end{minipage}
 \caption{\textit{Batch-min-norm risk and bounds. The algorithm performance appears as continuous lines, the upper bound as broken lines and the lower bounds are dotted lines. 
 Subfigures~\ref{fig:risk_and_UB}-(a) and \ref{fig:risk_and_UB}-(b) show the risk of \ac{batch-min-norm} alongside the upper and lower bounds of Theorem~\ref{thm:min_norm_risk}, for (a) $b=2$ and (b) $b=6$ . The upper bound is tight for small batch sizes, high $\snr$ levels, and large $\gamma$ values.
 Subfigures~\ref{fig:risk_and_UB}-(c) and \ref{fig:risk_and_UB}-(d) show the risk of \ac{s-batch-min-norm} alongside the upper and lower bounds of Theorem~\ref{thm:shrink_batch_min_norm_risk}, for (c) $b=2$ and (d) $b=6$ . The upper bounds are tighter for larger batch sizes, high $\snr$ levels, and large $\gamma$ values.
For both algorithms the lower bound is tighter for large $\gamma$ values
but as $\gamma\to 1/b$ the bound approaches $0$.
}}\label{fig:risk_and_UB}
\end{figure*}

Next, we study the performance of different estimators as a function of the overparametrization ratio $\gamma$.  The results are presented in Figure~\ref{fig:different_BMN_estimators}.
For each of the methods ( \ac{batch-min-norm} and \ac{s-batch-min-norm}), three estimators are tested, corresponding to batch sizes  $b=1$, $b=2$, and $b=10$. The interpolation point for each of the estimators depends on the batch size. 
For completeness, we plot the risk below the interpolation points $\gamma= 1/ b$ as well, though the analysis in this paper focuses on $\gamma >1/b$. 
The performance of regular \ac{min-norm} is also presented. As expected, the risk of \ac{min-norm} coincides with \ac{batch-min-norm} when $b=1$, since the two are equivalent. 
 It can be seen that the risk of \ac{s-batch-min-norm} is uniformly lower than the corresponding \ac{batch-min-norm} estimator with the same batch size. The gap between the two is especially notable near the interpolation limit $\gamma=1/b$. Again, we see that in the \ac{s-batch-min-norm} case the double-decent phenomenon is eliminated and a second bias-variance tradeoff appears. 
For $b=10$, a local minimum point of the risk in the overparametrized regime appears for both estimators. This is because, as the batch size $b$ increases the number of measurements $n/b$ in the second \ac{min-norm} step of the algorithms decreases. 
Then, although the modified features $\H'$ of the batch algorithm are favorably aligned with the parameter vector, the increased overparametrization ratio of the new problem still results in a larger bias. Therefore, as the weight of the bias in the risk grows (i.e., the $\snr$ grows), the optimal batch size decreases.  Above some $\snr$ threshold (that depends on $\gamma$), the minimal batch size $b=1$ is always preferable. 

\begin{figure*}[ht]
\centering
\includegraphics[width=0.8\textwidth]{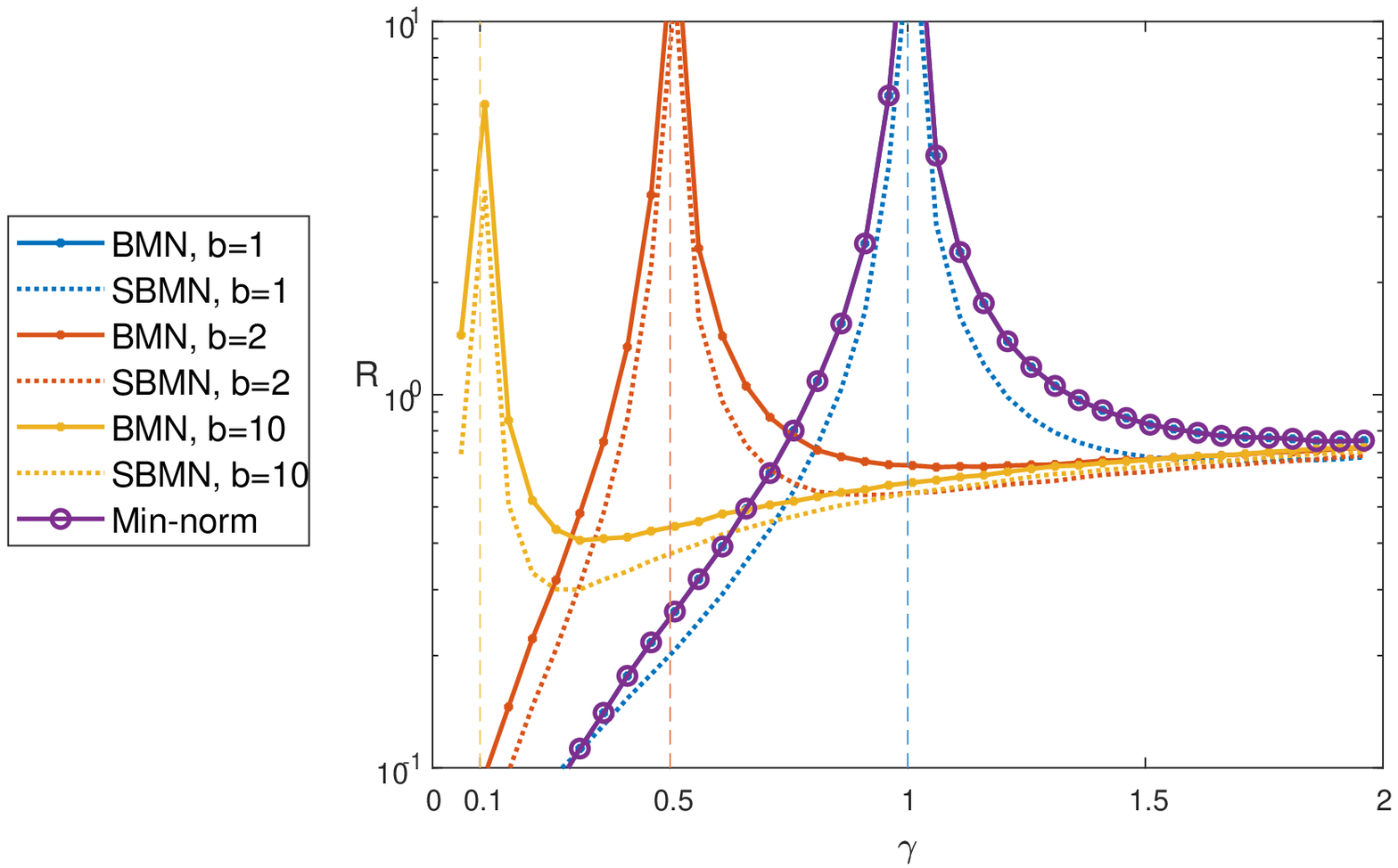} \\
  \caption{\textit{ Risk of \ac{batch-min-norm} vs. overparametrization ratio $\gamma$ with $\xi=0.8$. The interpolation limit is shifted to $\gamma=1/b$ due to the batch partition. For $b=1$ \ac{batch-min-norm} coincides with \ac{min-norm}. For $b=10$, a local minimum point appears for both estimators in the overparametrized regime. This is because, as the batch size $b$ increases, the bias becomes more dominant in the overall risk.}}\label{fig:different_BMN_estimators}
\end{figure*}

The next experiment examines the risk of the \ac{batch-min-norm} estimator vs. the batch size, with $\gamma=2$ and different $\snr$ scenarios. The results are presented in Figure~\ref{fig:b_size2_sim}.
Here we used $n=1000$. The $\snr$ level was controlled by fixing $\sigma^2$ and varying the signal energy $r^2$. 
Recall that in Section~\ref{sec:main_result} the batch size was optimized using the upper bound. We saw that as the $\snr$ decreases the optimal batch size grows. Specifically, below the $\snr$ threshold $\xi=0.6478$, increasing the batch size is always beneficial.  This is demonstrated in Figure~\ref{fig:b_size2_sim} by the two curves corresponding to $\xi=0.5$ and $\xi=0.6$, where the risk is monotonically decreasing (although very slowly) with $b$.
We further see that as discussed in the previous experiment, above some high-$\snr$ threshold that depends on $\gamma$ (in this case the threshold lies somewhere in the interval $(0.8,0.9]$), the minimal batch size $b=1$ is optimal. 

\begin{figure}[ht]
\centering
\hspace{-0.1cm}\includegraphics[width=1.05\columnwidth]{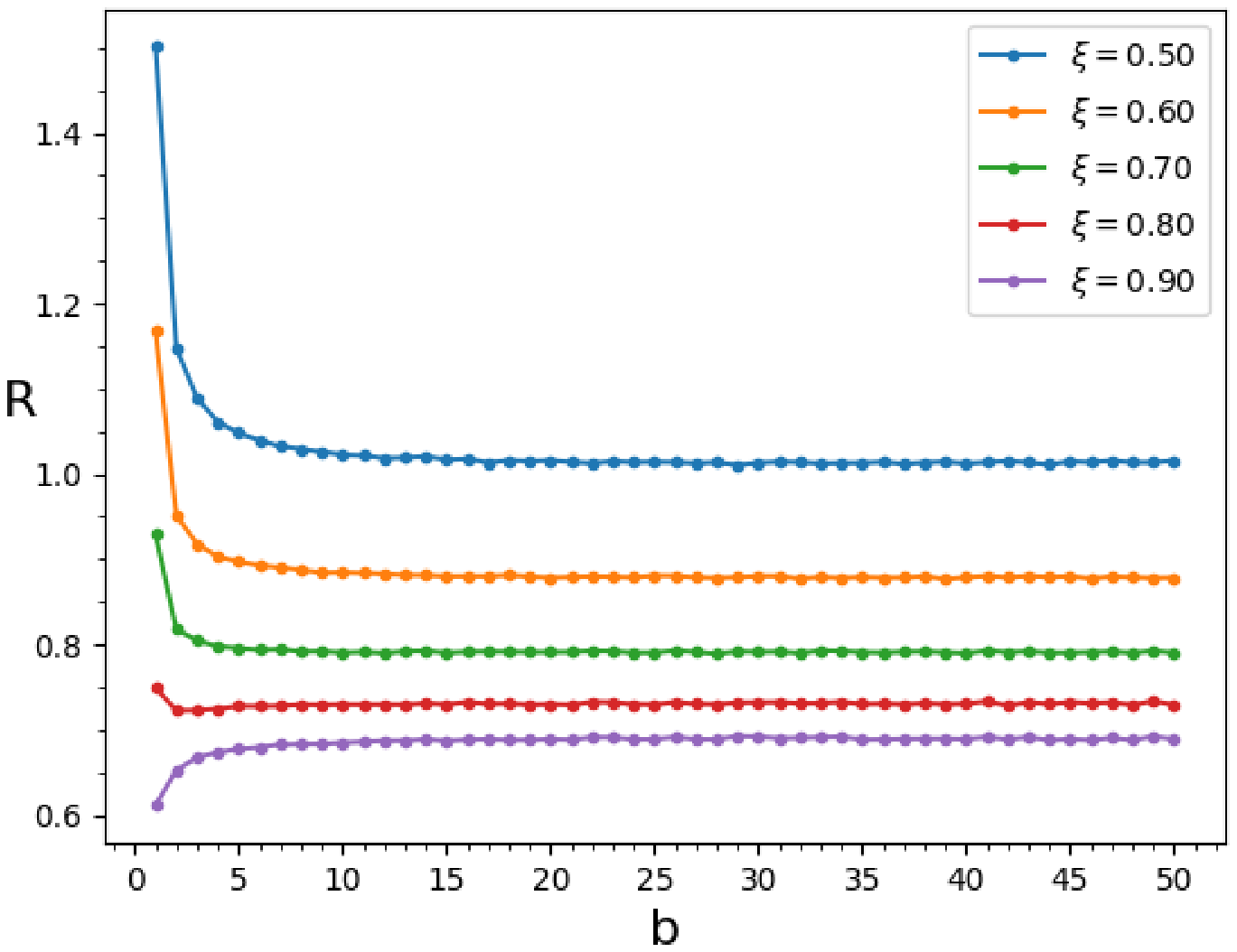}
\caption{\textit{\ac{batch-min-norm} risk vs. batch size for different $\snr$ levels $\xi$, with $n=1000$ and $\gamma=2$. The $\snr$ levels were controlled by fixing $\sigma$ and varying $r$, hence the null risk $r^2$ is different between the scenarios. For $\xi=0.5,0.6$ the risk monotonically decreases with $b$, as discussed in Section~\ref{sec:main_result}. For $\xi = 0.7,0.8$, the optimal batch size is  $b>1$, and for $\xi=0.9$ the minimal risk is obtained at $b=1$. 
}}\label{fig:b_size2_sim}
\end{figure}

Next, we demonstrate the performance of the \ac{batch-min-norm} and \ac{s-batch-min-norm} estimators with optimized batch size.
Figure~\ref{fig:optimized_algs} presents the risk of the optimized estimators vs. (a) normalized $\snr$ $\xi$ with $\gamma= 1.5$, and (b) overparametrization ratio $\gamma$ with $\xi = 0.7$. The estimators are compared with four other algorithms: \ac{min-norm}, \ac{batch-min-norm} with $b=2$, ridge regression with optimal regularization parameter, and a naively stabilized version of \ac{min-norm}, where samples are discarded to obtain the best possible risk (that is, optimizing the overparametrization ratio among all ratios $\geq \gamma$).  
It was previously shown in \cite{hastie2022surprises} that the optimal overparametrization ratio for \ac{min-norm} is 
\begin{align}
    \gamma_{\mathrm{opt}}=\begin{cases}
        \infty, &\;\; \xi\leq 1/2,\\
        \sqrt{\snr}/(\sqrt{\snr}-1), &\;\; \xi> 1/2.
    \end{cases}
\end{align}
Moreover, in the second case, the risk of \ac{min-norm} monotonically increases with $\gamma$ after the point $\gamma_{\mathrm{opt}}$. 
Therefore, if $\gamma < \gamma_{\mathrm{opt}}$ we discard some of the samples to obtain $\gamma_{\mathrm{opt}}$. Otherwise, we keep all the samples.
As expected, optimized \ac{batch-min-norm} achieves lower risk than \ac{min-norm} and \ac{batch-min-norm} with $b=2$, as they are both special cases of \ac{batch-min-norm}. As the $\snr$ grows, the optimal batch size decreases, and the risk of the optimized \ac{batch-min-norm} approaches that of \ac{min-norm}. The two coincide for $\xi\geq 0.93$. The naively stabilized \ac{min-norm} (trivially) outperforms \ac{min-norm} for any $\snr$. In high $\snr$, $\xi\geq 0.5$, (i.e., whenever \ac{min-norm} and \ac{batch-min-norm} can do better than the null risk), optimized \ac{batch-min-norm} is better than optimized \ac{min-norm}. When $\xi \leq 0.5$,  all algorithms except \ac{s-batch-min-norm} and ridge regression are no better than the null estimator. Hence in this range, the optimized \ac{min-norm} outperforms the optimized batch one, as it produces the null estimate. \ac{s-batch-min-norm} has uniformly lower risk than all other algorithms except for optimized ridge.  In low $\snr$ ($\xi\leq 0.4$), optimized \ac{s-batch-min-norm} approaches the performance of optimized ridge.

\begin{figure*}[ht]
 \centering
  \begin{minipage}[b]{0.55\textwidth}
\includegraphics[width=\textwidth]{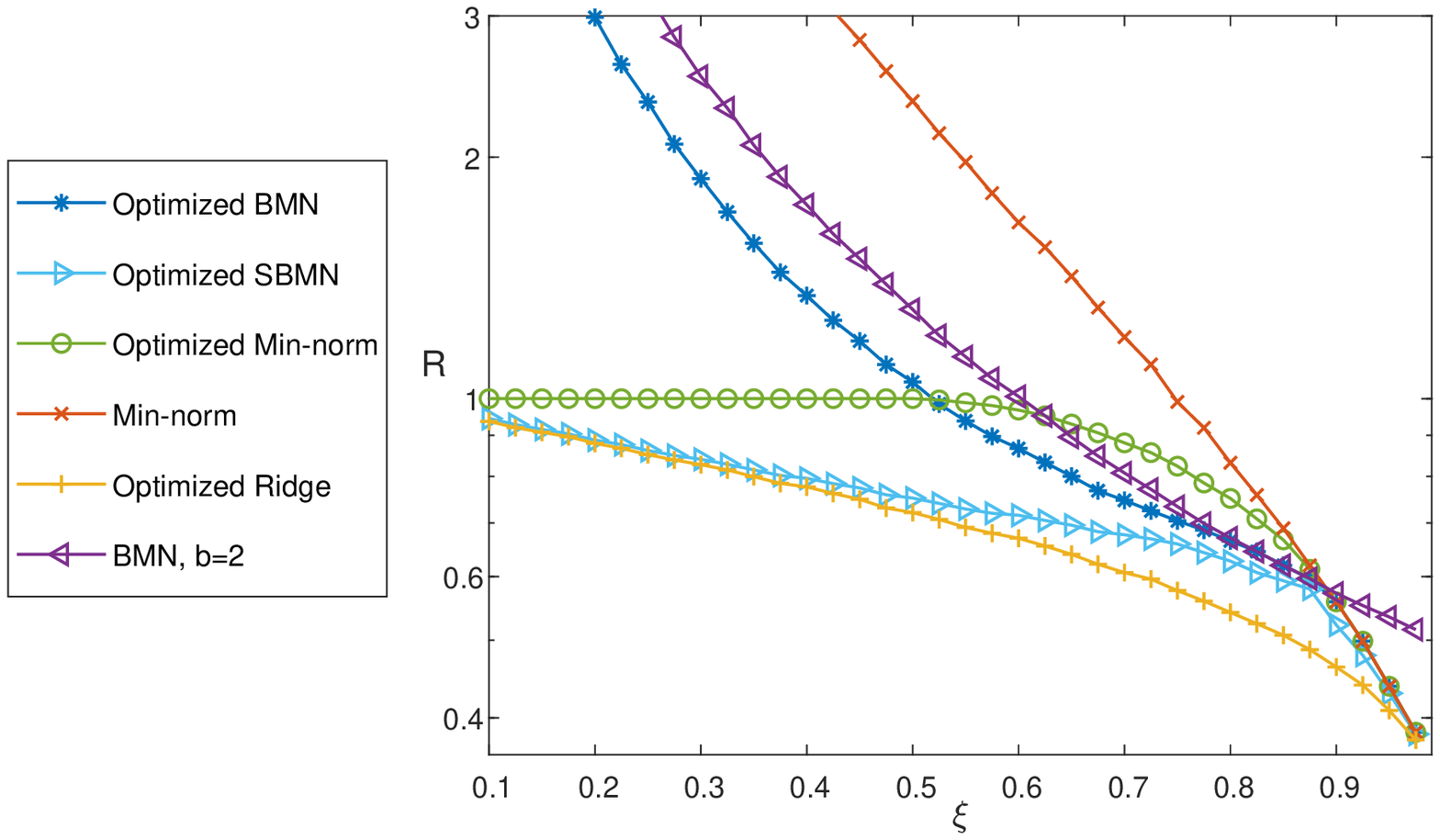}
\caption*{\qquad\qquad(a)}
\end{minipage}
\hfill
\begin{minipage}[b]{0.4\textwidth}
\hspace{-0.5cm} \includegraphics[width=1.05\textwidth]{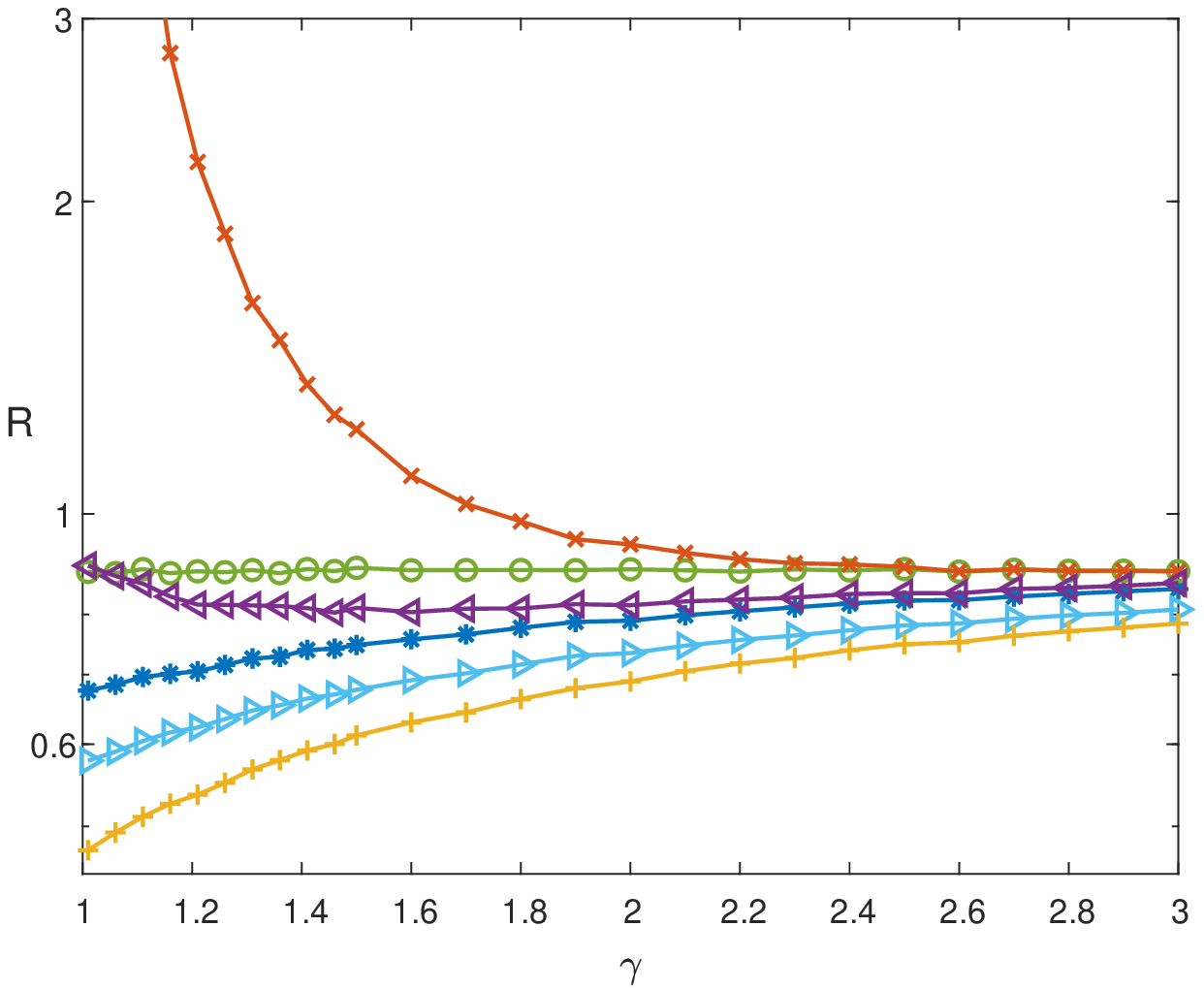}
\caption*{(b)}
\end{minipage}%
\caption{\textit{Risk of \ac{batch-min-norm} with optimal batch size vs. (a) normalized $\snr$ $\xi$, with $\gamma=1.5$, and (b) overparametrization ratio $\gamma$ with $\xi =0.7$. For $\xi\geq 0.5$ ($\snr\geq 1$), optimized \ac{batch-min-norm} outperforms the \ac{min-norm} algorithms. \ac{s-batch-min-norm} outperforms all other algorithms except for optimized ridge, in all settings.}}\label{fig:optimized_algs} 
\nonumber
\end{figure*}

\subsection{Comparison with Server Averaging}

As detailed in Subsection~\ref{sec:remifications}, one application of our algorithm is in distributed linear regression, where a main server distributes the data between multiple workers, and then merges the estimates given by each of the workers. The merging is typically done by simple averaging \cite{zhang2012communication,zhang2015divide,dobriban2020wonder,mucke2022data}. We now compare the performance of \ac{batch-min-norm} applied to this problem, with that of the server-averaging approach, and show that our \ac{s-batch-min-norm} estimator outperforms traditional server-averaging in almost all settings. 
 Here, the server-averaging is done with equal weights assigned to all the per-batch estimators $\hbt_i$ (which is optimal among all the fixed weights that sum up to $1$), that is
\begin{align}
    \hbt_{\mathrm{avg}} \triangleq \frac{b}{n}\sum_{i=1}^{\frac{n}{b}}\hbt_i.
\end{align}
This choice of weights is a common setup, previously considered, e.g., in \cite{mucke2022data}, which also studied distributed \ac{min-norm} linear regression in the overparametrized regime. 

Since all the per-batch estimators are independent and identically distributed, the bias and variance of the server-averaging estimator are given by
\begin{align}\label{eq:avg_bias}
    \mathsf{Bias}\left(\hbt_{\mathrm{avg}}\right)= \frac{b}{n}\sum_{i=1}^{\frac{n}{b}}  \mathsf{Bias}\left(\hbt_i\right) = \mathsf{Bias}\left(\hbt_1\right),
\end{align}
and
\begin{align}\label{eq:avg_var}
    \Var\left(\hbt_{\mathrm{avg}}\right)= \left(\frac{b}{n}\right)^2\sum_{i=1}^{\frac{n}{b}}  \Var\left(\hbt_i\right) = \frac{b}{n}\Var\left(\hbt_1\right).
\end{align}
Note that in \eqref{eq:avg_bias} and \eqref{eq:avg_var} the notations $\mathsf{Bias}$ and $\Var$ in the above refer to the actual bias and variance of the averaging estimator.
For any batch size $b$ that is sub-linear in $n$, the per-batch bias $\mathsf{Bias}(\hbt_i)\rightarrow 1$ as $n,p\rightarrow \infty$, therefore for any fixed batch size the risk of server-averaging tends to the null risk asymptotically.  
For batch sizes that are linear in $n$ it can be easily shown by adapting the result \eqref{eq:asymptotic_MN_risk} to our framework (that is, to the case when the expectation is taken over both the noise and the features) that
\begin{align} \label{eq:avg_risk}
  \lim_{n,p\rightarrow \infty }  R(\hbt_{\mathrm{avg}}) =  \left(1-\tilde{\gamma}^{-1}\right) \frac{\gamma+\tilde{\gamma}(\tilde{\gamma}-1)}{\tilde{\gamma}^2} + \frac{1-\xi}{\xi}\cdot\frac{\gamma}{\tilde{\gamma}}\frac{1}{\tilde{\gamma}-1},
\end{align}
where 
\begin{align}\label{eq:tilde_gamma}
\tilde{\gamma}\triangleq \frac{p}{b}.    
\end{align}
 The batch size that minimizes the above can be found numerically or analytically, for any fixed $(\gamma,\xi)$ pair. Here we drop the analytical expression and optimize the batch size numerically over all integer batch sizes.

Figure~\ref{fig:bmn_vs_avg_gamma} compares the performance of \ac{batch-min-norm} and \ac{s-batch-min-norm} with server-averaging for batch sizes $b=2$ and $b=200$. Subfigures~\ref{fig:bmn_vs_avg_gamma}-(a) and \ref{fig:bmn_vs_avg_gamma}-(b) plot the estimators' risk against $\gamma$ with $\xi=0.75$.  Subfigures~\ref{fig:bmn_vs_avg_gamma}-(c) and \ref{fig:bmn_vs_avg_gamma}-(d) show the risk vs. the $\snr$ $\xi$ with $\gamma=1.2$. Here, we used $n=400$, therefore the batch size $b=200$ corresponds to $n/2$. As expected, for the small batch size $b=2$ the risk of server-averaging is close to the null risk for all $\gamma$ values, due to the large per-batch bias. The effective per-batch overparametrization is  $\tilde{\gamma}_b=\frac{p}{b}=\frac{400\gamma}{b}$, and even for $\gamma=0.55$, we get that $\tilde{\gamma}_2=110$, therefore, the par-batch estimators are very biased. The \ac{batch-min-norm} and \ac{s-batch-min-norm} estimators are much less biased due to the second \ac{min-norm} step. Near the effective interpolation limit $\gamma=1/b$ the risk of both estimators explodes, but when $\gamma$ is sufficiently large the estimator enjoys good noise averaging. For $b=200$, server-averaging outperforms \ac{batch-min-norm} in all but very high $\snr$ levels. \ac{s-batch-min-norm} has lower risk than both \ac{batch-min-norm}
 and server-averaging in all settings except for very near the (effective) interpolation point. 
 It can be seen that the risk of server-averaging with $b=200=n/2$ explodes at $\gamma=1/2$. This is because $\tilde{\gamma}_{200}=\tilde{\gamma}_{n/2}=2\gamma$ and the effective interpolation limit of server-averaging $\tilde{\gamma}=1$ translates to $\gamma=0.5$. 
It can also be seen that the performance of \ac{batch-min-norm} with batch sizes $b$ and 
 $n/b$ are approximately the same (as demonstrated by the \ac{batch-min-norm} curves $b=2$ and $b=n/2$), and the two curves explode at the same point $\gamma=0.5$. This is since every \ac{batch-min-norm} estimator performs two minimum-norm estimation steps, one per-batch with effective overparametrization ratio $\frac{p}{b}$, and the second is joint for all per-batch estimators with effective overparametrization ratio $\frac{p}{n/b}$. 

\begin{figure*}
\centering
  \includegraphics[width=0.2\linewidth]{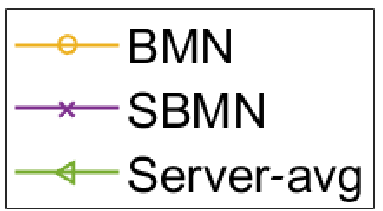}
\begin{minipage}{.5\textwidth}
  \centering
  \includegraphics[width=1\linewidth]{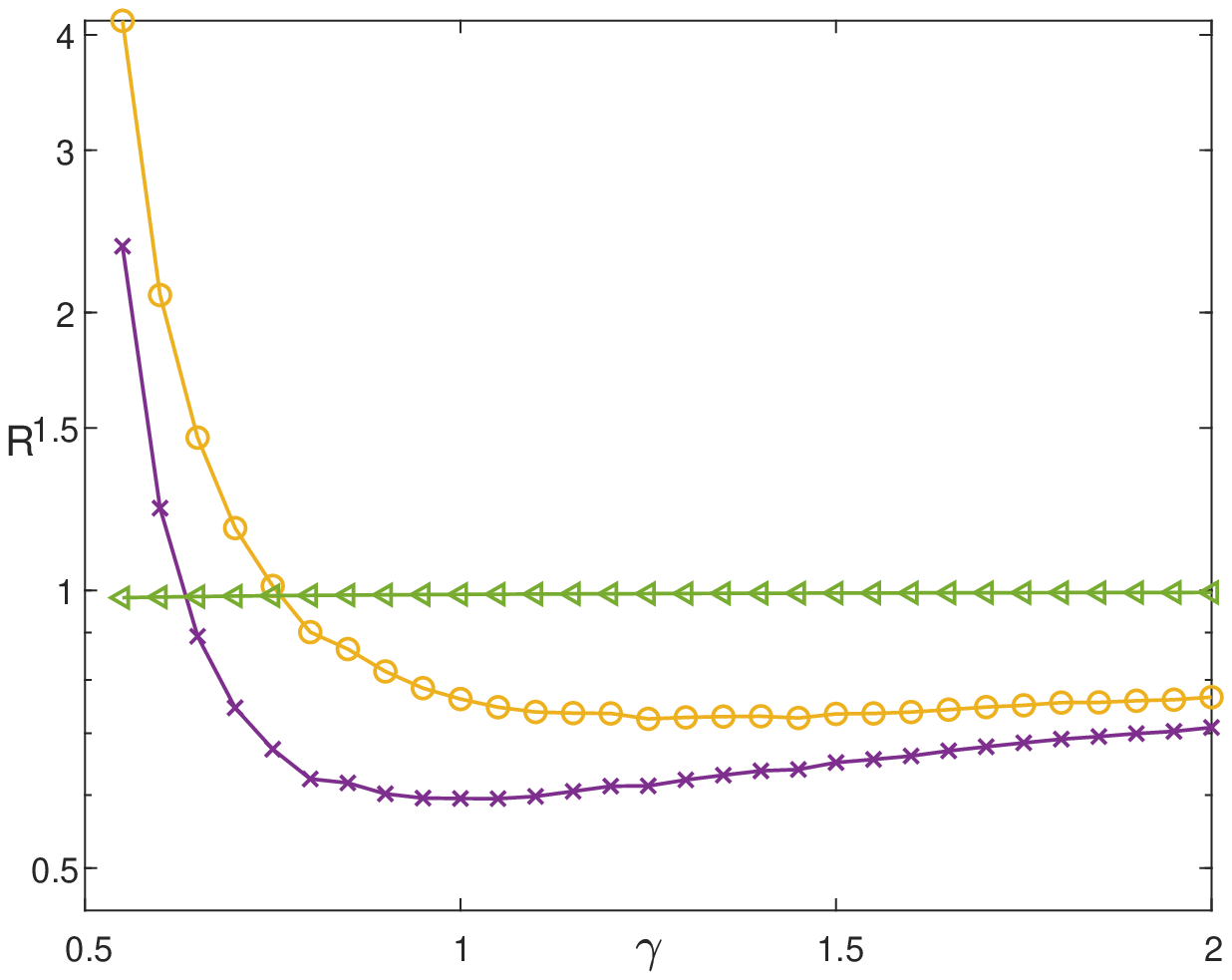}
  \caption*{(a)}
  \includegraphics[width=1\linewidth]{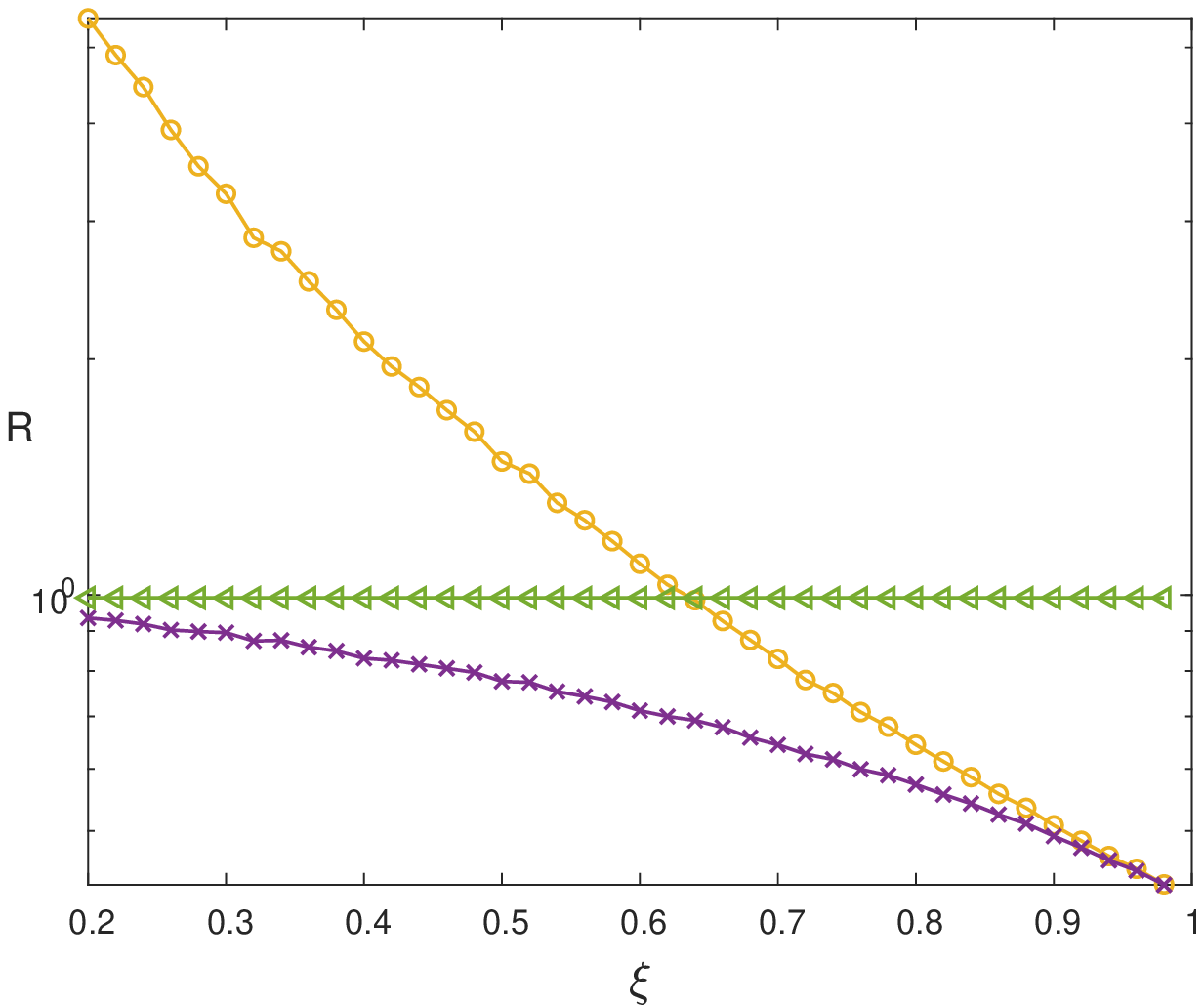}
  \caption*{(c)}
\end{minipage}%
\begin{minipage}{.5\textwidth}
  \centering
  \includegraphics[width=1\linewidth]{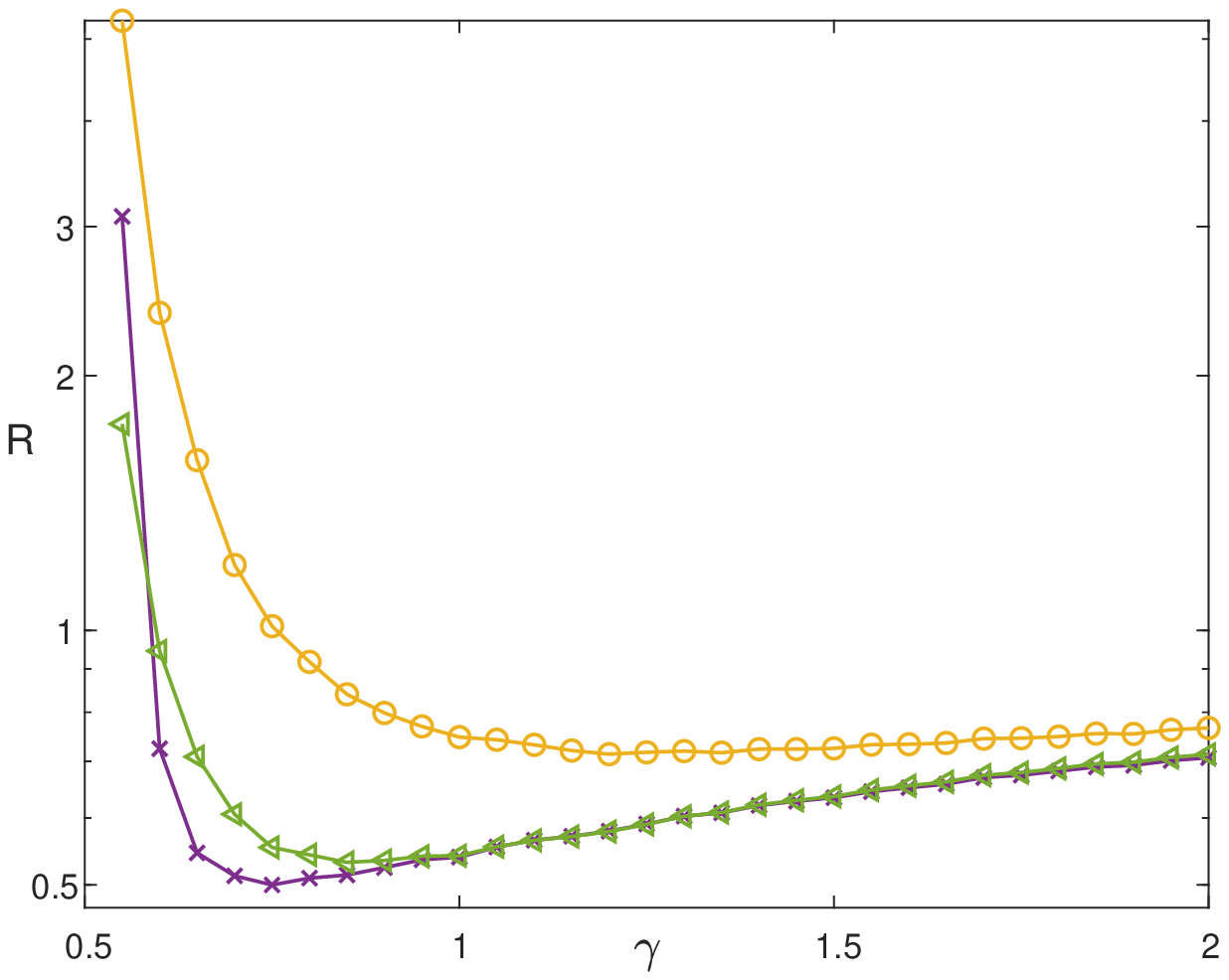}
  \caption*{(b)}
  \includegraphics[width=1\linewidth]{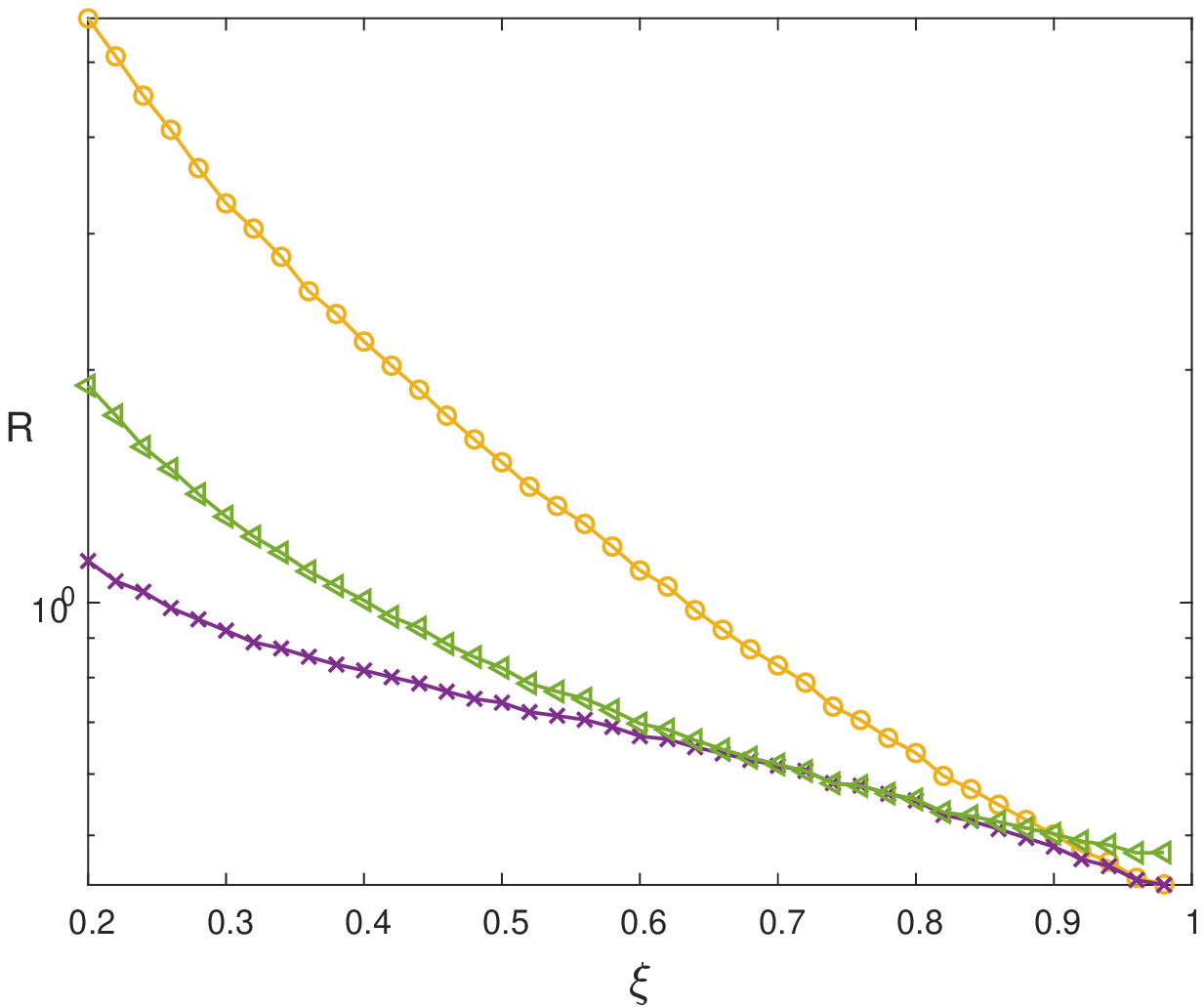}
  \caption*{(d)}
\end{minipage}
 \caption{\textit{ Risk of \ac{batch-min-norm} and server-averaging for (a) $b=2$ and $\xi=0.7$, (b) $b=200$ and $\xi=0.75$, (c) $b=2$ and $\gamma=1.2$ and (d) $b=200$ and $\gamma=1.2$. For the small batch size $b=2$ the risk of server-averaging is close to the null risk. The risk of  \ac{batch-min-norm} is below the null risk for $\xi\geq 0.6478$. When the batch size is linear in the number of samples server-averaging outperforms \ac{batch-min-norm} for low to moderate $\snr$ levels. \ac{s-batch-min-norm} has lower risk than both the other algorithms for all settings except near the (effective) interpolation limit $\gamma=1/b$.
}}\label{fig:bmn_vs_avg_gamma}
\end{figure*}

\begin{figure*}[ht]
\centering
\begin{minipage}{1\textwidth}
   \centering
  \begin{minipage}{0.5\textwidth}
  \centering
  \includegraphics[width=1.05\textwidth]{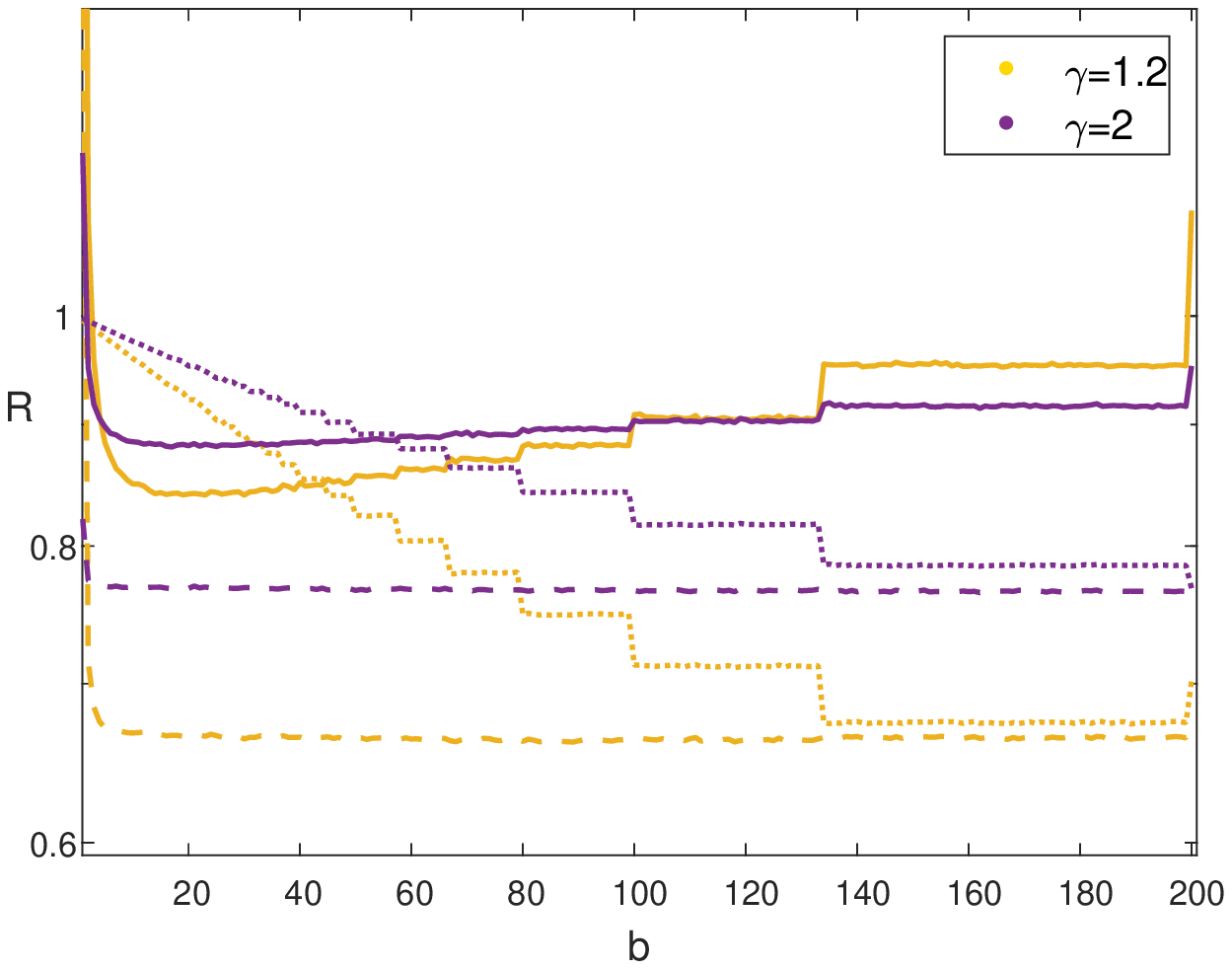}
    \caption*{(a)}
  \end{minipage}
\begin{minipage}{0.5\textwidth}
  \centering
     \includegraphics[width=1.05\textwidth]{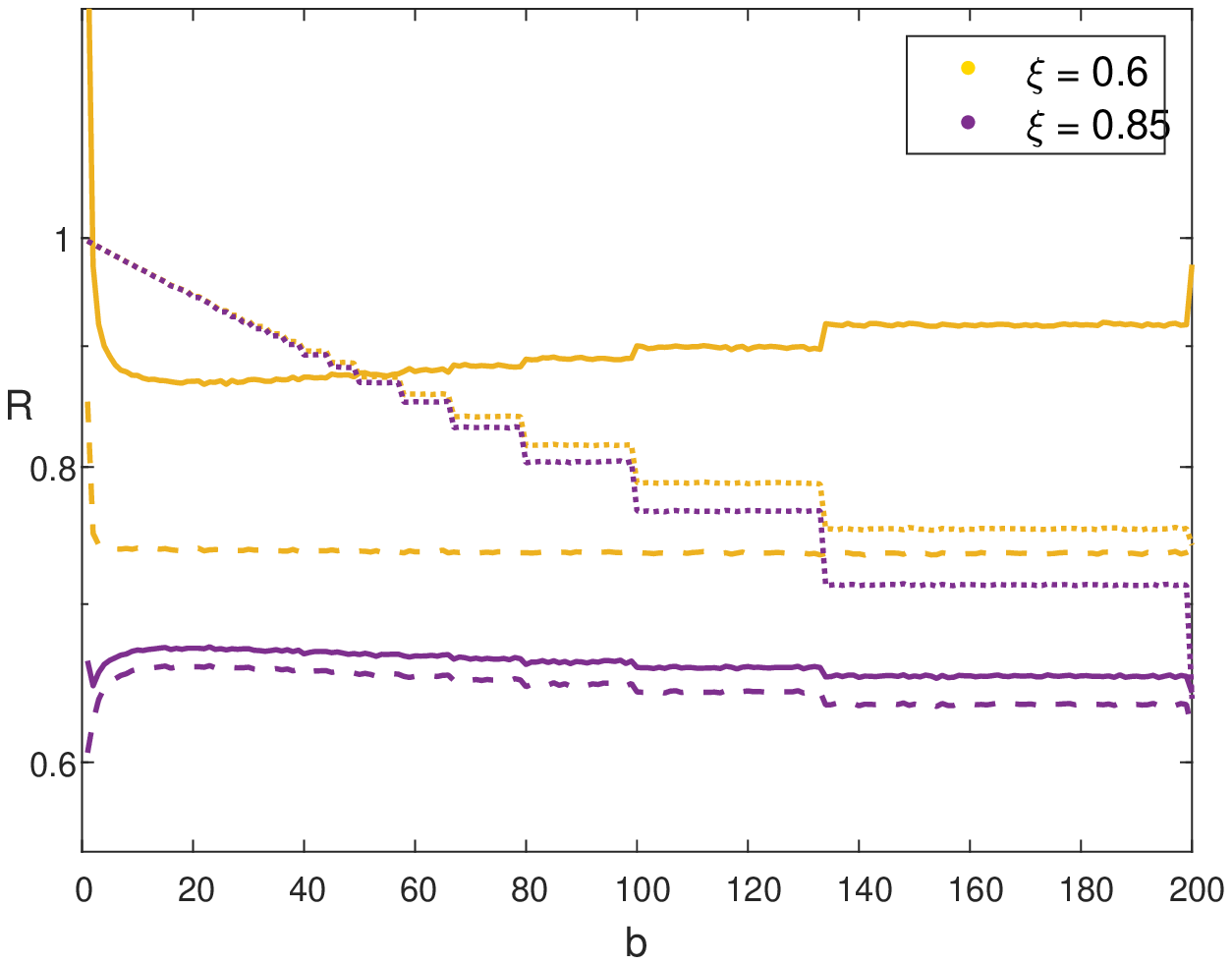}
  \caption*{(b)}
  \end{minipage}
\end{minipage}
\caption{\textit{Risk  \ac{batch-min-norm} (continuous line), \ac{s-batch-min-norm} (broken line) and server-averaging (dotted line), vs. the batch size $b$, for (a) overparametrization ratios $\gamma =1.2$ and $\gamma =2$ with $\xi=0.6$, and (b) $\snr$ levels $\xi=0.6$ and $\xi =0.85$ with $\gamma=1.7$. For small batch sizes, \ac{batch-min-norm} outperforms server-averaging. For linear batch size, server-averaging exhibits a lower risk than \ac{batch-min-norm} in low- to moderate-$\snr$. \ac{s-batch-min-norm} outperforms both other algorithms in all settings.  }}\label{fig:bmn_vs_avg_b_1} 
\nonumber
\end{figure*}

Next, we study the performance of \ac{batch-min-norm}, \ac{s-batch-min-norm}, and server-averaging as a function of the batch size, as presented in Figure~\ref{fig:bmn_vs_avg_b_1}. The number of samples for each $b$ value is set to be $n_b=\lceil 400/b\rceil \cdot b$ (the smallest integer divisible by $b$ that is greater or equal to $400$).
Subfigure~\ref{fig:bmn_vs_avg_b_1}-(a) presents the performance of the estimators for two overparametrization ratios: $\gamma=1.2$ and $\gamma=2$, with $\xi=0.6$. \ac{batch-min-norm} outperforms server-averaging for $b<45$ when $\gamma=1.2$ and for $b<55$ when $\gamma=2$. As the batch size decreases, the risk of server-averaging approaches the null risk. For $\gamma=1.2$ the minimal averaging risk is obtained at $b=133 \approx\frac{n}{3}$, while for $\gamma=2$ the optimal averaging batch size is $b= 200=\frac{n}{2}$.  \ac{s-batch-min-norm} outperforms the other estimators for all batch sizes. It can be observed that the risk of the algorithms exhibits steps in the points $b=n/s$, for $s=1,2,3,4$. This artifact is caused by our method for setting the number of samples $n$. Effectively, at each such $s$ step, the number of modified samples in the second \ac{min-norm} step of the algorithm increases by one (and is, in fact, equal to $s$ over the entire interval $b\in(\frac{n}{s}, \frac{n}{s+1})$).  
Subfigure~\ref{fig:bmn_vs_avg_b_1}-(b) presents the performance of the estimators for $\xi=0.6$, and $\xi=0.85$, with $\gamma=1.7$. 
 In high $\snr$, $\xi=0.85$, \ac{batch-min-norm} outperforms server-averaging for all batch sizes. For moderate $\snr$, $\xi=0.6$, server-averaging is batter than \ac{batch-min-norm} for $b>50$. Again, \ac{s-batch-min-norm} is uniformly better than both others.

\begin{figure*}[ht]
\centering
\begin{minipage}{1\textwidth}
   \centering
  \begin{minipage}{0.5\textwidth}
  \centering
  \includegraphics[width=1\textwidth]{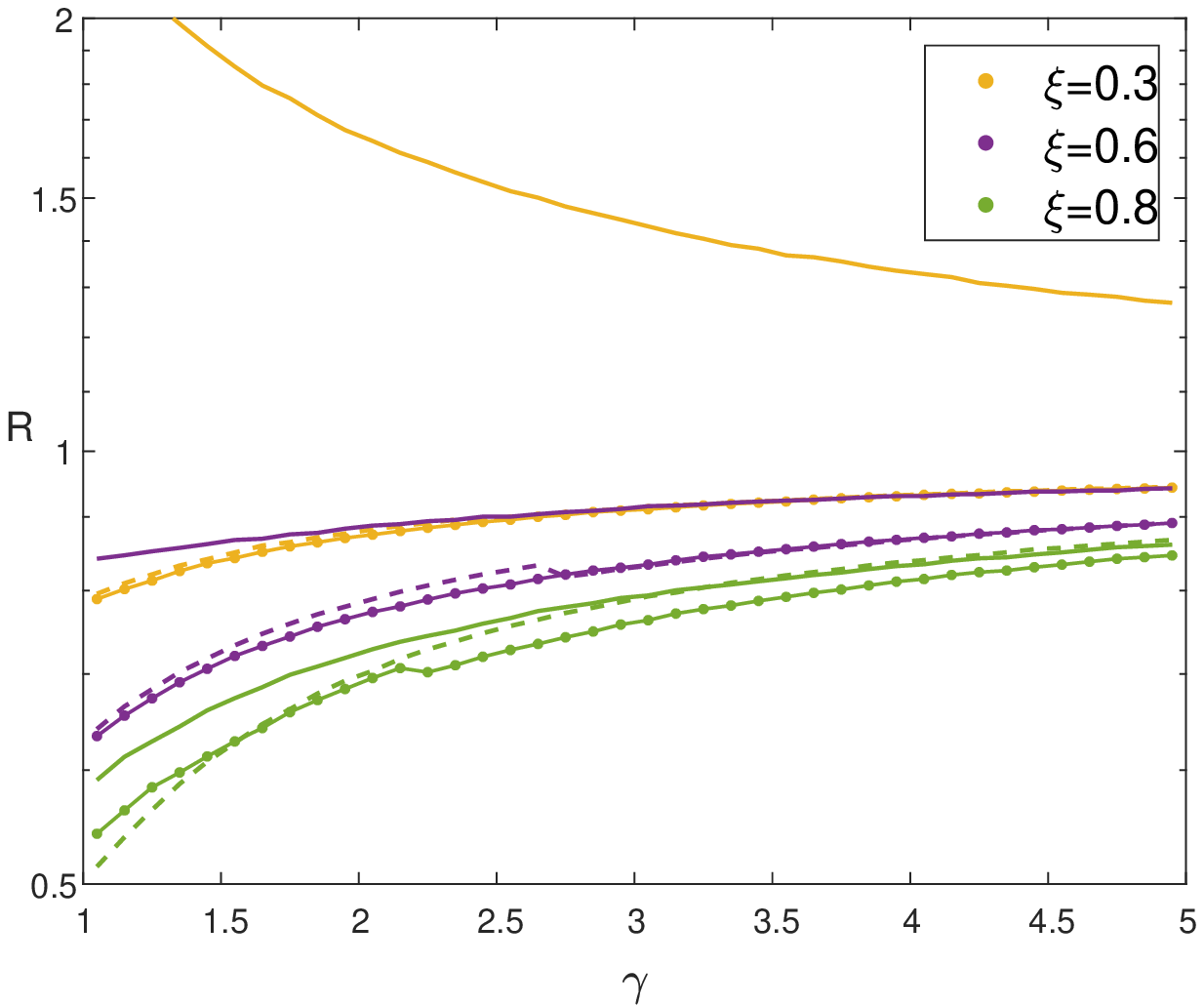}
    \caption*{(a)}
  \end{minipage}
\begin{minipage}{0.5\textwidth}
  \centering
      \includegraphics[width=1\textwidth]{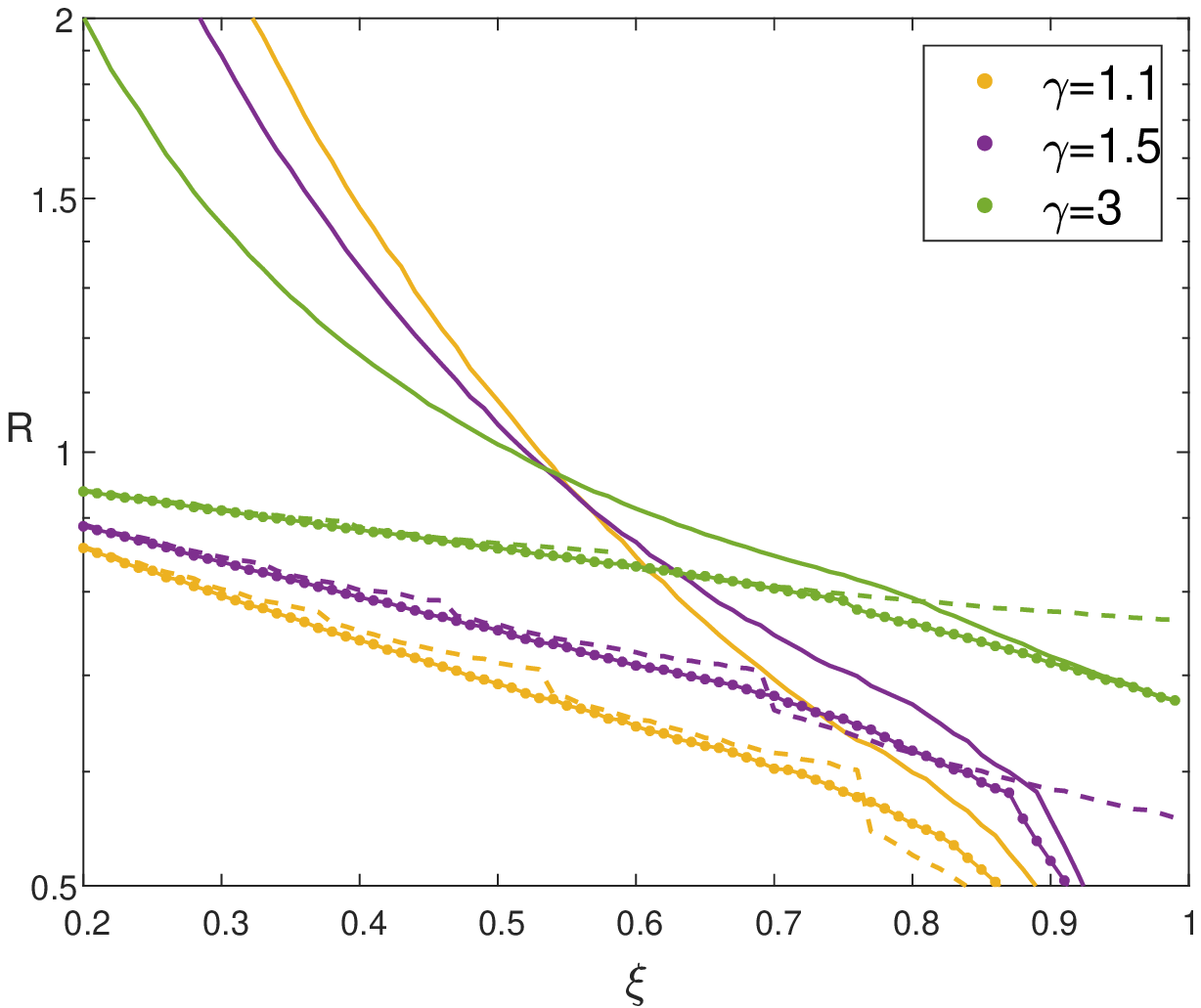}
  \caption*{(b)}
  \end{minipage}
\end{minipage}
\caption{\textit{Risk of optimized \ac{batch-min-norm} (continuous lines), optimized \ac{s-batch-min-norm} (dotted lines) and optimized server-averaging (broken lines). The optimal batch size of server-averaging is chosen as the integer minimizer of \eqref{eq:avg_risk}, and always lies in the linear batch size regime. In high $\snr$, \ac{batch-min-norm} outperforms server-averaging. \ac{s-batch-min-norm} has lower or equal risk compared to both algorithms} in all settings except for a small range around $(\gamma,\xi)=(1,0.85)$ and $(\gamma,\xi)=(1.5,0.75)$}. \label{fig:opt_bmn_vs_opt_avg_} 
\nonumber
\end{figure*}

Finally, we compare the performance of \ac{batch-min-norm}, \ac{s-batch-min-norm}, and server-averaging with optimized batch sizes. The optimal batch size of \ac{batch-min-norm} and \ac{s-batch-min-norm} are derived by minimizing the upper bounds \eqref{eq:UB} and \eqref{eq:shrunk_UB}, respectively, over all integer values. The optimal batch size of server-averaging is chosen as the integer minimizer of \eqref{eq:avg_risk} in the range $[1,\frac{n}{2}]$. Figure~\ref{fig:opt_bmn_vs_opt_avg_} shows the performance of the optimized algorithms vs. (a) the overparametrization ratio $\gamma$, for $\xi=0.3,0.6$, and $\xi=0.8$, and (b) the $\snr$ level $\xi$, for $\gamma=1.1,1.5$ and $\gamma=3$. Optimized \ac{batch-min-norm} outperforms optimized server-averaging only above some high-$\snr$ threshold (that depends on $\gamma$). The optimal batch size of server-averaging is always in the linear batch size regime, while the optimal \ac{batch-min-norm} size is often very small. The optimized \ac{s-batch-min-norm} outperforms \ac{batch-min-norm} in all settings, and except for a small range around $(\gamma,\xi)=(1,0.85)$ and $(\gamma,\xi)=(1.1,0.75)$, it also has a lower risk than server-averaging.

\section{Discussion} \label{sec:discussion}

In this paper, we proposed a batch version of the \ac{min-norm} algorithm, based on first linearly combining the elements inside each batch to obtain a new linear-like model, which is more overparametrized yet better aligned with the parameters, and then performing \ac{min-norm} on it. The coefficients of the per batch linear combination are chosen as the \ac{min-norm} weights. Therefore, our \ac{batch-min-norm} estimator has low parameter bias while still enjoying a high noise-averaging effect. We derived bounds on the limiting risk of \ac{batch-min-norm}, and then optimized the batch size to minimize the upper bound. The resulting optimized \ac{batch-min-norm} risk is stable for all overparametrization ratios (for sufficiently high $\snr$). We further suggested a shrinkage variation of \ac{batch-min-norm}. The risk of \ac{s-batch-min-norm} (with optimal batch size) is stable for all overparametrization ratios and $\snr$ levels. Via numerical experiments, we showed that \ac{s-batch-min-norm} is superior to server-averaging in the distributed linear regression problem, in almost all settings. Furthermore, as the  $\snr$ decreases or as the overparametrization ratio increases, \ac{s-batch-min-norm} approaches the performance of optimized ridge regression, which is optimal in the Gaussian isotropic model. 

Here, we defined the risk as the mean squared error of the parameter vector, where the expectation is taken over the features $\H$ and the noise $\W$. We note that taking the expectation over the feature matrix results in weaker results compared to what is known for standard min-norm and ridge regression, where the conditional risk given the features, $R_{\H}(\hbt)\triangleq\E[\|\hbt -\bt\| \big|\H]$, is considered, and the limiting results hold in the almost sure sense over the features~\cite{hastie2022surprises}. The reason we were only able to obtain expected bounds rather than almost sure ones is rooted in the parameter bias part. While our excess noise bounds hold in the $L^1$ sense, our bias bounds are limited to bounding the expected bias. This is inherent to our Gaussian approximation technique in Subsection~\ref{sec:semi_noisy}. To calculate the per-batch contribution to the overall bias, we approximated the distribution of the underlying inner products that define it via clean Gaussian statistics, in the $1$-Wasserstein distance.  This translates to convergence in distribution plus convergence in the first moment (as $n,p\to\infty$). We then used the dual form of the  $1$-Wasserstein distance to obtain a bound on the expected per-batch bias. This technique is therefore not suitable for obtaining a stronger result such as $L^1$ or almost sure convergence. 

Another limitation of our result is that it is restricted to the isotropic Gaussian setting, again in contrast to \cite{hastie2022surprises} where more general settings like latent space features and nonlinear model were considered. Due to the batch partition and the nonlinear dependence of the modified model on the parameter and features, the seemingly simple isotropic Gaussian case already posed significant new technical challenges when analyzing our estimators. Nevertheless, our results are extendable to the misspecified isotropic case of \cite{hastie2022surprises} by treating misspecified energy as noise, resulting in a risk that can have a minimum point for some $\gamma >1$ like \ac{min-norm}, but does not explode at the interpolation limit (due to the shift caused by the batch partition). The latent space case and nonlinear model are more challenging. Simulations indicate that in the latent case, our algorithm behaves similarly to \ac{min-norm} (i.e., attains a global minimum at $\gamma\to\infty$, see~\cite{mythesis}), while again avoiding explosion at $\gamma=1$. It thus appears that perhaps not much new insight into the batch problem can be gleaned by considering this more general setting.

A natural question that follows from our work is whether the performance of the batch algorithm can be further improved by optimizing the coefficients \eqref{eq:A_j} of the per-batch linear combination, replacing the per-batch \ac{min-norm} coefficients $\X_j^T(\H_j\H_j^T)^{-1}$ we used here. While this can be beneficial, the upside of our current choice stems from the projection property of \ac{min-norm}, and deriving upper bounds on the resulting risk for general coefficients (that depend on the measurements and features) can prove difficult. 

Another interesting question that arises is whether our algorithm can benefit from iterations. That is, after constructing the new linear-like model \eqref{eq:modified_linear_model}, partitioning it again into batches and performing another step of per-batch \ac{min-norm} estimation. This can be done repeatedly for the desired number of iterations before the centralized \ac{min-norm} step \eqref{eq:batch_estimator} is performed.  While this variation is again difficult to analyze (for example, since starting from the second iteration, the weights of the per-batch linear combination do not necessarily converge to the batch observations), it is intuitive to expect that the performance of iterative-\ac{batch-min-norm} (resp. iterative-\ac{s-batch-min-norm}) algorithm will be similar to the regular \ac{batch-min-norm} (resp. \ac{s-batch-min-norm}) with batch size that is equal to the product of the batch sizes at each of the iterations. For example, if we apply two iterations of \ac{batch-min-norm} with $b_1=4$ and $b_2=3$ the performance should be similar to a single iteration of \ac{batch-min-norm} with $b=12$. This is because the final step in both algorithms has the same overparametrization ratio and hence the same interpolation limit. Also, the modified features at the final step span approximately the same subspace (each modified vector is a linear combination of the same features), and the noise-averaging effect should also be similar due to the same reasons. This conjecture is supported by numerical results, see Figure~\ref{fig:iterativeBMN} below.  

\begin{figure*}
\centering
\begin{minipage}{.5\textwidth}
  \centering
  \includegraphics[width=1\linewidth]{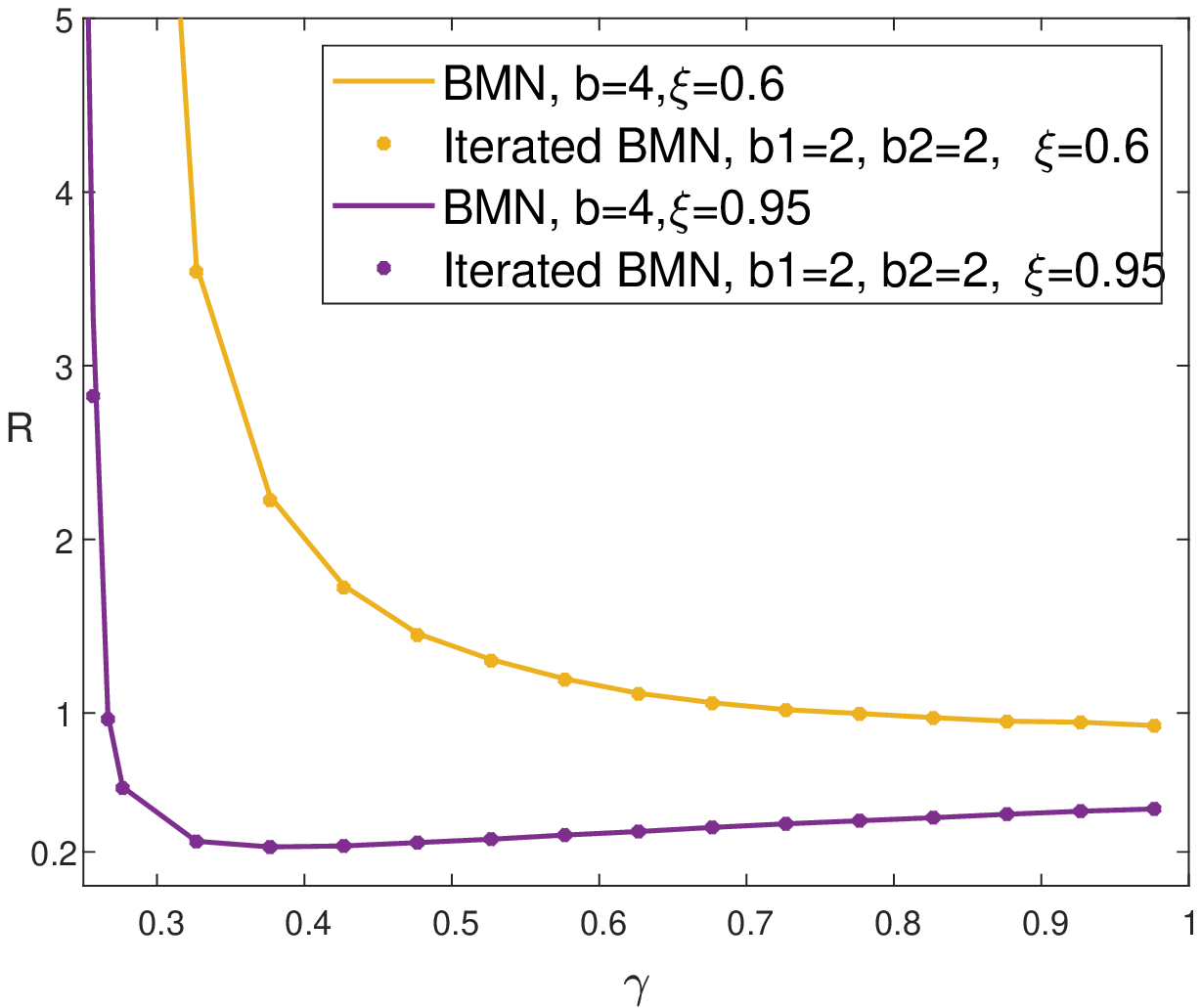}
    \caption*{(a)}
  \end{minipage}%
\begin{minipage}{.5\textwidth}
  \centering
  \includegraphics[width=1\linewidth]{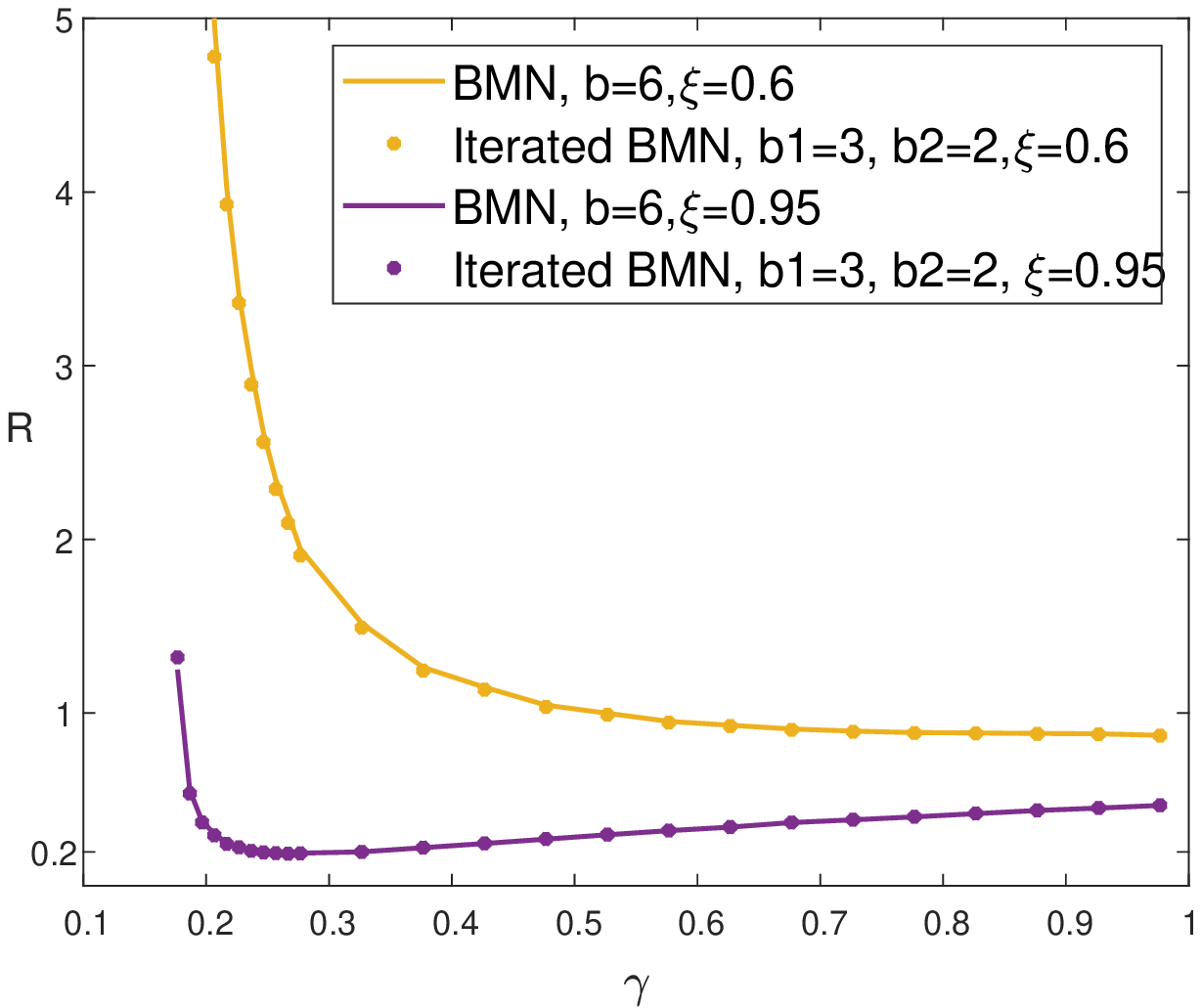}
  \caption*{(b)}
\end{minipage}
 \caption{Comparison between \ac{batch-min-norm} with batch size $b$ and Iterative-\ac{batch-min-norm} with two iterations and batch sizes $b_1$ and $b_2$, in each of the iterations, respectively. The batch sizes are set as (a) $b=4$ and $b_1=b_2= 2$, and (b) $b=6$, $b_1=3$, and $b_2=2$. It can be seen that in both cases the risk of the two estimators coincides. 
}\label{fig:iterativeBMN}
\end{figure*}

\appendix 
\section{Additional Proofs and Results}

\subsection{The Wasserstein Distance}\label{appen:wassertein}

We state and prove some properties of the Wasserstein distance used in Section~\ref{sec:semi_noisy}.

\begin{prop}[basic properties of Wasserstein distance]\label{prop:Wasserstein_properties}
    The following properties hold for the $1$-Wasserstein distance
    \begin{enumerate}
        \item (triangle inequality) $\dW_1(\mu,\nu)\leq \dW_1(\mu,\gamma)+\dW_1(\gamma,\nu)$.
        \item (stability under linear transformations) Let $A$ be a deterministic matrix and $\bX,\bY$ two random vectors. Then  
        \begin{align}
        \dW_1(A\bX,A\bY)\leq \|A\|_F\dW_1(\bX,\bY).
        \end{align}
        \item If $\bX$, $\bY\in\mathbb{R}^n$ and $\bs{Z}\in\mathbb{R}^m$ be mutually independent random vectors, then
                \begin{align}
                     \dW_1\left([\bX,\Z],[\bY,\Z]\right)= \dW_1(\bX,\bY).
                \end{align}       
    \end{enumerate}
\end{prop}

\begin{proof}
The proof of the first property is trivial and therefore omitted.
For the second property, let $\gamma^*$ be the coupling that achieves the $1$-Wasserstein distance between the marginals of $\bX$ and $\bY$, then
\begin{align}
    \|A\|_F\dW_1(\bX,\bY) &=\|A\|_F\E_{\gamma^*}\|\bX-\bY\|^p
    \geq \E_{\gamma^*}\|A\bX-A\bY\|_2 
    \geq\dW_1(A\bX,A\bY).
\end{align}
While proven here for $p=1$, the above property holds for any $p$ with the Euclidean metric. 
Next, we prove the third property.
Denote $\bX\sim P_{\bX}$, $\bY\sim P_{\bY}$ and $\Z\sim P_{\Z}$.
Let $P_{\bX,\bY}$ be the coupling achieving the Wasserstein distance $\dW_1(\bX,\bY)$ between $\bX$ and $\bY$,  let $P_{\Z,\Z'}$ be a coupling of $\Z$ and $\Z'$ that set $\Pr(\Z=\Z')=1$, and let $\Gamma$ be the set of all couplings of $[\bX,\Z]$ and $[\bY,\Z]$. 
Then, $P_{\bX,\bY}\otimes P_{\Z,\Z'}$ is a valid coupling for the marginals $P_{\bX}\otimes P_{\Z}$ and $P_{\bY}\otimes P_{\Z}$ and therefore in $\Gamma$. 
Hence
\begin{align}
    \dW_1([\bX,\Z]&,[\bY,\Z]) \triangleq \inf_{\zeta\in \Gamma}\E_\zeta \sqrt{\|\bX-\bY\|^2_2 +\|\Z-\Z'\|_2^2} \\
   &\leq \E_{P_{\bX,\bY}\otimes P_{\Z,\Z'}} \sqrt{\|\bX-\bY\|^2_2 +\|\Z-\Z'\|_2^2} \\
   &= \E_{P_{\bX,\bY}} \sqrt{\|\bX-\bY\|^2_2 }\\
   &= \dW_1(\bX,\bY).\label{eq:same_marginals}
\end{align}
In the above, we assumed that there exists a coupling $P_{\bX,\bY}$ that achieves $\dW_1(\bX,\bY)$. In the case that the infimum is obtained only in the limit, one can adjust the proof by taking the sequence of couplings that converge to the infimum.
Since $  \dW_1\left([\bX,\Z],[\bY,\Z]\right)$ is also trivially lower bounded by $  \dW_1\left(\bX,\bY\right)$, we get equality.
\end{proof}

The next lemma will demonstrate that when two measures are heavily concentrated on a bounded region, then conditioning on being in this region does not change the Wasserstein distance much.  
\begin{lemma}[stability]\label{lem:Wasserstein_conditioned}
Let $\bX$ and $\bY$ be random vectors taking values in $\mathbb{R}^n$,  
and $\mathcal{D}\subset\mathbb{R}^n$ be a compact set such that $\Pr(\bX\notin\mathcal{D})\leq \epsilon$ and $\Pr(\bY\notin\mathcal{D})\leq \delta$. 
Let $\bX_\mathcal{D}$ and $\bY_{\mathcal{D}}$ denote the conditional random variables obtained from $\bX$ and $\bY$ by conditioning on the events $\bX\in\mathcal{D}$ and $\bY\in\mathcal{D}$, respectively.  Then 
\begin{align}
    \dW_1(\bX_\mathcal{D},\bY_\mathcal{D}) \leq \dW_1(\bX,\bY)+(\eps+\delta)\cdot \max_{\bx,\by\in\mathcal{D}}\|\bx-\by\|_2.
\end{align}
\end{lemma}
\begin{proof}
Denote by $\mu$, $\nu$, $\tilde{\mu}$ and $\tilde{\nu}$ the pdf of $\bX$, $\bY$, $\bX_\mathcal{D}$ and $\bY_\mathcal{D}$, respectively.
Let $(\bX^*,\bY^*)\sim \zeta^*(\bx,\by)$ be a coupling that attains the infimum in \eqref{eq:Wp_distance_def} (if there is no minimizer one can adjust the proof by taking a sequence of couplings that converges to the infimum). Define a new pair of random vectors $(\bar{\bX},\bar{\bY})$ such that 
\begin{align}
    (\bar{\bX},\bar{\bY})=\begin{cases}(\bX^*,\bY^*),& (\bX^*,\bY^*)\in \mathcal{D}^2\\
                                        (\A,\bY^*), & \bX^*\notin \mathcal{D},\bY^*\in\mathcal{D},\\
                                        (\bX^*,\B), & \bX^*\in \mathcal{D},\bY^*\notin\mathcal{D},\\
                                        (\A,\B), & \bX^*\notin \mathcal{D},\bY^*\notin\mathcal{D},
    \end{cases}
\end{align}                                                 
with $\A\sim\tilde{\mu}$, $\B\sim\tilde{\nu}$ being two random vectors  such that $(\bX^*,\bY^*), \A, \B$ are mutually independent.
Denote by  $a$ the probability that $(\bX^*,\bY^*)\in \mathcal{D}^2$. Then on the one hand 
\begin{align}
    \dW_1(\bX,\bY)&= \E\left[\|\bX^*-\bY^*\|\right]\\
    &=a\cdot \E\left[\|\bX^*-\bY^*\|\big|(\bX^*,\bY^*)\in\mathcal{D}^2\right]+\\
    &\;\;\;(1-a)\cdot \E\left[\|\bX^*-\bY^*\|\big|(\bX^*,\bY^*)\notin\mathcal{D}^2\right]\\
    &\geq a\cdot \E\left[\|\bX^*-\bY^*\|\big|(\bX^*,\bY^*)\in\mathcal{D}^2\right],
\end{align}
and on the other hand 
\begin{align}
    \dW_1(&\bX_\mathcal{D},\bY_\mathcal{D})= \inf_{\zeta\in\Gamma[\tilde{\mu},\tilde{\nu}]} \E_{(\bX',\bY')\sim \zeta} \|\bX'-\bY'\| \\
    &\leq \E \| \bar{\bX}-\bar{\bY}\|\label{line:correct_marginals}\\
    & =a\cdot \E\left[\|\bar{\bX}-\bar{\bY}\|\big|(\bX^*,\bY^*)\in\mathcal{D}^2\right]\\
    &\;\;\;+(1-a)\cdot \E\left[\|\bar{\bX}-\bar{\bY}\|\big|(\bX^*,\bY^*)\notin\mathcal{D}^2\right]\label{line:total_law_of_E}\\
    &\leq a\cdot \E\left[\|\bar{\bX}-\bar{\bY}\|\big|(\bX^*,\bY^*)\in\mathcal{D}^2\right]+(1-a)\cdot R\label{line:maximal_radius}\\
    &= a\cdot \E\left[\|\bX^*-\bY^*\|\big|(\bX^*,\bY^*)\in\mathcal{D}^2\right]+(1-a)\cdot R.\label{line:XY_equalttildeXY}\\ 
&\leq \dW_1(\bX,\bY) + (1-a)\cdot R,
\end{align}
where in \eqref{line:correct_marginals} we used the fact that the coupling $(\bar{\bX},\bar{\bY})$ has marginals $\tilde{\mu}$ and $\tilde{\nu}$, in \eqref{line:total_law_of_E} we used the total law of expectation, in \eqref{line:maximal_radius} we used the fact that for any $\bx,\by\in\mathcal{D}$ we have $\|\bx-\by\|_2\leq R$, and in \eqref{line:XY_equalttildeXY} we used the fact that when $(\bX^*,\bY^*)\in\mathcal{D}^2$ then $(\bar{\bX},\bar{\bY})=(\bX^*,\bY^*)$.
Let us now bound $a$ from below
\begin{align}
    1-\delta&\leq \Pr(\bY\in D)=\Pr(\bY^*\in D)
    = a + \Pr(\bY^*\in D, \bX^*\not\in D) \leq a + \eps, 
\end{align}
hence $a\geq 1-\eps-\delta$ and we get
\begin{align}
    \dW_1(\bX_\mathcal{D},\bY_\mathcal{D})&\leq \dW_1(\bX,\bY) + (\eps+\delta)\cdot R.
\end{align}
\end{proof}

\subsection{Proof of the Lemmas in Subsection~\ref{sec:semi_noisy}}\label{appen:semi_noisy}

To prove Lemma~\ref{lem:proj_clos_to_fx} we will use the following proposition. 
\begin{prop}\label{prop:noisy_D}
Let $\D_j$ be the projection onto the row space of $\H_j$, and let $\Z_j\triangleq \H_j^T(\H_j\H_j^T)^{-1}\W_j$, then 
$\D_j(\bt+\Z_j)$ can be written as a noisy projection of the form \eqref{eq:noisy_proj}.
\end{prop}
\begin{proof}
Let $\U_1,\cdots,\U_b$ be the orthonormal basis vectors for the subspace spans by the rows $\H_j$, then
\begin{align}
    \D_j(\bt+\Z_j)&=\D_j\bt +\D_j  \H_j^T(\H_j\H_j^T)^{-1}\W_j\\
    &= \D_j\bt +\sum_{i=1}^{b} \U_i \lambda_i W_{ij},
\end{align}
with  $\lambda_i \triangleq\lambda_i((\H_j\H_j)^{-1/2})$ and $W_{ij}$ the $i$th element of $\W_j$. Since $\H_j$ is i.i.d.~Gaussian, it is well-known that the eigenvalues $\{\lambda_i\}$ are mutually independent of the eigenvectors $\{\U_i\}$. 
Define $T_i \triangleq p\lambda_i^2$ and $Z_i \triangleq\lambda_i W_{ij}$, then $Z_i\sim \mathcal{N}(0,T_i\cdot \frac{\sigma^2}{p})$. 
As required in the conditions of Lemma~\ref{lem:semi_noisy_projection}, $T_i$, $i=1,\cdots,b$, are highly concentrated around $1$, and specifically we have (see for example Corollary 5.35 in~\cite{vershynin2010introduction})
\begin{align}
    \Pr\left(\sqrt{\frac{T_i}{p}}\in \left[\frac{1}{\sqrt{p}}-\frac{1}{p^{-\frac{7}{8}}},\frac{1}{\sqrt{p}}+\frac{1}{p^{-\frac{7}{8}}}\right]\right)\geq 1-2e^{-\frac{\sqrt[4]{p}}{2}}.
\end{align}    
\end{proof}

\begin{proof}[Proof of Lemma~\ref{lem:proj_clos_to_fx}.]
The proof for this lemma is based on the fact that for a large $p$, the length $\|\P \U_i\|^2\approx\delta$ and $\langle \P \U_i, P\U_j\rangle\approx 0$. 
If this was true with equality, then by 
noticing that 
\begin{align}\label{eq:proj_nomerator}
    \langle \P\tilde{\D}\bt, \bt\rangle &= \langle \sum_i \P \U_i\left(\U_i^T\bt + Z_i\right) , \bt\rangle
    \\&= \sum_i \left(\langle\U_i,\bt\rangle+Z_i \right)\cdot  \langle \P \U_i, \bt\rangle
    \\&= \sum_{i=1}^b \SS_i\cdot (\SS_i+\SS_{b+i}+Z_i)
\end{align}
and 
\begin{align}\label{eq:proj:denom}
    \|\P&\tilde{\D}\bt\|^2 = \langle \P\tilde{\D}\bt, \P\tilde{\D}\bt \rangle 
    \\&= \langle \sum_i \P \U_i\left(\U_i^T\bt + Z_i\right), \sum_i \P \U_i\left(\U_i^T\bt + Z_i\right)\rangle
    \\&= \langle \sum_i (\langle \U_i, \bt\rangle +Z_i)\cdot \P \U_i \,,  \sum_i (\langle \U_i, \bt\rangle +Z_i) \cdot \P \U_i\rangle
    \\ & = \sum_i (\SS_i+\SS_{b+i}+Z_i)^2 \cdot \|\P \U_i\|^2\\
    &\quad + \sum_{i\neq \ell} (\langle \U_i, \bt\rangle +Z_i)\cdot (\langle \U_\ell, \bt\rangle +Z_\ell)\cdot \langle \P \U_i, P\U_\ell\rangle,
\end{align}
we would have got $\frac{1}{\delta}f(\S,\Z)=\frac{ \langle \P\tilde{\D}\bt, \bt\rangle^2}{\|\P\tilde{D}\bt\|^2}$.

To show that this is asymptotically true let us define  $E_i\triangleq \|\P \U_i\|^2-\delta $, $i=1,\cdots,b$,  and $ C_{i\ell}\triangleq \langle \P \U_i, P\U_\ell\rangle $, $i\neq j\in\{1,\cdots,b\}$.
Then, define the events
\begin{align}
    \mathcal{E}&= \left\{ \{|\E_i|\leq \eps,i=1\cdots,b\}\;\cap \;\{|C_{i\ell}|\leq\ \eps, \;\forall i\neq \ell\}\right\},\\
    \mathcal{Z}&= \left\{|Z_i|<a, i=1,\cdots,b\right\}.
\end{align}
Denote by $\bar{\mathcal{E}}$ and $\bar{\mathcal{Z}}$ the complement events of $\mathcal{E}$ and $\mathcal{Z}$, respectively,
and note that $\mathcal{E}$ and $\mathcal{Z}$ are independent.
From the triangle inequality, we get
\begin{align}\label{eq:proj_f_triangle_ineq}
     \Bigg|  \mathbb{E}\Big[&\frac{1}{\delta}f(\S,\Z)\Big]-\mathbb{E}\Big[\frac{ \langle \P\tilde{\D}\bt, \bt\rangle^2}{\|\P\tilde{\D}\bt\|^2}\Big]\Bigg| \\
     &\leq
     \underset{T_1}{\underbrace{\left| \mathbb{E}\left[\frac{ \langle \P \tilde{\D}\bt, \bt\rangle^2}{\|\P \tilde{\D}\bt\|^2}\right]-\mathbb{E}\left[\frac{ \langle \P \tilde{\D}\bt, \bt\rangle^2}{\|\P \tilde{\D}\bt\|^2}\Big| \mathcal{E}\cap \mathcal{Z}\right]\right|}}\\
     &+\underset{T_2}{\underbrace{ \left| \mathbb{E}\left[\frac{ \langle \P \tilde{\D}\bt, \bt\rangle^2}{\|\P \tilde{\D}\bt\|^2}\Big| \mathcal{E}\cap \mathcal{Z}\right]- \mathbb{E}\left[\frac{1}{\delta}f(\S,\Z)\Big| \mathcal{E}\cap \mathcal{Z}\right]\right| }}
     \\&+\underset{T_3}{ \underbrace{\left|\mathbb{E}\left[\frac{1}{\delta}f(\S,\Z)\Big| \mathcal{E}\cap \mathcal{Z}\right]-\mathbb{E}\left[\frac{1}{\delta}f(\S,\Z)\right]\right|}}.
\end{align}
We now estimate the above terms $T_1,T_2,T_3$. 

First, let us look at $T_1$. 
Let $\eps = \frac{\delta}{2b\sqrt[4]{p}}$ , then using the total law of expectation we get
\begin{align}
  \mathbb{E}\left[\frac{ \langle \P \tilde{\D}\bt, \bt\rangle^2}{\|\P \tilde{\D}\bt\|^2}\right] - &\mathbb{E}\left[\frac{ \langle \P \tilde{\D}\bt, \bt\rangle^2}{\|\P \tilde{\D}\bt\|^2}\Big| \mathcal{E}\cap \mathcal{Z}\right]\leq  \Pr(\bar{\mathcal{E}}\cup \bar{\mathcal{Z}}) \times\\
  &\hspace{-2.5cm}\left(\mathbb{E}\left[\frac{ \langle \P \tilde{\D}\bt, \bt\rangle^2}{\|\P \tilde{\D}\bt\|^2}\Big|\bar{\mathcal{E}}\cup \bar{\mathcal{Z}}\right]- \mathbb{E}\left[\frac{ \langle \P \tilde{\D}\bt, \bt\rangle^2}{\|\P \tilde{\D}\bt\|^2}\Big| \mathcal{E}\cap \mathcal{Z}\right]\right).
\end{align}
Using Cauchy-Schwarz and the union bound we get that for any event $\mathcal{A}$
\begin{align}
    \Pr(\bar{\mathcal{E}}&\cup \bar{\mathcal{Z}}) \mathbb{E}\left[\frac{ \langle \P \tilde{\D}\bt, \bt\rangle^2}{\|\P \tilde{\D}\bt\|^2}\Big|\mathcal{A}\right]\\
    &\leq  \left(\Pr(\bar{\mathcal{E}})+\Pr(\bar{\mathcal{Z}})\right) \mathbb{E}\left[\frac{\|\P \tilde{\D}\bt\|^2\|\bt\|^2}{\|\P \tilde{\D}\bt\|^2}\Big| \mathcal{A}\right]
    \\&= \left(\Pr(\bar{\mathcal{E}})+\Pr(\bar{\mathcal{Z}})\right)\cdot r^2.
\end{align}
From the Johnson–Lindenstrauss lemma (see for example Lemma 2.2 in~\cite{dasgupta2003elementary}) we get that
\begin{align}
    \Pr\left(\left|\|\P \U_i\|^2-\delta\right|\geq \eps\right)&\leq e^{-\frac{\eps^2 p}{2\delta}}, \;\;\Pr\left(|\langle \P \U_i, \P\U_j\rangle|\geq \eps\right)\\
    &\leq e^{-\frac{\eps^2 p}{8}}.
\end{align}
Therefore, for our choice of $\eps$ we get
\begin{align}
    \Pr\left(\bar{\mathcal{E}}\right) \leq 2b^2 e^{-\frac{\delta\sqrt{p}}{64b}}
\end{align}
Also from the concentration properties of $T_i$ we have that 
\begin{align}
    \Pr(|Z_i|>a)& = \E\Pr(|Z_i|>a|T_i)\\
  &  \leq 1\cdot \Pr\left(|Z_i|>a\big| T_i = 1+o(1)\right) \\
  &\;\;\;+ o(1/p)\cdot \Pr\left(|Z_i|>a\big| T_i =1+\omega(1)\right)\\
   & \leq e^{-\frac{a^2(p+o(p))}{2\sigma^2}}+o(1/p) = o(1/p),\label{line:bound_concentration_Z}
\end{align}
where in \eqref{line:bound_concentration_Z} we used the bound $\Pr(\mathcal{N}(0,1)\geq x)\leq \exp\{-x^2/2\}$ and the concentration property of $T_i$.
Therefore we get
\begin{align}\label{eq:dist_proj_conditioned_proj}
   \left|\mathbb{E}\left[\frac{ \langle \P \tilde{\D}\bt, \bt\rangle^2}{\|\P \tilde{\D}\bt\|^2}\right] - \mathbb{E}\left[\frac{ \langle \P \tilde{\D}\bt, \bt\rangle^2}{\|\P \tilde{\D}\bt\|^2}\Big| \mathcal{E}\cap \mathcal{Z}\right]\right| = o(1/p).
\end{align}
In a similar way, and using the fact that $0\leq f(\S,\Z)\leq b r^2$, we get for $T_3$ 
\begin{align}\label{eq:dist_fx_conditioned_fx}
 \left|  \mathbb{E}\left[\frac{1}{\delta}f(\S,\Z)\right]-\mathbb{E}\left[\frac{1}{\delta}f(\S,\Z)\Big| \mathcal{E}\cap \mathcal{Z}\right]\right| = o(1/p)
   \end{align}
   
 In order to bound $T_2$ note that from \eqref{eq:proj_nomerator} and \eqref{eq:proj:denom} we get 
\begin{align}
   \frac{ \langle \P \tilde{\D}\bt, \bt\rangle^2}{\|\P \tilde{\D}\bt\|^2}= &\frac{1}{\delta}f(\S,\Z)\times \\
   &\hspace{-2cm}\frac{1}{1  + \frac{\sum_i E_i \left(\langle \U_i, \bt\rangle+Z_i\right)^2}{\delta \sum_i\left( \langle \U_i, \bt\rangle+Z_i\right)^2} + \frac{\sum_{i\neq j} C_{ij}\left(\langle \U_i, \bt\rangle+Z_i\right)\cdot\left( \langle \U_j, \bt\rangle+Z_j\right)}{\delta \sum_i \left(\langle \U_i, \bt\rangle+Z_i\right)^2}}.
\end{align}
Hence, using the Taylor expansion $\frac{1}{1+x}=\sum_{\ell=0}^\infty (-x)^\ell$ we get 
\begin{align}\label{eq:dist_proj_conditioned_fx}
    \Big| \mathbb{E}&\Big[\frac{ \langle \P \tilde{\D}\bt, \bt\rangle^2}{\|\P \tilde{\D}\bt\|^2}\Big| \mathcal{E}\cap \mathcal{Z}\Big]- \mathbb{E}\left[\frac{1}{\delta}f(\S,\Z)\Big| \mathcal{E}\cap \mathcal{Z}\right]\Big|\\
    &\leq \Big|\mathbb{E}\Big[\frac{1}{\delta}f(\S,\Z)\Big(\frac{\sum_i |E_i| (\langle \U_i, \bt\rangle+Z_i)^2}{\delta \sum_i (\langle \U_i, \bt\rangle+Z_i)^2} \\
    &+ \frac{\sum_{i\neq j} |C_{ij}|(\langle \U_i, \bt\rangle+Z_i)\cdot( \langle \U_j, \bt\rangle+Z_j)}{\delta \sum_i (\langle \U_i, \bt\rangle+Z_i)^2}\Big)\Big| \mathcal{E}\cap \mathcal{Z}\Big]\Big|\\
   & \leq \frac{2b\eps}{\delta}\mathbb{E}\left[\frac{1}{\delta}f(\S,\Z)\Big| \mathcal{E}\cap \mathcal{Z}\right]\\
   &\leq  \frac{1}{\sqrt[4]{p}}\left(  \mathbb{E}\left[\frac{1}{\delta}f(\S,\Z)\right]+ o(1/p)\right).
\end{align}

   Combining \eqref{eq:dist_proj_conditioned_proj}, \eqref{eq:dist_proj_conditioned_fx} and \eqref{eq:dist_fx_conditioned_fx} we get
   \begin{align}   \label{eq:dist_proj_fx}
 \biggl|  \mathbb{E}\left[\frac{1}{\delta}f(\S,\Z)\right]&-\mathbb{E}\left[\frac{ \langle \P \tilde{\D}\bt, \bt\rangle^2}{\|\P \tilde{\D}\bt\|^2}\right]\biggr| 
 \leq\\
 &\frac{1}{\sqrt[4]{p}}\cdot \mathbb{E}\left[\frac{1}{\delta}f(\S,\Z)\right]+o(1/p)
   \end{align}
\end{proof}

\begin{proof}[Proof of Corollary~\ref{lem:Wasser_distance_XY}.]
This corollary is  a direct result of Theorem~\ref{thm:Chatterjeee_Wasserstein_thm}, since $\S$ is a low-rank projection of the uniformly random orthonormal matrix $\U\triangleq [\U_1,\cdots,\U_b]$ onto a subspace of dimension $2b$.
Let $\A_i$ be a $p\times p$ matrix with all zeros except for the $i$th row which is given by $\sqrt{\frac{p}{\alpha}}\bt^T \P$ and $\B_i$ be a $p\times p$ matrix with all zeros except for the $i$th row which is given by $\sqrt{\frac{p}{1-\alpha}}\bt^T (\I-\P)$.
Note that $\SS_i = \sqrt{\frac{\alpha}{p}}\trace{(\A_i \U)}$,  and $\SS_{b+i} = \sqrt{\frac{1-\alpha}{p}}\trace{(\B_i \U)}$, $i=1,\cdots,b$, and define the length $2b$ random vector $\tilde{\S}$ such that  $\tilde{\SS}_i \triangleq \trace{(\A_i \U)}= \sqrt{\frac{p}{\alpha}}\cdot \SS_i$, and $\tilde{\SS}_{b+i} \triangleq \trace{(\A_i \U)}= \sqrt{\frac{p}{1-\alpha}}\cdot \SS_{b+i}$, $i=1,\cdots ,b$.
Then, since we have
\begin{align}
    \trace{(\A_i \A_j^T)} &= p \cdot\indfunc{i=j},\\
    \trace{(\B_i \B_j^T)} &= p \cdot \indfunc{i=j},\\
    \trace{(\A_i \B_j^T)} &= 0, \forall i,j.
\end{align}
we get from Theorem~\ref{thm:Chatterjeee_Wasserstein_thm} that 
\begin{align}
    \dW_1(\tilde{\S},\tilde{\G}) \leq \frac{2\sqrt{2}b}{p-1}.
\end{align}
where $\tilde{\G}$ is i.i.d.~Gaussian with zero mean and unit variance. Let $\Lambda$ be a $b\times b$ diagonal matrix with first $b$ diagonal elements equal to $\sqrt{\frac{\alpha}{p}}$ and last $b$ elements equal to $\sqrt{\frac{1-\alpha}{p}}$.
Then, by noting that $\S=\Lambda\tilde{\S}$, $\G=\Lambda \tilde{\G}$  we get that
\begin{align}\label{eq:XY_Wasserstein_distance}
    \dW_1(\S,\G) &\leq \Vert \Lambda\Vert_F \dW_1(\tilde{\S},\tilde{\G})\\
    &=  \sqrt{b\cdot\frac{\alpha}{p}+b\cdot\frac{1-\alpha}{p}} \dW_1(\tilde{\S},\tilde{\G}) \\
    &\leq \sqrt{\frac{b}{p}}\cdot\frac{2\sqrt{2}b}{p-1}.
\end{align}
\end{proof}

The following Lemmas shows that indeed $f$ is Lipschitz over the region where $\S$ and $\G$ are concentrated on. This is needed for the proof of Lemma~\ref{lem:fx_close_to_fy}.

\begin{lemma}\label{lem:bounded_poly_ratio_grad}
Let $\G$ be as in \eqref{eq:Y_i_defenition}, $\S$ be as in \eqref{eq:X_i_defenition} and $\Z$ be as in Lemma~\ref{lem:semi_noisy_projection}, and let $f$ be as in \eqref{eq:f}. 
Then there exists a constant $C$ that does not depend on $p$, and a set $\mathcal{D}$ with diameter no larger than $C$, over which $f$ is $C$-Lipschitz and concentrated: 
\begin{align}
    \Pr\left((\S,\Z)\in \mathcal{D}\right) &= 1-o(1/p),\;\; \Pr\left((\G,\Z)\in \mathcal{D}\right) \\
    &= 1-o(1/p),
\end{align}
\end{lemma}

\begin{proof}
We will show that the gradient of $f$ is bounded with sufficiently high probability under both distributions and that the corresponding high probability set is bounded. 
To that end, let us write $v_i = s_i+s_{b+i}+z_i$. The partial derivatives of $f$ are given by 
\def\v{\bs{v}}
\def\bv{\bs{V}}
\begin{align}
    \frac{\partial f(\s,\z)}{\partial s_\ell}  &= 2(s_\ell +v_\ell)\cdot \frac{\langle \bs{s}, \v\rangle}{\|\v\|^2}- 2v_\ell\cdot \frac{|\langle \bs{s}, \v\rangle |^2}{\|\v\|^4},\label{line:dfdsi}\\
    \frac{\partial f(\s,\z)}{\partial s_{b+\ell}  } &= \frac{\partial f(\s,\z)}{\partial z_{\ell} } = 2 s_\ell\cdot \frac{\langle \bs{s}, \v\rangle}{\|\v\|^2} - 2v_\ell \cdot \frac{|\langle \bs{s}, \v\rangle |^2}{\|\v\|^4}.\label{line:dfdz}
\end{align}
For brevity, we perform the analysis using the more difficult random vector $\S$, and note that all our derivations hold verbatim for the Gaussian $\G$. Also, we will only analyze the term $v_\ell\cdot |\langle \bs{s}, \v\rangle |^2 / \|\v\|^4$ in the partial derivatives, as the other terms follow similarly (and are in fact statistically smaller, hence easier). 

First, recalling that $\Z$ is a zero mean i.i.d.~ Gaussian vector and variances $\sigma^2/p$, we have 
\begin{align}
    \Pr(\|\bv\|^2 \leq \eps ) &\leq \Pr(\|\bv\|_\infty \leq \sqrt{\eps} ) \\
    &= \E \prod_{i=1}^b \Pr(|V_i|  \leq \sqrt{\eps} | S_i) \\ 
    & = \E \prod_{i=1}^b \Pr(|Z_i+S_i|  \leq \sqrt{\eps} | S_i) \label{line:z+s}\\
    & \leq \prod_{i=1}^b \Pr(|Z_i|  \leq \sqrt{\eps}) \label{line:just_z}\\
    & \leq \left(\frac{\eps p}{\sigma^2}\right)^{b/2}
\end{align}
where we have used the standard anti-concentration inequality for the Gaussian distribution. Let us therefore pick $\eps = \Theta(p^{-(1+2/b)})$, under which this event has $o(1/p)$ probability. Furthermore, note that $\Pr(\|\S\|_\infty > \frac{\log^2p}{\sqrt{p}}) = o(1/p)$ as well as $\Pr(\|\Z\|_\infty > \frac{\log^2p}{\sqrt{p}}) = o(1/p)$, and hence $\Pr(\|\V\|_\infty > \frac{\log^2p}{\sqrt{p}}) = o(1/p)$, which all remain order-wise the same when conditioned on the high-probability $1-o(1/p)$ event where $\|\V\|^2 > \eps = \Theta(p^{-(1 + 2/b)})$. Thus, with probability $1-o(1/p)$ we have that 
\begin{align}
    |V_\ell| \cdot \frac{|\langle \S, \V\rangle |^2}{\|\v\|^4} & \leq \|\V\|_\infty \cdot \frac{\|\S\|^2}{\|\V\|^2} \\
    &=  O(\frac{\log^2p}{\sqrt{p}}) \cdot O(\frac{\log^4p}{p}) \cdot  \Theta(p^{1 + 2/b}) \\
    &= O\left(p^{2/b - 1/2}\cdot \log^6 p\right), 
\end{align}
which is bounded for $b>4$. It can be trivially verified that a similar analysis shows the boundedness of the first term in the partial derivatives, and a more subtle version of this analysis yields the same result for any $b\geq 1$. This immediately implies a bounded gradient (and hence Lipschitzness) over the set defined by the norm constraints above, and it is easy to see that this set is also bounded, concluding the proof. 

\end{proof}

 \begin{proof}[Proof of Lemma~\ref{lem:fx_close_to_fy}.]
First, note that that $Z_i$ is independent of $\U_i$, $i=1,\cdots,b$, and therefore  the vectors $\S$ and $\Z$ are independent.
From Lemma~\ref{lem:bounded_poly_ratio_grad} we get that there exist a set $\mathcal{D}$ on which $f$ is $C$-Lipschitz and both $t \triangleq \Pr((\G,\Z)\notin\mathcal{D})= o(1/p)$ and $t' \triangleq \Pr((\S,\Z)\notin\mathcal{D})= o(1/p)$. Let $(\G,\Z)_{\mathcal{D}}$ (resp. $(\S,\Z)_{\mathcal{D}}$) be the random vector $(\G,\Z)$ (resp. $(\S,\Z)$) conditioned on the event $(\G,\Z)\in\mathcal{D}$ (resp. $(\S,\Z)\in\mathcal{D}$) and  $(\G,\Z)_{\bar{\mathcal{D}}}$ (resp. $(\S,\Z)_{\bar{\mathcal{D}}}$) be $(\G,\Z)$ (resp. $(\S,\Z)$) conditioned on the complementary event $(\G,\Z)\notin\mathcal{D}$ (resp. $(\S,\Z)\notin\mathcal{D}$). Then, 
\begin{align}\label{eq:EfX_EfY_diff}
    \Big|\mathbb{E}f(\S,\Z)&-\mathbb{E}f(\G,\Z)\Big| 
    = \Big| \E f((\S,\Z)_{\mathcal{D}})-\E f((\G,\Z)_{\mathcal{D}}) \\
    &+t'\left(\E f((\S,\Z)_{\bar{\mathcal{D}}})-\E f((\S,\Z)_{\mathcal{D}})\right) \\
    &+ t\left(\E f((\G,\Z)_{\mathcal{D}})-\E f((\G,\Z)_{\bar{\mathcal{D}}})\right)\Big|.
\end{align}
Recall that
\begin{align}
    f(\S,\Z)\leq \sum_{i=1}^{b}\SS_i^2\leq b r,
\end{align}
hence
\begin{align}
   \left|t'\left(\E f((\S,\Z)_{\bar{\mathcal{D}}})-\E f((\S,\Z)_{\mathcal{D}})\right)\right|\leq  2b r^2o(1/p).\label{eq:bound_fX}
\end{align}

From Lemma~\ref{lem:bounded_poly_ratio_grad} we get that $f(\s,\z)$ is Lipschitz over $\mathcal{D}\times \mathbb{R}^b$ with some Lipschitz constant $C(a,b)$ that does not depend on $p$. 
Since $\mathcal{D}$ has bounded diameter $C$ and $f$ is $C$-Lipschitz on $\mathcal{D}$ we get
\begin{align}
   \Big| \E f(\S,&\Z)_{\mathcal{D}}-\E f(\G,\Z)_{\mathcal{D}}\Big| \leq  C \cdot \dW\left((\S,\Z)_{\mathcal{D}},(\G,\Z)_{\mathcal{D}}\right)\label{line:using_dual_Wasser}\\
      &\leq C\biggl( \dW_1((\S,\Z),(\G,\Z)) +(t+t')\times \\
      &\qquad\max_{(\s,\z),(\g,\z')\in\mathcal{D}}\|(\s,\z)-(\g,\z')\|\biggr)\label{line:bounding_Wasser_distance_with_R}\\
   &\leq C\biggl( \dW_1(\S,\G) +(t+t')\times \\
   &\qquad\max_{(\s,\z),(\g,\z')\in\mathcal{D}}\|(\s,\z)-(\g,\z')\|\biggr)\label{line:subadditive_wasser}\\
  & \leq  C\cdot \sqrt{\frac{b}{p}}\cdot\frac{2\sqrt{2}b}{p-1}+ C\cdot o(1/p)\label{eq:difference_of_conditional_f}
\end{align}
where in \eqref{line:using_dual_Wasser} we used the dual representation~\eqref{eq:Wasser_dual_rep} of the Wasserstein distance from Theorem~\ref{thrm:Wasser_dual_rep},
 in \eqref{line:bounding_Wasser_distance_with_R} we used  Lemma~\ref{lem:Wasserstein_conditioned} along with the fact that $\mathcal{D}$ has bounded diameter,
in \eqref{line:subadditive_wasser} we used  Proposition~\ref{prop:Wasserstein_properties} along with the fact that $\Z$ is independent of $\S$ and $\G$,
 and in \eqref{eq:difference_of_conditional_f} we used \eqref{eq:XY_Wasserstein_distance}.
Finally, note that since $\G$ is independent Gaussian and the diameter $C$ of $\mathcal{D}$  does not depend on $p$, we have  $t =e^{-\omega(p)}$, we get
\begin{align}\label{eq:bound_f_Y_bar_bar}
    t\E f((\G,\Z)_{\bar{\mathcal{D}}})&\leq t\cdot \sum_{i=1}^b \E \GG_{\bar{\mathcal{D}},i}^2 = b \cdot 2\int_a^\infty y^2e^{-\frac{py^2}{2\alpha}}dy \\
    &\leq \frac{2b\alpha}{p}\cdot(a+1)\cdot e^{-\frac{pa^2}{\alpha}}  .
\end{align}
and also 
\begin{align}\label{eq:bound_f_Y_bar_bar2}
    t\E f((\G,\Z)_{{\mathcal{D}}})\leq t' bC^2 = o(1/p)  .
\end{align}
where we used the relation
\begin{align}
    \int_a^\infty y^2e^{-\frac{py^2}{\alpha}}dy &= -y\cdot \frac{\alpha}{p}e^{-\frac{py^2}{2\alpha}}\Big|_a^\infty+\int_a^\infty \frac{\alpha}{p}e^{-\frac{py^2}{2\alpha}}dy\\
    &\leq a\cdot \frac{\alpha}{p}e^{-\frac{pa^2}{2\alpha}}+ \frac{\alpha^2}{p^2}e^{-\frac{pa^2}{2\alpha}} 
\end{align}
Plugging \eqref{eq:bound_fX}, \eqref{eq:difference_of_conditional_f}, \eqref{eq:bound_f_Y_bar_bar} and \eqref{eq:bound_f_Y_bar_bar2}  back to  \eqref{eq:EfX_EfY_diff} we get 
\begin{align}
    \left|\mathbb{E}f(\S,\Z)-\mathbb{E}f(\G,\Z)\right|  
    &=o(1/p)
\end{align}
\end{proof}

\begin{proof}[Proof of Lemma~\ref{lem:Efy}.]
Recall that $    f(\G,\Z) = \frac{\left(\sum_{i=1}^b  \GG_i\left(\GG_i+ \GG_{b+i}+Z_i\right) \right)^2}{\sum_{i=1}^b \left(\GG_i+ \GG_{b+i}+Z_i \right)^2}$ , where $Z_i\sim\mathcal{N}(0,T_i \sigma^2/p)$, and $T_i$ are concentrated on the interval $[t_\mathrm{min},t_\mathrm{max}]$ with probability $1-o(1/p)$, where
\begin{align}
    t_\mathrm{min}=1-o(1),\;\; t_\mathrm{max}=1+o(1).
\end{align}
Denote $\T = [T_1,\cdots,T_b]$, let $\mathcal{U}$ be the event that $\T\in [t_\mathrm{min},t_\mathrm{max}]^b$, and let  $\bar{\mathcal{U}}$ be its complement. Then, 
\begin{align}
    \E f(\G,\Z) =  \E \left[f(\G,\Z)\Big|\mathcal{U}\right]&+ \Pr(\bar{\mathcal{U}}) \biggl(\E \left[f(\G,\Z)\Big|\bar{\mathcal{U}}\right]\\
    &-\E \left[f(\G,\Z)\Big|\mathcal{U}\right]\biggr).
\end{align}
Using the bound $f(\G,\Z)\leq \sum_{i=1}^b \GG_i^2$ and the fact that $\{\GG_i\}$ are independent of $\T$ and therefore of $\mathcal{U}$, we get
\begin{align}
    \biggl|\Pr(\bar{\mathcal{U}}) \biggl(\E \left[f(\G,\Z)\Big|\bar{\mathcal{U}}\right]-&\E \left[f(\G,\Z)\Big|\mathcal{U}\right]\biggr)\biggr|
    \leq  2\Pr(\bar{\mathcal{U}}) \E \left[\sum_{i=1}^b \GG_i^2\right]\\
    &\leq  4\frac{\alpha}{p}\cdot o(1/p),
\end{align}
which implies that 
\begin{align}\label{eq:fYZ_vs_fYZ_given_U}
  \left|  \E f(\G,\Z) - \E \left[f(\G,\Z)\Big|\mathcal{U}\right]\right| = o(1/p^2).
\end{align}

To estimate $\E[f(\G,\Z)|\mathcal{U}]$, note that $f(\G,\Z)$ depends on $\Z_i$ only through the sum $\GG_i+\GG_{b+i}+Z_i\triangleq A_i$, hence we can write $f(\G,\Z)=\bar{f}(\G,\A)$, for $\A=[A_1,\cdots, A_b]$. Further note that given $\T=[t_1,\cdots,t_b]\triangleq \t$, the variables $A_i$, $i=1,\cdots,b$, are mutually independent and Gaussian. 
 Denote by $\bar{\GG}_{i}$ the random variable obtain from $\GG_i$ by conditioning on the events $\mathcal{U}$, $\A=\a$ and $\T=\t$, then
\begin{align}
   \bar{\GG}_{i} &\sim \mathcal{N}\left(\frac{\alpha r^2/p}{r^2/p+\sigma^2t_i/p}\cdot  a_i,\frac{\frac{\alpha r^2}{p}\left(\frac{(1-\alpha)r^2}{p}+\sigma^2t_i/p\right)}{\frac{r^2}{p} + \sigma^2t_i/p}\right) \\
    & = \mathcal{N}\left(\xi_i \alpha a_i,\, \frac{\alpha \xi_i }{p}\cdot \left((1-\alpha)r^2 + \sigma^2t_i\right)\right),
\end{align}
where $\xi_i=r^2/(r^2+\sigma^2t_i)$. 
Denote $\eta_i \triangleq \frac{\alpha \xi_i }{p}\cdot \left((1-\alpha)r^2 + \sigma^2t_i\right)$, then
\begin{align}
    \E \left[f(\G,\Z)\mid  \mathcal{U}\right]&=\E\left[\E \left[\bar{f}(\bar{\G},\a)\right]\Big|\mathcal{U}\right], 
    \end{align}
    and
    \begin{align}\label{Ef(y|a,t)}
     \E \left[\frac{1}{\|\a\|^2}\left(\sum_{i=1}^b a_i \bar{\GG}_i\right)^2\right]
     &=
     \frac{1}{\|\a\|^2}\left(\left(\sum a_i^2\xi_i\right)^2\alpha^2 + \sum a_i^2\eta_i\right) ,
\end{align}
where we used the fact that $\sum_{i=1}^b a_i \bar{\GG}_i\sim \mathcal{N}(\sum a_i^2\xi_i\alpha , \sum a_i^2 \eta_i) $.
Next, note that $\xi_i\in [\xi_\mathrm{min},\xi_\mathrm{max}]$ and $\eta_i\in [\eta_\mathrm{min},\eta_\mathrm{max}]$, for all $i=1,\cdots,b$, where
\begin{align}
    \xi_\mathrm{min}&\triangleq \frac{r^2}{r^2+t_\mathrm{max}\sigma^2},\\
    \xi_{\mathrm{max}}&\triangleq \frac{r^2}{r^2+t_\mathrm{min}\sigma^2},\\
    \eta_\mathrm{min}&\triangleq \frac{\alpha \xi_\mathrm{min} }{p}\cdot \left((1-\alpha)r^2 + \sigma^2t_\mathrm{min}\right),\\
    \eta_\mathrm{max}&\triangleq \frac{\alpha \xi_\mathrm{max} }{p}\cdot \left((1-\alpha)r^2 + \sigma^2t_\mathrm{max}\right).
\end{align}
Using the above to bound \eqref{Ef(y|a,t)} we get
\begin{align}
    \|\a\|^2\xi_\mathrm{min}^2\alpha^2+\eta_\mathrm{min}\leq \E \bar{f}(\bar{\G},\a) \leq \|\a\|^2\xi_\mathrm{max}^2\alpha^2+\eta_\mathrm{max} ,
\end{align}
Since given $\bs{T}=\bs{t}$ the variables $A_i\sim \mathcal{N}(0,(r^2+t_i\sigma^2)/p)$ are mutually independent, we get  
\begin{align}
    \E \Big[\bar{f}(\G,\A)\mid \bs{T}=\bs{t},\mathcal{U}\Big]  
    &\leq  \xi_\mathrm{max}^2\alpha^2\sum_{i=1}^b\E \left[A_i^2\Big|  \bs{T}=\bs{t},\mathcal{U}\right] + \eta_\mathrm{max} 
    \\& =  \biggl(\xi_\mathrm{max}^2\alpha^2\cdot \sum_{i=1}^b\left(r^2/p+\sigma^2t_i/p\right) \\
    &\quad +  \frac{\alpha \xi_\mathrm{max} }{p}\cdot \left((1-\alpha)r^2 + \sigma^2 t_\mathrm{max}\right)\biggr)
    \\&\leq    \biggl(\xi_\mathrm{max}^2\alpha^2\cdot b\left(r^2/p+\sigma^2t_\mathrm{max}/p\right) \\
    &\quad+  \frac{\alpha \xi_\mathrm{max} }{p}\cdot \left((1-\alpha)r^2 + \sigma^2 t_\mathrm{max}\right)\biggr)
    \\&=   \biggl(\frac{r^2\alpha^2b}{p}\cdot\frac{r^2+\sigma^2t_\mathrm{max}}{\left(r^2+\sigma^2t_\mathrm{min}\right)^2}  \\
    &\quad+  \frac{\alpha \xi_\mathrm{max} }{p}\cdot \biggl((1-\alpha)r^2+ \sigma^2 t_\mathrm{max}\biggr)\biggr)
    \\& = \frac{\alpha r^2}{ p}\left(1  + (b-1)\xi \alpha\right)+o(1/p),
\end{align}
where $\xi$ is the normalized $\snr$, and we used the fact that $t_\mathrm{max}= 1+o(1)$, and
\begin{align}
    \xi_\mathrm{max}& = \frac{r^2}{r^2+\sigma^2-\sigma^2\cdot o(1)}
    =\xi\cdot \frac{1}{1-\frac{\sigma^2}{r^2+\sigma^2}\cdot o(1)}\\
    &= \xi+o(1),\\
    \frac{r^2+\sigma^2t_\mathrm{max}}{\left(r^2+\sigma^2t_\mathrm{min}\right)^2}&=\frac{r^2+\sigma^2}{(r^2+\sigma^2(1-o(1))}+ \frac{\sigma^2\cdot o(1)}{(r^2+\sigma^2(1-o(1))}\\
    &=\frac{1}{r^2+\sigma^2}+o(1).
\end{align}

Similarly, we get the lower bound 
\begin{align}
  \E \left[\bar{f}(\G,\A)\mid \bs{T}=\bs{t},\mathcal{U}\right] &  \geq  \frac{\alpha r^2}{ p}\left(1  + (b-1)\xi \alpha\right)+o(1/p).    
\end{align}
Since the bounds we got does not depend on $\bs{t}$, we get 
\begin{align}
  \E \left[f(\G,\Z)\mid \mathcal{U}\right]&=\E \left[\bar{f}(\G,\A)\mid \mathcal{U}\right]  =  \frac{\alpha r^2}{ p}\left(1  + (b-1)\xi \alpha\right)+o(1/p),   
\end{align}
and combining with \eqref{eq:fYZ_vs_fYZ_given_U} we get the result of lemma.
\end{proof}

\begin{proof}[Proof of Lemma~\ref{lem:semi_noisy_projection}.]
    From Lemma~\ref{lem:fx_close_to_fy} and Lemma~\ref{lem:proj_clos_to_fx} we get that 
          \begin{align}  
 \biggl|  \mathbb{E}\biggl[\frac{ \langle \P\D\bt, \bt\rangle^2}{\|\P\D\bt\|^2}\biggr]&-\mathbb{E}\left[\frac{1}{\delta}f(\G,\Z)\right]\biggr| \leq \frac{2}{\delta\sqrt[4]{p}}\cdot \mathbb{E}\left[f(\G,\Z)\right]+o\left(1/p\right).
   \end{align}
   Plugging in the result of Lemma~\ref{lem:Efy} yields
 \begin{align}  
 \biggl|  \mathbb{E}\left[\frac{ \langle \P\D\bt, \bt\rangle^2}{\|\P\D\bt\|^2}\right]&-\frac{\alpha(1+(b-1)\xi\alpha)}{\delta p}\biggr| \leq \frac{2}{\delta \sqrt[4]{p}}\cdot \left(\frac{\alpha(1+(b-1)\xi\alpha)}{\delta p}\right)+o\left(1/p\right)
 \\&= o\left(1/p\right).
   \end{align}
\end{proof}


\subsection{Proof of the Lemmas in Subsection~\ref{sec:noisy}}\label{appen:variance}

Before we can prove Lemma~\ref{lem:QFFQ}, we need some preliminary results.
In Propositions~\ref{prop:prod_of_sequences} and~\ref{prop:Sherman_morison_ineq} we state two useful technical results. In Lemma~\ref{lem:lambda_max_conv} and Corollary~\ref{cor:lambda_max} we characterize the maximal eigenvalue of an inverse-Wishart matrix, which we then use in Lemma~\ref{lem:norm_AS_conv} to show that the expected excess noise $\E[\W'^T(\H'\H'^T)^{-1}\W']$ converges to the expected value of the asymptotic "clean" statistics $\E[\Q^T(\bs{F}\bs{F}^T)^{-1}\Q]$. Finally, we prove Lemma~\ref{lem:QFFQ} that gives an exact limiting expression for the expected excess noise.

\begin{prop}\label{prop:prod_of_sequences}
    Let $X_n$ and $Y_n$ be two sequences of random variables such that 
    \begin{align}
        X_n \overset{P}{{\longrightarrow}} 0, \;\;\;
         \E|X_n|^2 \longrightarrow 0,\;\text{ and }\; \E Y_n^2 \longrightarrow C,
    \end{align}
    for some constant $C$, then
    \begin{align}
        X_nY_n \overset{P}{{\longrightarrow}} 0. 
    \end{align}
\end{prop}
\begin{proof}
    Let $\eps>0$, then
    \begin{align}
        \Pr\left(|X_nY_n|\geq \eps\right) \leq \frac{\E |X_nY_n|}{\eps}\leq \frac{\sqrt{\E X_n^2\E Y_n^2}}{\eps} \rightarrow 0,
    \end{align}
    where the first transition due to Markov's inequality and the second is due to Cauchy--Schwarz.
\end{proof}

\begin{prop}\label{prop:Sherman_morison_ineq}
    Let $\F$ be as in Lemma~\ref{lem:norm_AS_conv} and denote by $\bs{t}$ its first column, and by $\F'$ its last $p-1$ columns i.e.,
\begin{align}
    \F= \left[\bs{t}\;\; \F'\right].
\end{align}
Then, for any vector $\bx$, 
\begin{align}
    \bx^T\left(\bs{F}\bs{F}^T\right)^{-1}\bx \leq \bx^T\left(\F'\F'^T\right)^{-1}\bx.
\end{align}

\end{prop}
\begin{proof}
   Denote 
\begin{align}
    \bs{H} = \frac{1}{1+\bs{t}^T\left(\F'\F'^T\right)^{-1}\bs{t}}\left(\F'\F'^T\right)^{-1}\bs{t}\bs{t}^T\left(\F'\F'^T\right)^{-1},
\end{align}
and note that since $\left(\F'\F'^T\right)^{-1}$ and $\bs{t}\bs{t}^T$ are p.s.d.~matrices, then $\bs{H}$ is also p.s.d.
Then
\begin{align}
    \bx^T\left(\bs{F}\bs{F}^T\right)^{-1}\bx
   & =  \bx^T\left[\left(\F'\F'^T\right)^{-1}-\bs{H}\right]\bx\label{line:sherman_morisson}\\
   &= \bx^T\left(\F'\F'^T\right)^{-1}\bx-\bx^T\bs{H}\bx\label{line:QFFQ}\\
   &\leq \bx^T\left(\F'\F'^T\right)^{-1}\bx\label{line:psd_matrix},
\end{align}
where in \eqref{line:sherman_morisson} we used the Sherman Morrison inversion lemma and in \eqref{line:psd_matrix} we used that fact that $\bs{H}$ is a p.s.d.~matrix.
\end{proof}

\begin{lemma}\label{lem:lambda_max_conv}
     Let $W_n$ be an $n\times n$  Wishart matrix with covariance $\I$ and $\gamma n$ degrees of freedom, $\gamma>1$. Denote by
    \begin{align}
        \lambda_n \triangleq \frac{\lambda_{\text{min}}(W_n)}{n},
    \end{align}
    the (normalized) minimal eigenvalue of $W_n$. 
    Let $g(x):\RR^+\to\RR$ be a continuous monotonically decreasing function with $\lim_{x\rightarrow\infty}g(x) = 0^+$.
    Then
    \begin{align}
        g(\lambda_n)\inAS g\left(\gamma\left(1-\frac{1}{\sqrt{\gamma}}\right)^2\right).
    \end{align}
    Moreover, if there exist two constants constant $c_1>0$ and $0<c_2<\sqrt{\gamma}-1$ that depends only on $\gamma$ such that
    \begin{align}
    \textbf{Condition $1$:}\quad  &g(c_1/n) = o\left(e^{\frac{(\sqrt{\gamma}-1-c_2)^2n}{2}}\right),\\
    \textbf{Condition $2$:}\quad &    \int_0^{c_1} g\left(\lambda/n\right)\lambda^{-\frac{1}{2}((\gamma-1)n-1)}e^{-\frac{\lambda}{2}}d\lambda \\
    &\qquad=o\left(n^{\frac{1}{2}(\gamma-1)n}\right),
    \end{align}
    then $g(\lambda_n)$ are uniformly integrable and
    \begin{align}
       \E g(\lambda_n)\longrightarrow g\left(\gamma\left(1-\frac{1}{\sqrt{\gamma}}\right)^2\right).
    \end{align}
\end{lemma}

\begin{proof}
    The almost sure convergence is immediate from the fact that \cite{silverstein1985smallest}
    \begin{align}
        \lambda_n \inAS \gamma\left(1-\frac{1}{\sqrt{\gamma}}\right)^2,
    \end{align}
    and $g(x)$ is continuous. 
    Next, we show that $g(\lambda_n)$ are uniformly integrable. 
    If $\lim_{x\to 0}g(x)<\infty$ then $g$ is bounded and uniform integrability is trivial. Otherwise, since the pdf $f_n(\lambda)$ of $\lambda_n$ exists and continuous for any $n$ \cite{edelman1988eigenvalues}, it is enough to verify that for some $a>0$ that depends only on $\gamma$ 
    \begin{align}\label{eq:UIcond}
        \lim_{n\rightarrow \infty}\E\left[|g(\lambda_n)|\cdot \indfunc{|g(\lambda_n)|\geq a}\right]=0.
    \end{align}
    Since $g$ is monotonically decreasing, this is equivalent to showing that for some $\eps>0$ that depends only on $\gamma$
        \begin{align}
        \lim_{n\rightarrow \infty}\int_{0}^{\eps\cdot n}g\left(\lambda/n\right)f_{n}(\lambda)d\lambda= 0.
    \end{align}
 Let us write
\begin{align}\label{eq:integral_breaking}
    \int_{0}^{\eps\cdot n}g\left(\lambda/n\right)f_n(\lambda)d\lambda &= \int_{0}^{c_1}g\left(\lambda/n\right)f_n(\lambda)d\lambda\\&
    \quad+\int_{c_1}^{\eps\cdot n}g\left(\lambda/n\right)f_n(\lambda)d\lambda.
\end{align}
We have
\begin{align}
    \int_{c_1}^{\eps\cdot n}g\left(\lambda/n\right)f_n(\lambda)d\lambda\leq g\left(c_1/n\right)\cdot \Pr\left(\lambda_n\leq \epsilon n\right)
\end{align}
    Since $\lambda_n$ is the minimal eigenvalue of a Wishart matrix we have that for any $0<\eps<(\sqrt{\gamma}-1)^2$ \cite{vershynin2010introduction}
    \begin{align}
        \Pr\left(\lambda_n\leq \epsilon n\right) \leq 2e^{-\frac{(\sqrt{\gamma}-1-\sqrt{\eps})^2n}{2}}.
    \end{align}
    Then, from Condition~1 we get that for any $\eps\leq c_2^2$     \begin{align}
      g\left(1/n\right)\cdot \Pr\left(\lambda_n\leq \epsilon n\right) =o(1).
    \end{align}
    To show the decay of the first addend in \eqref{eq:integral_breaking}, note that from \cite{edelman1988eigenvalues} we have
    \begin{align}
        f_{n}(\lambda)\leq C_n \lambda^{\frac{1}{2}((\gamma-1)n-1)}e^{-\frac{\lambda}{2}},
    \end{align}
    where 
    \begin{align}
        C_n = \frac{\pi^\frac{1}{2}2^{-\frac{1}{2}((\gamma-1)n+1)}\Gamma\left(\frac{\gamma n -1}{2}\right)}{\Gamma\left(\frac{n }{2}\right)\Gamma\left(\frac{(\gamma-1) n +1}{2}\right)\Gamma\left(\frac{(\gamma-1) n +2}{2}\right)} = O\left(n^{-\frac{1}{2}(\gamma-1)n}\right).
    \end{align}
     Hence if Condition~2 holds we have
    \begin{align}
        \int_0^{c_1}g&\left(\frac{\lambda}{n}\right)f_{\lambda,n}d\lambda\leq \\ &C_n   \int_0^{c_1} g\left(\frac{\lambda}{n}\right)\lambda^{\frac{1}{2}((\gamma-1)n-1)}e^{-\frac{\lambda}{2}}d\lambda =o\left(1\right),
    \end{align}
    which concludes the proof.
\end{proof}

\begin{cor}\label{cor:lambda_max}
    Let $Q_n$ be an $n\times n$ matrix distributed according to the inverse-Wishart distribution with covariance $\I$ and $\gamma n$ degrees of freedom. Then for any fixed $r\geq 1$ we have 
    \begin{align}
        \left(n\lambdamax\left(Q_n\right)\right)^r \inAS  \left(\frac{1}{\gamma\left(1-\sqrt{\frac{1}{\gamma}}\right)^2}\right)^r ,
    \end{align}
    and 
        \begin{align}
       \E \left(n\lambdamax\left(Q_n\right)\right)^r \longrightarrow \left(\frac{1}{\gamma\left(1-\sqrt{\frac{1}{\gamma}}\right)^2}\right)^r .
    \end{align}
\end{cor}
\begin{proof}
    An immediate result of Lemma~\ref{lem:lambda_max_conv} with $g(x) = 1/x^r$, $c_1 =1$ and any $c_2<\sqrt{\gamma}-1$. 
\end{proof}

\begin{proof}[Proof of Lemma~\ref{lem:norm_AS_conv}.]
For convenience, we normalize the coefficient $\A_j$ such that the modified elements are given by
    \begin{align}\label{eq:normalized_Wtag}
      \Ytag_j &= \frac{1}{\|\A_j\|}\sum_{i=1}^b A_{ij}\YY_{ij}, \;\;   \h'_j \\
      &= \frac{1}{\|\A_j\|}\sum_{i=1}^b A_{ij}\h_{ij}, \;\;  W'_j \\
      &= \frac{1}{\|\A_j\|}\sum_{i=1}^b A_{ij}W_{ij}, 
    \end{align}
    and similarly define 
        \begin{align}\label{eq:normalized_Q}
        \bs{f}_j = \frac{1}{\|\X_j\|}\sum_{i=1}^b \YY_{ij}\h_{ij}, \;\;  Q_j = \frac{1}{\|\X_j\|}\sum_{i=1}^b \YY_{ij}W_{ij}. 
    \end{align}
    Note that these normalization does not change $\Q^T\left(\F\F^T\right)^{-1}\Q$ nor \\$\W'^T\left(\H'\H'^T\right)^{-1}\W'$.
Next, by applying the strong law of large numbers per entry to the $b\times b$ matrix $\frac{1}{p}\H_j\H_j^T$ we get that 
\begin{align}
    \frac{1}{p}\H_j\H_j^T\inAS \I,
\end{align}
and since the $b$ eigenvalues of $\frac{1}{p}\H_j\H_j^T$ are all equal to $1$ almost surely   we also get
\begin{align}
    \left(\frac{1}{p}\H_j\H_j^T\right)^{-1}\inAS \I.
\end{align}
This yields   
\begin{align}
  \left \| p\A_j-\X_j\right\|_F^2&=\left \| p(\H_j\H_j^T)^{-1} \X_j-\X_j\right\|_F^2\\
  &= \left\| \left((\frac{1}{p}\H_j\H_j^T)^{-1} -I\right)\X_j\right\|_F^2\\
 & \leq \left\| \X_j\right\|_F^2\left\|\left((\frac{1}{p}\H_j\H_j^T)^{-1} -I\right)\right\|_F^2\\
 &\inP 0,
\end{align}
where we used Proposition~\ref{prop:prod_of_sequences} along with the fact that $\X_j\sim \mathcal{N}(\bs{0},(r^2+\sigma^2)\I)$, hence $\E\|\X_j\|_F^4$ is finite and does not depend on $p$, and
\begin{align}
  p\A_j\inP \X_j.
\end{align}
Denote $\bar{\X}_j =\X_j/\|\X_j\|$ and $\bar{\A}_j =p\A_j/\|p\A_j\|$. Since $\|p\A_j\|^2$ and $\|\X_j\|^2$ are continuous mappings of $\A_j$ and $\X_j$, and are also bounded away from zero almost surely, we get   
\begin{align}
  \bar{\A}_j\inP\bar{\X}_j.
\end{align}
Moreover, since $\bar{\A}_j$ and $\bar{\X}_j$ have unit norms, their elements are bounded with probability $1$ and therefore they have bounded $r$th moment, for any $r\geq 0$. Hence, using Theorem 4.6.2 in~\cite{durrett2019probability} we can deduce that $\bar{\A}_j$  converges to $\bar{\X}_j$ in the $r$th mean, that is, for any $r>0$
\begin{align}
    \|\bar{\A}_j-\bar{\X}_j\|^r\inP 0,&\;j=1,\cdots,n/b\\
        \E\|\bar{\A}_j-\bar{\X}_j\|^r\longrightarrow 0&,\;j=1,\cdots,n/b.
\end{align}
Next, note that for any $r\geq 2$
\begin{align}
 0\leq   \left \|\frac{1}{\sqrt{p}}\h'_j-\frac{1}{\sqrt{p}}\bs{f}_j\right\|_F^r &= \left\|\left(\bar{\A}_j-\bar{\X}_j\right)^T \cdot  \frac{1}{\sqrt{p}} \H_j\right\|_F^r
    \\&  \leq \left\|\bar{\A}_j-\bar{\X}_j\right\|_F ^r\cdot \left\|\frac{1}{\sqrt{p}} \H_j\right\|_F^r \\ &\inP 0,
\end{align}
where we used the fact that $ \|\frac{1}{\sqrt{p}}\H_j\|^r$ is almost surely bounded since $\H_j$ are i.i.d.~standard Gaussian.
In a similar way 
\begin{align}
 0\leq   \E\left \|\frac{1}{\sqrt{p}}\h'_j-\frac{1}{\sqrt{p}}\bs{f}_j\right\|_F^r 
    &\leq \E\left\|\bar{\A}_j-\bar{\X}_j\right\|_F ^r \E\left\|\frac{1}{\sqrt{p}} \H_j\right\|_F^r \\
    &{\longrightarrow}0,
\end{align}
where we used the fact that $\E \|\frac{1}{\sqrt{p}}\H_j\|^r\rightarrow b^r$ for any $r\geq 2$.
One important observation from the above is that for any $r\geq 2$, there exists some constant $C_r$ that does not depend on $p$, such that  $\E\|\frac{1}{\sqrt{p}}\h'_j-\frac{1}{\sqrt{p}}\bs{f}_j\|^r\leq C_r$ for any $p$.
This will yield that
\begin{align}\label{eq:F_H_conv}
    \left\|\frac{1}{\sqrt{\sqrt{n}p}}\F-\frac{1}{\sqrt{\sqrt{n}p}}\H'\right\|_F^2&=\frac{1}{\sqrt{n}}\sum_{j=1}^{n/b}\left \|\frac{1}{\sqrt{p}}\h'_j-\frac{1}{\sqrt{p}}\bs{f}_j\right\|_F^2\\
    &\inP 0.
\end{align}
where we used the fact that 
\begin{align}
    \E\frac{1}{\sqrt{n}}\sum_{j=1}^{n/b}\left \|\frac{1}{\sqrt{p}}\h'_j-\frac{1}{\sqrt{p}}\bs{f}_j\right\|_F^2&=0,
   \end{align} 
   and
    \begin{align}
        \E\biggl(\frac{1}{\sqrt{n}}\sum_{j=1}^{n/b}\biggl \|&\frac{1}{\sqrt{p}}\h'_j-\frac{1}{\sqrt{p}}\bs{f}_j\biggr\|_F^2\biggr)^2=\frac{1}{b}\E\left( \left \|\frac{1}{\sqrt{p}}\h'_1-\frac{1}{\sqrt{p}}\bs{f}_1\right\|_F^2\right)^2\longrightarrow 0.
   \end{align} 
In a similar way 
\begin{align}
          \|  W'_j- Q_j\|_2^r &=\left\|\left(\bar{\X}_j-\bar{\A}_j\right) \cdot \W_j\right\|_F^r \inAS 0,\\
                    \E\|  W'_j- Q_j\|_2^r &=\E\left\|\left(\bar{\X}_j-\bar{\A}_j\right) \cdot \W_j\right\|_F^r \inAS 0.\label{eq:bounded_noise_var}
\end{align}
 where we used the fact that for any $j$ the vector $\W_j$ is a $b\times 1$ i.i.d.~Gaussian vector.
Let us write
\begin{align}
    \left\|\frac{1}{\sqrt[4]{n}}\W'-\frac{1}{\sqrt[4]{n}}\Q\right\|_F^2 =\frac{1}{\sqrt{n}} \sum_{j=1}^{n/b} \|W'_j- Q_j\|^2.
\end{align}
 From \eqref{eq:bounded_noise_var} we get that the mean of the above is $\E\frac{1}{\sqrt{n}} \sum_{j=1}^{n/b} \|W'_j- Q_j\|^2 =0$ and the second moment is 
 \begin{align}
    \E\left( \frac{1}{\sqrt{n}} \sum_{j=1}^{n/b} \|W'_j- Q_j\|^2\right)^2= \frac{1}{b}\E\left(\|W'_1- Q_1\|^2\right)^2\longrightarrow 0,
 \end{align}
where we used the facet that $\W'$ and $\Q$ have i.i.d.~entries. 
Hence
\begin{align}
   \left\|\frac{1}{\sqrt[4]{n}}\W'-\frac{1}{\sqrt[4]{n}}\Q\right\|_F^2 \inP 0,
\end{align}
and therefore
\begin{align}\label{eq:QW_conv}
   \left\|\frac{1}{\sqrt[4]{n}}\W'-\frac{1}{\sqrt[4]{n}}\Q\right\|_F^2 \inP 0,
\end{align}
From \eqref{eq:F_H_conv} it easily follow that
\begin{align}
    \left\|\frac{1}{n^{\frac{3}{2}}}\F\F^T-\frac{1}{n^{\frac{3}{2}}}\H'\H'^T\right\|_F^2\inP 0.
\end{align}
Then, note that using the relation $\|\A\B\|_F^2\leq\lambdamax(\A)\|\B\|_F^2$ we get
\begin{align}\label{eq:l2_lF_decomposition}
    \Big\|\left(\frac{1}{\sqrt{n}}\F\F^T\right)^{-1}&-\left(\frac{1}{\sqrt{n}}\H'\H'^T\right)^{-1}\Big\|_F^2 \leq  
    \\ &\hspace{-1.5cm}\lambdamax\left(\sqrt{n}\left(\F\F^T\right)^{-1}\right)\lambdamax\left(\sqrt{n}\left(\H'\H'^T\right)^{-1}\right) \\
    &\times\left\|\frac{1}{n^{\frac{3}{2}}}\F\F^T-\frac{1}{n^{\frac{3}{2}}}\H'\H'^T\right\|_F^2.
    \end{align}
Note that with the exception of the first column, $\F$ is i.i.d.~standard-Gaussian. Then, from Proposition~\ref{prop:Sherman_morison_ineq} we get 
\begin{align}\label{eq:lambda_max_F_Ftag}
    \lambdamax\left(\sqrt{n}\left(\F\F^T\right)^{-1}\right)\leq  \lambdamax\left(\sqrt{n}\left(\F'\F'^T\right)^{-1}\right).
\end{align}
The entries of $\F'$ are i.i.d $\mathcal{N}$(0,1), therefore, $\F'\F'^T$ is a Wishart matrix with $p-1$ degrees of freedom. From Corollary~\ref{cor:lambda_max}
\begin{align}     \lambdamax\left(\sqrt{n}\left(\F'\F'^T\right)^{-1}\right)\inAS 0,
\end{align}
which implies 
\begin{align}
     \lambdamax\left(\sqrt{n}\left(\F\F^T\right)^{-1}\right)\inAS 0.
\end{align}
Since each row of $\H'$ is a linear combination of $b$ rows of $\H$ with coefficients of unit norm, and $\H\H^T$ is Wishart matrix, we also have  
\begin{align}
  \lambdamax\left(\sqrt{n}\left(\H'\H'^T\right)^{-1}\right)\leq \lambdamax\left(\sqrt{n}\left(\H\H^T\right)^{-1}\right)\inAS 0,
\end{align}
and therefore, we get from  \eqref{eq:l2_lF_decomposition} that
\begin{align}\label{eq:inv_conv}
    \left\|\left(\frac{1}{\sqrt{n}}\F\F^T\right)^{-1}-\left(\frac{1}{\sqrt{n}}\H'\H'^T\right)^{-1}\right\|_F^2\inP 0.
\end{align}
Finally,  we combine \eqref{eq:QW_conv} and \eqref{eq:inv_conv} to get 
\begin{align}
    \biggl\|\frac{1}{\sqrt[4]{n}}&\Q^T\left(\frac{1}{\sqrt{n}}\F\F^T\right)^{-1}\frac{1}{\sqrt[4]{n}}\Q\\
    &- \frac{1}{\sqrt[4]{n}}\W'^T\left(\frac{1}{\sqrt{n}}\H'\H'^T\right)^{-1}\frac{1}{\sqrt[4]{n}}\W'\biggr\|_F^2
    \inP 0.
\end{align}
which implies 
\begin{align}\label{eq:var_conv}
    \Q^T\left(\F\F^T\right)^{-1}\Q-\W'^T\left(\H'\H'^T\right)^{-1}\W'\inP 0.
\end{align}
To deduce $L^1$ convergence note the from \eqref{eq:lambda_max_F_Ftag} and Corollary~\ref{cor:lambda_max} we get
\begin{align}
    \E\lambdamax\left(\left(\frac{1}{n}\F\F^T\right)^{-1}\right)\leq  \E\lambdamax\left(\left(\frac{1}{n}\F'\F'^T\right)^{-1}\right)\leq \infty,
\end{align}
and
\begin{align}
    \E\lambdamax\left(\left(\frac{1}{n}\H'\H'^T\right)^{-1}\right)&\leq \E\lambdamax\left(\left(\frac{1}{n}\H\H^T\right)^{-1}\right) \\
    &\leq\infty.
\end{align}
Then
\begin{align}
    \E\|\Q^T&\left(\F\F^T\right)^{-1}\Q- \W'^T\left(\H'\H'^T\right)^{-1}\W'\|^2
    \\
    &\leq  \E \|\Q^T\left(\F\F^T\right)^{-1}\Q\|^2 \label{line:positive}\\
  \;\;& \qquad +\E \|\W'^T\left(\H'\H'^T\right)^{-1}\W'\|^2\\
    &\leq  \E\left[\frac{1}{n^2/b^2}\|\Q\|^4\lambdamax^2\left(\left(\frac{1}{n/b}\F\F^T\right)^{-1}\right)\right]
  \\&  \;\;+\E\left[\frac{1}{n^2/b^2}\|\W'\|^4\lambdamax^2\left(\left(\frac{1}{n/b}\H'\H'^T\right)^{-1}\right)\right]\label{line:max_lambda}\\
    &\leq  \E\left[\frac{1}{n^2/b^2}\|\Q\|^4\lambdamax\left(\left(\frac{1}{n/b}\F\F^T\right)^{-2}\right)\right]
  \\&  \;\;+\E\left[\frac{1}{n^2/b^2}\|\W'\|^4\lambdamax\left(\left(\frac{1}{n/b}\H\H^T\right)^{-2}\right)\right]\\
    &\leq  \E\left[\frac{b^4}{n^4}\|\Q\|^8\right]\E\left[\lambdamax\left(\left(\frac{1}{n/b}\F\F^T\right)^{-4}\right)\right]
  \\&  \;\;+\E\left[\frac{b^4}{n^4}\|\W'\|^8\right]\E\left[\lambdamax\left(\left(\frac{1}{n/b}\H\H^T\right)^{-4}\right)\right]\label{line:CS2}\\
    &<\infty\label{line:lambda_prop},
\end{align}
where in \eqref{line:positive} we used the fact that $\Q^T\left(\F\F^T\right)^{-1}\Q$ and $\W'^T\left(\H'\H'^T\right)^{-1}\W'$ are non-negative, in \eqref{line:max_lambda} we used $\x^TA\x\leq \|\x\|^2\lambdamax(A)$ for any p.s.d $A$,  in \eqref{line:CS2} we used Cauchy--Schwarz inequality, and in \eqref{line:lambda_prop} we used Corollary~\ref{cor:lambda_max}.
Hence we get that   \eqref{eq:var_conv} is uniformly integrable and
\begin{align}
    |\Q^T\left(\F\F^T\right)^{-1}\Q- \W'^T\left(\H'\H'^T\right)^{-1}\W'|\inL1 0,
\end{align}
as desired.
\end{proof}

The following lemma characterizes the covariance of the modified noise vector $\Q$.

\begin{lemma}\label{lem:Q_variance}
    Let $\Q$ be as in Lemma~\ref{lem:norm_AS_conv}. Then 
    \begin{align}
        \mathbb{E}\left[\Q\Q^T\right]=\sigma^2\cdot \left(\frac{\sigma^2 b + r^2 - \sigma^2q^2}{\sigma^2+r^2}\cdot I +  \frac{\sigma^2q^2}{\sigma^2+r^2}\cdot \bs{1}\bs{1}^T\right),
\end{align}
with $q\triangleq \mathbb{E}\sqrt{B}$, where $B$ is a random variable that is distributed according to the $\chi^2$ distribution with $b$ degrees of freedom.
\end{lemma}

\begin{proof}
Recall that inside the $j$th batch, we have
\begin{align}
    \YY_{ij}= r\cdot \alpha_{ij} + W_{ij} \sim \mathcal{N}(0, r^2 + \sigma^2),
\end{align}
where $\alpha_{ij}$ is the first entry of the random vector $\h_{ij}$.
Therefore, $W_{ij}| \X \sim \mathcal{N}\left(\frac{\sigma^2\cdot \YY_{ij}}{\sigma^2+r^2}, \frac{\sigma^2r^2}{\sigma^2+r^2}\right)$ and $W_{ij}| \X $, $i=1.\cdots,b$, are mutually independent. 
Then we get
\begin{align}
    Q_j|\X \sim \mathcal{N}\left(\frac{\sigma^2\cdot \|\X_j\|}{\sigma^2+r^2}, \, \frac{\sigma^2r^2}{\sigma^2+r^2}\right). 
\end{align}
The $(i,j)$-th entry of the covariance matrix of $\Q$ given $\X$ is then
\begin{align}
    \left[\E (\Q\Q^T | \X)\right]_{ij} = \sigma^2\cdot \left\{\begin{array}{cc}
         \frac{\sigma^2\|\X_i\|\|\X_j\|}{(\sigma^2+r^2)^2} & i\neq j \\
         \frac{\sigma^2 \frac{\|\X_j\|^2}{\sigma^2+r^2} + r^2}{\sigma^2+r^2} & i = j 
    \end{array}\right..
\end{align}
Note that $\frac{\|\X_j\|^2}{r^2+\sigma^2}\sim \chi^2(b)$  has a $\chi^2$-distribution with $b$ degrees of freedom and is independent between different batches.
Therefore, $\E \|\X_j\|^2/(r^2+\sigma^2) = b$, and the unconditional covariance is given by 
\begin{align}
    \E (\Q\Q^T) = \sigma^2\cdot \left(\frac{\sigma^2 b + r^2 - \sigma^2q^2}{\sigma^2+r^2}\cdot I +  \frac{\sigma^2q^2}{\sigma^2+r^2}\cdot \bs{1}\bs{1}^T\right) 
\end{align}
\end{proof}

The following lemma shows that when the first column of the modified features (i.e., the direction of the parameters vector) is removed, we get weak convergence properties for the excess noise.

\begin{lemma}\label{lem:QFQ_converge}
    Let $\Q$ and $\F$ be as in Lemma~\ref{lem:norm_AS_conv} and let $\t$ be the first column of $\F$, and $\F'$ be the matrix obtained from $\F$ by removing $\t$. Then 
    \begin{align}
        \Q^T\left(\F'\F'^T\right)^{-1}\Q&\inL1  \sigma^2\cdot \frac{b-(b-1)\xi}{(\gamma b-1)},
        \\ \t^T\left(\F'\F'^T\right)^{-1}\t&\inL1 \frac{1+(b-1)\xi}{(\gamma b-1)},
    \end{align}
\end{lemma}

\begin{proof}
We prove for $\Q^T\left(\F'\F'^T\right)^{-1}\Q$, the proof for $\t^T\left(\F'\F'^T\right)^{-1}\t$ follows the same steps. 
First note that normalizing $Q_j$ and $\bs{f}_j$ in Lemma~\ref{lem:norm_AS_conv} by $\|\X_j\|$ does not change the value of $\Q^T\left(\F'\F'^T\right)^{-1}\Q$. Therefore we prove the above for the normalized $\Q$ and $\F'$ as given in \eqref{eq:normalized_Q}.   
 Recall that $\F'$ is independent of $\X$ and therefore of $\Q$.  Then
    \begin{align}
   \frac{1}{n/b}\E\Q^T&\left(\frac{1}{n/b}\F'\F'^T\right)^{-1}\Q
   =  \frac{1}{n/b}\E\left[\Q^T\E\left(\frac{1}{n/b}\F'\F'^T\right)^{-1}\Q\right]
 \\&=  
 \frac{1}{\gamma b -1-\frac{2b}{n}}\cdot \frac{1}{n/b}\E\left[\Q^T\Q\right]
\\  & =  \frac{1}{\gamma b -1-\frac{2b}{n}}\cdot \sigma^2\frac{\sigma^2b+r^2}{\sigma^2+r^2}
    \\& \longrightarrow \frac{\sigma^2(\sigma^2b+r^2)}{(\gamma b-1) (\sigma^2+r^2)}\\
    &= \sigma^2\cdot \frac{b-(b-1)\xi}{\gamma b-1},
\end{align}
where we used the fact that $(\F'\F'^T)^{-1}$ is an inverse-Wishart matrix and therefore \cite{mardia1979multivariate}
\begin{align}
    \E\left[ \left(\frac{1}{n/b}\F'\F'^T \right)^{-1}\right] =\frac{1}{\gamma b -1-\frac{2b}{n}}\I.
\end{align}
Next, we show the convergence in probability, then because $\Q^T\left(\F'\F'^T\right)^{-1}\Q$ are none negative, $L^1$ convergence will follow.  
To that end, we will show that the variance of the variable $\Q^T(\F'\F'^T)^{-1}\Q$ decays, that is, \\$\Var\left(\Q^T(\F'\F'^T)^{-1}\Q\right)\rightarrow 0$.
Recall that $ Q_j|\X \sim \mathcal{N}\left(\frac{\sigma^2\|Y_j\|}{\sigma^2+r^2}, \, \frac{\sigma^2r^2}{\sigma^2+r^2}\right)$. Therefore:
\begin{align}
    \mu_2 &\triangleq \E\left[Q_i-\E\left[Q_i|\X\right]\Big|\X\right]^2 =\frac{\sigma^2r^2}{\sigma^2+r^2}
    \\ \mu_3 &\triangleq \E\left[Q_i-\E\left[Q_i|\X\right]\Big|\X\right]^3= 0,
    \\ \mu_3 &\triangleq \E\left[Q_i-\E\left[Q_i|\X\right]\Big|\X\right]^4 = 3\mu_2^2.
\end{align}
Then we have \cite{seber2003linear}
\begin{align}
    \Var\biggl(\frac{1}{n/b}&\Q^T\biggl(\frac{1}{n/b}\F'\F'^T\biggr)^{-1}\Q\Big| \X\biggr)
     =\frac{1}{n^2}\Big( 2\mu_2^2\trace\left(\E\left(\frac{1}{n/b}\F'\F'^T\right)^{-2}\right)
        \\&\qquad +4\mu_2\bs{\nu}^T\E\left(\frac{1}{n/b}\F'\F'^T\right)^{-2}\bs{\nu}\Big),
\end{align}
where $\bs{\nu}=\E\Q|\X$, and we used the fact that $\F'$ is independent of $\Q$. Since $(\frac{1}{n/b}\F'\F'^T)^{-1}$ is inverse Wishart we have that all the off-diagonal elements of $\E[(\frac{1}{n/b}\F'\F'^T)^{-2}]$ are $O(1/n)$ and the diagonal elements are $\Theta(1)$ \cite{mardia1979multivariate}. Moreover, the entries of $\E[\bs{\nu}\bs{\nu}^T]$ are all $\Theta(1)$ (and do not depend on $n$), therefore we get
\begin{align}
     \Var\left(\frac{1}{n/b}\Q^T\left(\frac{1}{n/b}\F'\F'^T\right)^{-1}\Q\right)& = O(1/n).
\end{align}
Convergence in probability then follows again from Chebyshev's inequality.
\end{proof}

We are now ready to prove the upper and lower bounds of the excess noise.
\begin{proof}[Proof of Lemma~\ref{lem:QFFQ}.]
Denote 
\begin{align}\label{eq:F_tag_t}
    \F= \left[\bs{t}\;\; \F'\right],
\end{align}
that is, $\bs{t}$ is the first column of $\F$ and $\F'$ is the matrix composed from the remaining $p-1$ columns of $\F$. 
Then, from Proposition~\ref{prop:Sherman_morison_ineq} we get
\begin{align}\label{eq:QFQ_QHQ}
    \Q^T\left(\bs{F}\bs{F}^T\right)^{-1}\Q=\Q^T\left(\F'\F'^T\right)^{-1}\Q-\Q^T\bs{H}\Q,
\end{align}
with 
\begin{align}
    \bs{H} =& \frac{1}{n/b+\bs{t}^T\left(\frac{1}{n/b}\F'\F'^T\right)^{-1}\bs{t}}\times
    \\ &\left(\frac{1}{n/b}\F'\F'^T\right)^{-1}\bs{t}\bs{t}^T\left(\frac{1}{n/b}\F'\F'^T\right)^{-1}.
\end{align}
From Lemma~\ref{lem:QFQ_converge} we get 
\begin{align}\label{eq:EQFQ}
 \E\Q^T\left(\F'\F'^T\right)^{-1}\Q \longrightarrow \sigma^2\cdot \frac{b-(b-1)\xi}{\gamma b-1},
\end{align}
which along with \eqref{eq:QFQ_QHQ} establishes the upper bound.
Next,
\begin{align}
    \E\left[\Q^T\bs{H}\Q\right] &=\E\left[\frac{\left(\bs{t}^T\left(\F'\F'^T\right)^{-1}\Q\right)^2}{1+\bs{t}^T\left(\F'\F'^T\right)^{-1}\bs{t}}\right] 
    \\& \hspace{-0.5cm}\leq \E\left[\Q^T\left(\F'\F'^T\right)^{-1}\Q\cdot\frac{\bs{t}^T\left(\F'\F'^T\right)^{-1}\t}{1+\bs{t}^T\left(\F'\F'^T\right)^{-1}\bs{t}}\right] 
\end{align}
From Lemma~\ref{lem:QFQ_converge} we get that $\bs{t}^T\left(\F'\F'^T\right)^{-1}\t$ and $\Q^T\left(\F'\F'^T\right)^{-1}\Q$ converges in $L^1$ to a constant limit, therefore
\begin{align}
\Q^T\left(\F'\F'^T\right)^{-1}&\Q\cdot\frac{\bs{t}^T\left(\F'\F'^T\right)^{-1}\t}{1+\bs{t}^T\left(\F'\F'^T\right)^{-1}\bs{t}}
\inP \sigma^2\cdot \frac{b-(b-1)\xi}{\gamma b-1}\cdot \frac{\frac{1+(b-1)\xi}{\gamma b-1}}{1+\frac{1+(b-1)}{\gamma b-1}}.
\end{align}
Moreover, 
\begin{align}   \label{eq:CS} 
\Q^T\left(\F'\F'^T\right)^{-1}&\Q\cdot\frac{\bs{t}^T\left(\F'\F'^T\right)^{-1}\bs{t}}{1+\bs{t}^T\left(\F'\F'^T\right)^{-1}\bs{t}}
 \leq \Q^T\left(\F'\F'^T\right)^{-1}\Q.
\end{align}
Since $\Q^T\left(\F'\F'^T\right)^{-1}\Q$ converges in $L^1$, the l.h.s of \eqref{eq:CS} also uniformly converges in $L^1$, and we get
\begin{align}
\E\biggl[\Q^T\left(\F'\F'^T\right)^{-1}&\Q\cdot\frac{\bs{t}^T\left(\F'\F'^T\right)^{-1}\t}{1+\bs{t}^T\left(\F'\F'^T\right)^{-1}\bs{t}}\biggr]
\to \sigma^2\cdot \frac{b-(b-1)\xi}{\gamma b-1}\cdot \frac{\frac{1+(b-1)\xi}{\gamma b-1}}{1+\frac{1+(b-1)}{\gamma b-1}},
\end{align}
which yields
\begin{align}\label{eq:EQHQ}
   \lim_{p\to\infty} \E\left[\Q^T\bs{H}\Q\right]\leq \sigma^2\cdot \frac{b-(b-1)\xi}{\gamma b-1}\cdot \frac{\frac{1+(b-1)\xi}{\gamma b-1}}{1+\frac{1+(b-1)}{\gamma b-1}}.
\end{align}
Plugging \eqref{eq:EQFQ} and \eqref{eq:EQHQ} back in \eqref{eq:QFQ_QHQ} we get 
\begin{align}\label{eq:EQFFQ}
  \lim_{p\to\infty}  \Q^T\left(\bs{F}\bs{F}^T\right)^{-1}\Q\geq \sigma^2\cdot \frac{b-(b-1)\xi}{\gamma b-1}\cdot C_{\gamma,\xi,b}.
\end{align}

\end{proof}

\bibliography{min_batch}
\bibliographystyle{ieeetr}

\end{document}